\documentclass[final]{siamltexmm}
\usepackage{graphicx}
\usepackage{amssymb}
\usepackage{subfigure}
\usepackage{algorithm}
\usepackage{algorithmic}
\usepackage{amsmath}
\usepackage{enumerate}
\usepackage{color}
\usepackage{textcomp}
\usepackage{relsize}
\usepackage{mathbbol}
\usepackage{wrapfig}

\newcommand{\revis}[1]{#1}

\newcommand{\siden}{\ensuremath{\mathcal{S}}}
\newcommand{\estsiden}{\ensuremath{\mathcal{Q}}}
\newcommand{\bparam}{\ensuremath{b}}
\newcommand{\Rsq}{\ensuremath{\mathbb{R}^2}}
\newcommand{\intRsq}{\ensuremath{ \int_{\mathbb{R}^2} }}
\newcommand{\half}{\ensuremath{ \frac{1}{2} }}

\newcommand{\erf}{\ensuremath{ \mathrm{erf} } }
\newcommand{\filtsc}{\ensuremath{ \Upsilon}}
\newcommand{\coef}{\ensuremath{ c } }

\newcommand{\atomk}{\ensuremath{\phi_{\gamma_k}(X) }}

\newcommand{\Tauxj}{\ensuremath{ \tau_{x,j}  }}

\newcommand{\Tauxk}{\ensuremath{ \tau_{x,k}  }}

\newcommand{\tsiden}{\ensuremath{ \omega_{\mathsmaller{T}} }}
\newcommand{\sumk}{\ensuremath{\sum_{k=1}^K}}
\newcommand{\sumj}{\ensuremath{\sum_{j=1}^K}}

\newcommand{\bndvardeltahUnif}{\ensuremath{R_{\sigma^2_{\Delta h}}}}
\newcommand{\bndstdevdeltahUnif}{\ensuremath{R_{\sigma_{\Delta h}}}}
\newcommand{\tbound}{\ensuremath{\overline{t}_0}}
\newcommand{\bndvarsecderh}{\ensuremath{R_{\sigma^2_{h''(tT)}}}}
\newcommand{\bndvargUnifsecderh}{\ensuremath{R_{\sigma^2_{h''}}}}
\newcommand{\bndstdevUnifsecderh}{\ensuremath{R_{\sigma_{h''}}}}
\newcommand{\bndnormsecderp}{\ensuremath{R_{p''}}}
\newcommand{\bndcorr}{\ensuremath{r_{pz}}}
\newcommand{\etabound}{\ensuremath{\eta_0}}
\newcommand{\tming}{\ensuremath{t_0}}
\newcommand{\Tming}{\ensuremath{T_0}}
\newcommand{\tminggen}{\ensuremath{u_{0}}}
\newcommand{\Tminggen}{\ensuremath{U_0}}
\newcommand{\tintg}{\ensuremath{t_1}}
\newcommand{\tintf}{\ensuremath{t_2}}
\newcommand{\tboundA}{\ensuremath{\overline{t}_0}}
\newcommand{\tboundB}{\ensuremath{R_{t_0}}}
\newcommand{\tboundBgen}{\ensuremath{R_{u_0}}}
\newcommand{\tboundBgenuncor}{\ensuremath{Q_{u_0}}}

\newcommand{\cthree}{\ensuremath{C_{\sigma^2_{\Delta h}}}}
\newcommand{\cfour}{\ensuremath{C_{\sigma^2_{h''}}}}
\newcommand{\sqrtcthree}{\ensuremath{C_{\sigma_{\Delta h}}}}
\newcommand{\sqrtcfour}{\ensuremath{C_{\sigma_{h''}}}}

\newcommand{\hatsiden}{\ensuremath{\hat{\mathcal{S}}}}
\newcommand{\hatestsiden}{\ensuremath{\hat{\mathcal{Q}}}}
\newcommand{\hattboundB}{\ensuremath{\hat{R}_{t_0}}}
\newcommand{\hattboundBgen}{\ensuremath{\hat{R}_{u_0}}}
\newcommand{\hattboundBgenuncor}{\ensuremath{\hat{Q}_{u_0}}}

\newcommand{\hateta}{\ensuremath{\hat{\eta}}}
\newcommand{\hatepsilon}{\ensuremath{\hat{\epsilon}}}

\newcommand{\hatbndstdevdeltahUnif}{\ensuremath{\hat{R}_{\sigma_{\Delta h}}}}

\newcommand{\hatbetazerolb}{\ensuremath{\hat{\underline{r}}_0 }}
\newcommand{\hatbndstdevUnifsecderh}{\ensuremath{\hat{R}_{\sigma_{h''}}}}

\newcommand{\hatsigma}{\ensuremath{\hat{\sigma}}}

\newcommand{\hattsiden}{\ensuremath{ \hat{\omega}_{\mathsmaller{T}} }}

\newcommand{\BonejkUBunif}{\ensuremath{ \overline{ \overline{\mathcal{B}}}_{jk}}}

\newcommand{\BonejkLBunif}{\ensuremath{ \underline{ \underline{\mathcal{B}}}_{jk}}}

\newcommand{\BtwojkUBunif}{\ensuremath{ \overline{ \overline{\mathcal{C}}}_{jk}}}

\newcommand{\BtwojkLBunif}{\ensuremath{ \underline{ \underline{\mathcal{C}}}_{jk}}}

\newcommand{\BthreejkUB}{\ensuremath{ \overline{\mathcal{D}}_{jk}}}
\newcommand{\BthreejkUBx}{\ensuremath{ \overline{\mathcal{D}}^x_{jk}}}
\newcommand{\BthreejkUBy}{\ensuremath{ \overline{\mathcal{D}}^y_{jk}}}
\newcommand{\BthreejkLB}{\ensuremath{  \underline{\mathcal{D}}_{jk}}}
\newcommand{\BthreejkLBx}{\ensuremath{  \underline{\mathcal{D}}^x_{jk}}}
\newcommand{\BthreejkLBy}{\ensuremath{  \underline{\mathcal{D}}^y_{jk}}}
\newcommand{\Ej}{\ensuremath{\mathcal{E}_{j}}}

\newcommand{\EjUBunif}{\ensuremath{\overline{\overline{\mathcal{E}}}_{j}}}
\newcommand{\EkUBunif}{\ensuremath{\overline{\overline{\mathcal{E}}}_{k}}}

\newcommand{\Fj}{\ensuremath{\mathcal{F}_{j}}}

\newcommand{\FjUBunif}{\ensuremath{\overline{\overline{\mathcal{F}}}_{j}}}
\newcommand{\FkUBunif}{\ensuremath{\overline{\overline{\mathcal{F}}}_{k}}}

\newcommand{\Galphax}{\ensuremath{ \overline{G}^x_{jk}}}
\newcommand{\Galphay}{\ensuremath{ \overline{G}^y_{jk}}}
\newcommand{\Gbetax}{\ensuremath{ \underline{G}^x_{jk}}}
\newcommand{\Gbetay}{\ensuremath{ \underline{G}^y_{jk}}}
\newcommand{\Halphax}{\ensuremath{ \overline{H}^x_{jk}}}
\newcommand{\Halphay}{\ensuremath{ \overline{H}^y_{jk}}}
\newcommand{\Hbetax}{\ensuremath{ {\underline{H}^x_{jk}}}}
\newcommand{\Hbetay}{\ensuremath{ {\underline{H}^y_{jk}}}}
\newcommand{\cLx}{\ensuremath{\mathfrak{b}^x_{jk} }}
\newcommand{\cLy}{\ensuremath{\mathfrak{b}^y_{jk} }}
\newcommand{\Rx}{\ensuremath{\mathfrak{d}^x_{jk} }}
\newcommand{\Ry}{\ensuremath{\mathfrak{d}^y_{jk} }}
\newcommand{\ctwoLx}{\ensuremath{\mathfrak{c}^x_{jk} }}
\newcommand{\ctwoLy}{\ensuremath{\mathfrak{c}^y_{jk} }}

\newcommand{\betazerolb}{\ensuremath{\underline{r}_0 }}
\newcommand{\Sigmabetazero}{\ensuremath{R_0}}

\newcommand{\Sigmabetatwojk}{\ensuremath{R^{jk}_2}}
\newcommand{\betatwolb}{\ensuremath{\underline{r}_2 }}

\newcommand{\betathreelb}{\ensuremath{\underline{r}_3 }}
\newcommand{\cdeltahposjk}{\ensuremath{ \mathbb{c}_{jk} }}
\newcommand{\cdeltahnegjk}{\ensuremath{ \mathbb{d}_{jk} }}
\newcommand{\cdeltahposjkUB}{\ensuremath{ \overline{\mathbb{c}}_{jk} }}
\newcommand{\cdeltahnegjkLB}{\ensuremath{\underline{ \mathbb{d}}_{jk} }}
\newcommand{\csecderhposjk}{\ensuremath{ \mathbb{e}_{jk} }}
\newcommand{\csecderhnegjk}{\ensuremath{ \mathbb{f}_{jk} }}
\newcommand{\csecderhposjkUB}{\ensuremath{\overline{\mathbb{e}}_{jk} }}
\newcommand{\csecderhnegjkLB}{\ensuremath{\underline{ \mathbb{f}}_{jk} }}

\newcommand{\ThAonejUB}{\ensuremath{ \overline{L}_j}}

\newcommand{\ThBonejUB}{\ensuremath{ \overline{M}_j}}

\newcommand{\ThtwojUB}{\ensuremath{ \overline{N}_j}}

\newcommand{\cj}{\ensuremath{ \kappa_j}}
\newcommand{\ck}{\ensuremath{ \kappa_k}}
\newcommand{\normp}{\ensuremath{ R_p}}
\newcommand{\tzero}{\ensuremath{t_0}}
\newcommand{\secderh}{\ensuremath{h''(tT)}}

\title{Analysis of Descent-Based Image Registration\thanks{This work has been partly funded by the Swiss National Science Foundation under Grant  $200020-132772$.}} 

\author{Elif Vural and Pascal Frossard \thanks{E. Vural and P. Frossard are with Ecole Polytechnique F\'{e}d\'{e}rale de Lausanne (EPFL), Signal Processing Laboratory - LTS4, CH-1015 Lausanne, Switzerland 
(\email{elif.vural@epfl.ch}, \email{pascal.frossard@epfl.ch}).}}

\begin{document}
\maketitle
\newcommand{\slugmaster}{%
\slugger{siims}{xxxx}{xx}{x}{x--x}}

\begin{abstract}

We present a performance analysis for image registration with gradient descent. We consider a typical multiscale registration setting where the global 2-D translation between a pair of images is estimated by smoothing the images and minimizing the distance between them with gradient descent. Our study particularly concentrates on the effect of noise and low-pass filtering on the alignment accuracy. We analyze the well-behavedness of the image distance function by estimating the neighborhood of translations for which it is free of undesired local minima. This is the neighborhood of translations that are correctly computable with a simple gradient descent minimization. We show that the area of this neighborhood increases at least quadratically with the smoothing filter size. We then examine the effect of noise on the alignment accuracy and derive an upper bound for the alignment error in terms of the noise properties and filter size. Our main finding is that the error increases at a rate that is at least linear with respect to the filter size. Therefore, smoothing improves the well-behavedness of the distance function; however, this comes at the cost of amplifying the alignment error in noisy settings. Our results provide a mathematical insight about why hierarchical techniques are effective in image registration, suggesting that the multiscale alignment strategy of these techniques is very suitable from the perspective of the trade-off between the well-behavedness of the objective function and the registration accuracy. To the best of our knowledge, this is the first such study for descent-based image registration.

\end{abstract}



\pagestyle{myheadings}
\thispagestyle{plain}
\markboth{Analysis of Descent-Based Image Registration}{E. Vural and P. Frossard}

\section{Introduction}
\label{sec:Intro}

The estimation of the relative motion between two images is one of the important problems of image processing. The necessity for registering images arises in many different applications like image analysis and classification \cite{Simard98}, \cite{Fitzgibbon03},  \cite{VasconcelosL05};  biomedical imaging \cite{Maintz98}, stereo vision \cite{Kanade81}, motion estimation for video coding \cite{Tziritas94}. The alignment of an image pair typically requires the optimization of a dissimilarity (or similarity) measure, whose common examples are  sum-of-squared difference (SSD), approximations of SSD, and cross-correlation \cite{Brown92}, \cite{Barron94}. Many registration techniques adopt, or can be coupled with, a multiscale hierarchical search strategy. In hierarchical registration, reference and target images are aligned by applying a coarse-to-fine estimation of the transformation parameters, using a pyramid of low-pass filtered and downsampled versions of the images. Coarse scales of the pyramid are used for a rough estimation of the transformation parameters. These scales have the advantage that the solution is less likely to get trapped into the local minima of the dissimilarity function as the images are smoothed by low-pass filtering. Moreover, the search complexity is lower at coarse scales as the image pair is downsampled accordingly. The alignment is then refined gradually by moving on to the finer scales. Since it offers a good compromise between complexity and accuracy, the coarse-to-fine alignment strategy has been widely used in many registration and motion estimation applications \cite{VasconcelosL05}, \cite{Nestares00}, \cite{Bergen92}, \cite{Barron94}, \cite{Anandan89}, \cite{Memin02}.

In this work, we present a theoretical study that analyzes the effect of smoothing on the performance of image registration. One of our main goals is to understand better the mathematical principles behind multiscale registration techniques. Most theoretical results in the image registration literature (e.g., \cite{Robinson04}, \cite{Yetik06},  \cite{Xu2009}) investigate how image noise affects the registration accuracy. However, the analysis of the effect of smoothing on the registration performance has generally been given little attention in the literature. Although it is widely known as a practical fact that smoothing an image pair is helpful for overcoming the undesired local minima of the dissimilarity function  \cite{Alvarez99}, \cite{Tziritas94}, to the best of our knowledge, this has not been extensively studied on a mathematical basis yet. Some of the existing works examine how smoothing influences the bias on the registration with gradient-based methods \cite{Robinson04}, and the bias and model conditioning in optical flow \cite{Kearney87}, \cite{Brandt94}, whose scopes are however limited to methods employing a linear approximation of the image intensity function. Hence, the understanding of the exact relation between smoothing and the well-behavedness of the image dissimilarity function constitutes the first objective of this study. Our second objective is to characterize the effect of noise on the performance of multi-scale image registration, i.e., to derive the noise performance in the registration of noisy images as a function of the smoothing parameter.

We consider a setting where the geometric transformation between the reference and target images is a global 2-D translation. 
Although the registration problem is formulated for an image pair in this work, one can equivalently assume that the considered reference and target patterns are image patches rather than complete images. For this reason, our study is of interest not only for registration applications where the transformation between the image pair is modeled by a pure translation (e.g., as in satellite images), but also for various motion estimation techniques, such as block-matching algorithms and region-based matching techniques in optical flow that assign constant displacement vectors to image subregions. We adopt an analytic and parametric model for the reference and target patterns and formulate the registration problem in the continuous domain of square-integrable functions $L^2(\Rsq)$. We use the squared-distance between the image intensity functions as the dissimilarity measure. This distance function is the continuous domain equivalent of SSD. We study two different aspects of image registration in this work; namely, alignment regularity and alignment accuracy. 

We first look at \textit{alignment regularity}; i.e., the well-behavedness of the distance function, and estimate the largest neighborhood of translations such that the distance function has only one local minimum, which is also the global minimum. Then we study the influence of  smoothing the reference and target patterns on the neighborhood of translations recoverable with local minimizers such as descent-type algorithms without getting trapped in a local minimum. In more details, we consider the set of patterns that are generated by the translations of a reference pattern, which forms the translation manifold of that pattern. In the examination of the alignment regularity, we assume that the target pattern lies on the translation manifold of the reference pattern. We then consider the distance function $f(U)$ between the reference and target patterns, where $U$ is the translation vector. The global minimum of $f$ is at the origin $U=0$. Then, in the translation parameter domain, we consider the largest open neighborhood around the origin within which $f$ is an increasing function along any ray starting out from the origin. We call this neighborhood the Single Distance Extremum Neighborhood (SIDEN). The SIDEN of a reference pattern is important in the sense that it defines the translations that can be correctly recovered by minimizing $f$ with a descent method. We derive an analytic estimation of the SIDEN. Then, in order to study the effect of smoothing on the alignment regularity, we consider the registration of low-pass filtered versions of the reference and target patterns and examine how the SIDEN varies with the filter size. Our main result is that the volume (area) of the SIDEN increases at a rate of at least $O(1+\rho^2)$ with respect to the size $\rho$ of the low-pass filter kernel, which controls the level of smoothing. This formally shows that, when the patterns are low-pass filtered, a wider range of translation values can be recovered with descent-type methods; hence, smoothing improves the regularity of alignment. Then, we demonstrate the usage of our SIDEN estimate for constructing a regular multiresolution grid in the translation parameter domain with exact alignment guarantees. Based on our estimation of the neighborhood of translations that are recoverable with descent methods, we design an adaptive search grid in the translation parameter domain such that large translations can be recovered by locating the closest solution on the grid and then refining this estimation with a descent method.

Then we look at \textit{alignment accuracy} and study the effect of image noise on the accuracy of image alignment. We also characterize the influence of low-pass filtering on the alignment accuracy in a noisy setting. This is an important question, as the target image is rarely an exactly translated version of the reference image in practical applications. When the target pattern is noisy, it is not exactly on  the translation manifold of the reference pattern. The noise on the target pattern causes the global minimum of the distance function to deviate from the solution $U=0$. We  formulate the alignment error as the perturbation in the global minimum of the distance function, which corresponds to the misalignment between the image pair due to noise. We focus on two different noise models. In the first setting, we look at Gaussian noise. In the second setting, we examine arbitrary square-integrable noise patterns, where we consider general noise patterns and noise patterns that have small correlation with the points on the translation manifold of the reference pattern. We derive upper bounds on the alignment error in terms of the noise level and the pattern parameters in both settings. We then consider the smoothing of the reference and target patterns in these settings and look at the variation of the alignment error with the noise level and the filter size. It turns out that the alignment error bound increases at a rate of $O\left(  {\eta}^{1/2} \, (1-\eta)^{-1/2}  \right)$ and $O\left(  {\nu}^{1/2} \, (1-\nu)^{-1/2}  \right)$ in respectively the first and second settings with respect to the noise level, where $\eta$ is the standard deviation of the Gaussian noise, and $\nu$ is the norm of the noise pattern. Another observation is that the alignment error is small if the noise pattern has small correlation with  translated versions of the reference pattern. Moreover, the alignment error bounds increase at the rates $O \left( {\rho^{3/2}} \, (1- \rho)^{-1/2} \right)$ and $O \left( (1+\rho^2)^{1/2} \right)$ in the first and second settings, with respect to the filter size $\rho$. Therefore, our main finding is that smoothing the image pair tends to increase the alignment error when the target pattern does not lie on the translation manifold of the reference pattern. The experimental results confirm that the behavior of the theoretical bound as a function of the noise level and filter size reflects well the behavior of the actual error. 

The results of our analysis show that smoothing has the desirable effect of improving the well-behavedness of the distance function; however, it also leads to the amplification of the alignment error caused by the image noise. This suggests that, in the development of multiscale image registration methods, one needs to take the noise level into account as a design parameter.

The rest of the text is organized as follows. In Section \ref{sec:related_work}, we give an overview of related work. In Section \ref{sec:siden_bnds}, we focus on the alignment regularity problem, where we first derive an estimation of the SIDEN and then examine its variation with filtering. Then in Section \ref{sec:noise_anly}, we look into the alignment accuracy problem and present our results regarding the influence of noise on the alignment accuracy. In Section \ref{sec:exp_res}, we present  experimental results. In Section \ref{sec:discussion}, we give a discussion of our results and interpret them in comparison with the previous studies in the literature. Finally, we conclude in Section \ref{sec:Conclusion}.

\section{Related Work}
\label{sec:related_work}

The problem of estimating the displacement between two images has been studied extensively in the image registration \cite{Zitova03} and motion estimation \cite{Tziritas94} literatures. Here we limit our discussion mostly to region-based methods. We first give a brief overview of some hierarchical multiscale registration and motion estimation methods and then mention some theoretical results about image alignment.

The coarse-to-fine alignment strategy has been used in various types of image registration applications; e.g., registration for stereo vision \cite{Kanade81}, alignment with multiresolution tangent-distance for image analysis \cite{VasconcelosL05}, biomedical image registration \cite{Nestares00}. The hierarchical search strategy has proved useful in motion estimation, since it accelerates the algorithm and leads to better solutions with reduced sensitivity to local minima \cite{Tziritas94}, \cite{Anandan89}, \cite{Bergen92}. It is also used very commonly in gradient-based optical flow techniques such as those in \cite{Barron94}, \cite{Memin02}, which apply a first-order approximation of the variations in the image intensity function. The hierarchical search that filters and downsamples the images permits the design of gradient-based methods that remain in the domain of linearity.

Region-based registration and motion estimation methods use different dissimilarity measures and optimization techniques. Many methods use SSD (sum-of-squared difference) as the dissimilarity measure \cite{Tziritas94}. SSD corresponds to the squared-norm of what is usually called the displaced frame difference (DFD)  in motion estimation. The direct correlation is also widely used as a similarity measure, and it can be shown to be equivalent to SSD \cite{Robinson04}.


In this work, we consider (the continuous domain equivalent of) SSD as the dissimilarity measure. We will essentially consider gradient descent as the minimization technique in our analysis; however, our main motivation is to understand to what extent local minimizers are efficient in image registration. Hence, the implications of our study concern a wide range of registration and motion estimation techniques that minimize SSD (or its approximations) with  local minimizers, e.g., \cite{Netravali79}, \cite{Cafforio83}, \cite{Walker84}, \cite{Bergen92}, \cite{VasconcelosL05}, and fast block-matching methods relying on convexity assumptions \cite{Tziritas94}.


We now overview some theoretical studies about the performance of  registration algorithms. The work by Robinson et al.~\cite{Robinson04} studies the estimation of global translation parameters between an image pair corrupted with additive Gaussian noise. The authors first derive the Cram\'{e}r-Rao lower bound (CRLB) on the translation estimation. Given by the inverse of the Fisher information matrix, the CRLB is a general lower bound for the MSE of an estimator that computes a set of parameters from noisy observations. The authors then examine the bias on multiscale gradient-based methods. A detailed discussion of the results in \cite{Robinson04} is given in Section \ref{sec:discussion} along with a comparison to our results. Another work that examines Cram\'{e}r-Rao lower bounds in registration is given in \cite{Yetik06}, where the bounds are derived for several models with transformations estimated from a set of matched feature points with noisy positions.

The studies in \cite{Kearney87}, \cite{Brandt94} have also examined the bias of gradient-based shift estimators and shown that presmoothing the images reduces the bias on the estimator. However, smoothing also has the undesired effect of impairing the conditioning of the linear system to be solved in gradient-based estimators \cite{Kearney87}. Therefore, this tradeoff must be taken into account in the selection of the filter size in coarse-to-fine gradient-based registration. The papers \cite{Robinson04}, \cite{Pham05} furthermore show that the bias on gradient-based estimators increases as the amount of translation increases. Robinson et al.~\cite{Robinson04} use this observation to explain the benefits of multiscale gradient-based methods. At large scales, downsampling, which reduces the amount of translation, and smoothing help to decrease the bias on the estimator. Then, as the change in the translation parameters is small at fine scales, the estimation does not suffer from this type of bias anymore. Moreover, at fine scales, the accuracy of the estimation increases as high-frequency components are no more suppressed. This is due to the fact that the CRLB of the estimation is smaller when the bandwidth of the image is larger. 

Next, in the article \cite{Lefebure01} the convergence of gradient-based registration methods is examined. It is shown that, if the images are smoothed with ideal low-pass filters by doubling the bandwidth in each stage, coarse-to fine gradient-based registration algorithms converge to the globally optimal solution provided that the amount of shift is sufficiently small. However, this convergence guarantee is obtained for an ideal noiseless setting.

Lastly, the article \cite{Sabater11} is a recent theoretical study on the accuracy of subpixel block-matching in stereo vision, which has relations to our work. The paper first examines the relation between the discrete and continuous block-matching distances, and then presents a continuous-domain analysis of the effect of noise on the accuracy of disparity estimation from a rectified stereo pair corrupted with additive Gaussian noise. An estimation of the disparity that globally minimizes the windowed squared-distance between blocks is derived. A comparison of the results presented in \cite{Sabater11} and in our work is given in Section \ref{sec:discussion}.

The previous works that have studied the alignment accuracy of multiscale registration by examining the effect of smoothing are limited to gradient-based methods, i.e., methods that employ a linear approximation of the image intensity. Moreover, none of these studies focus on the alignment regularity aspect of image registration. In this work, we address both of these issues and derive bounds on both alignment regularity and alignment accuracy in multiscale registration.

\section{Analysis of Alignment Regularity}
\label{sec:siden_bnds}
\subsection{Notation and Problem Formulation}
\label{ssec:siden_notations}

Let $p \in L^2(\mathbb{R}^2)$ be a visual pattern with a non-trivial support on $\Rsq$ (i.e., $p(X)$ is not equal to 0 almost everywhere on $\Rsq$). In order to study the image registration problem analytically, we adopt a representation of $p$ in an analytic and parametric dictionary manifold
\begin{equation}
\label{eq:dictManifoldExp}
\mathcal{D}  = \{ {\phi}_{\gamma}: \gamma = (\psi, \tau_x, \tau_y, \sigma_x, \sigma_y) \in \Gamma  \} \subset L^2(\mathbb{R}^2).
\end{equation}
Here, each atom $\phi_{\gamma}$ of the dictionary $\mathcal{D}$ is derived from an analytic mother function $\phi$  by a geometric transformation specified by the parameter vector $\gamma$, where $\psi$ is a rotation parameter, $\tau_x$ and $ \tau_y$ denote translations in $x$ and $y$ directions, and $\sigma_x$ and $ \sigma_y$ represent an anisotropic scaling in $x$ and $y$ directions. $\Gamma$ is the transformation parameter domain over which the dictionary is defined. By defining the spatial coordinate variable $X=[x \, \, y]^T \in \mathbb{R}^{2\times 1}$, we refer to the mother function as $\phi(X)$. Then an atom $\phi_{\gamma}$ is given by
\begin{equation}
\phi_{\gamma}(X)=\phi(\sigma^{-1} \, \Psi^{-1} \, (X-\tau)),
\end{equation}
where
\begin{equation}
\sigma= \left[
\begin{array}{c c}
 \sigma_x & 0  \\
 0 & \sigma_y
\end{array} \right], \, \, \,
\Psi= \left[
\begin{array}{c c}
 \cos(\psi) & -\sin(\psi)  \\
\sin(\psi) & \cos(\psi)
\end{array} \right], \, \, \,
\tau= \left[
\begin{array}{c}
\tau_x \\
\tau_y
\end{array} \right].
\end{equation}

It is shown in \cite{Antoine2004} (in the proof of Proposition 2.1.2) that the linear span of a dictionary $\mathcal{D}$ generated with respect to the transformation model in (\ref{eq:dictManifoldExp}) is dense in $L^2(\mathbb{R}^2)$ if the mother function $\phi$ has nontrivial support (unless $\phi(X)=0$ almost everywhere). In our analysis, we choose $\phi$ to be the Gaussian function 
$ \phi(X) = e^{-X^T X}=e^{-(x^2+y^2)} $
as it has good time-localization and it is easy to treat in derivations due to its well-studied properties. This choice also ensures that $Span(\mathcal{D})$ is dense in $L^2(\mathbb{R}^2)$; therefore, any pattern $p \in L^2(\mathbb{R}^2)$ can be approximated in $\mathcal{D}$ with arbitrary accuracy. We assume that a sufficiently accurate approximation  of $p$ with finitely many atoms in $\mathcal{D}$ is available; i.e.,
\begin{equation}
p(X) \approx \sum_{k=1}^K \coef_k \, \phi_{\gamma_k}(X)
\label{eq:pXgaussAtoms}
\end{equation}
where $K$ is the number of atoms used in the representation of $p$, $\gamma_k$ are the atom parameters and $\coef_k$ are the atom coefficients.\footnote{\revis{In the practical computation of the representation of a digital image in the Gaussian dictionary, the number of atoms should be chosen such that a substantial part of the energy of the image is captured in the approximation. Atom parameters and coefficients can be computed in various ways, for instance, by using pursuit algorithms such as \cite{1028585}, \cite{Pati93} with a redundant sampling of the Gaussian dictionary, or with DC (difference-of-convex) optimization as in \cite{Vural13}.}}

Throughout the discussion, $T = [T_x \, \,  T_y]^T \in S^1$ denotes a unit-norm vector and $S^1$ is the unit circle in $\Rsq$. We use the notation $tT$ for translation vectors, where $t\geq0$ denotes the magnitude of the vector (amount of translation) and $T$ defines the direction of translation. Then, the translation manifold $\mathcal{M}(p)$ of $p$ is the set of patterns generated by translating $p$
\begin{equation}
\mathcal{M}(p)=\{p(X-t\,T): T \in S^1,\,  t\in [0,+\infty) \} \subset  L^2(\mathbb{R}^2).
\end{equation}

We consider the squared-distance between the reference pattern $p(X)$ and its translated version $p(X-tT)$. This distance is the continuous domain equivalent of the SSD measure that is widely used in registration methods. The squared-distance in the continuous domain is given by

\begin{equation}
f(tT)= \|  p(\cdot)-p(\cdot-tT) \|^2 =  \int_{\mathbb{R}^2} (p(X)-p(X-tT))^2 dX
\label{eq:ftT}
\end{equation}
where the notation\footnote{The notations $\| p \|$ or $\| p(\cdot) \|$ are always used to refer to the $L^2$-norm of $p$, considered as an element of the vector space of functions. Since it is then clear whether the $L^2(\Rsq)$-norm or the $\ell^2$-norm is meant, we denote these in the same way for simplicity of notation.} $\| . \|$ stands for the $L^2$-norm for vectors in $L^2(\Rsq)$ and the $\ell^2$-norm for vectors in $\Rsq$.

\revis{Note that in image registration, windowed versions of the image pair centered around a point of interest can be used as well as the entire images \cite{Zitova03}. If the reference and target images are windowed around a point $X_0$, the distance function becomes  
\begin{equation*}
\int_{\Rsq} \big(\varphi(X-X_0)  p(X)- \varphi(X-X_0) p(X-tT) \big)^2 \, dX
\end{equation*}
where $\varphi: \Rsq \rightarrow [0, 1]$ is a window function. If the window is chosen sufficiently large such that it covers the region where the reference pattern $p$ has significant intensity as well as its translated versions $p(X-tT)$ for realistic values of $tT$, one can approximate this function with the function $f(tT)$ in (\ref{eq:ftT}). In this section, we base our analysis on the non-windowed distance function $f(tT)$ in order to keep the derivations simple.}

The global minimum of $f$ is at the origin $tT=0$. Therefore, there exists a region around the origin within which the restriction of $f$ to a ray $t T_a$ starting out from the origin along an arbitrary direction $T_a$ is an increasing function of $t>0$ for all $T_a$. This allows us to define the Single Distance Extremum Neighborhood (SIDEN) as follows.

\begin{definition}
We call the set of translation vectors
\begin{equation}
\siden= \{ 0 \} \cup  \{\tsiden T: T \in S^1, \tsiden>0, \text{ and  } \frac{d f(tT)}{dt}>0\, \, \text{ for all }  0<t \leq \tsiden \}
\label{eq:defn_truesiden}
\end{equation}
the Single Distance Extremum Neighborhood (SIDEN) of the pattern $p$. 
\end{definition}

Note that the origin $\{ 0 \}$ is included separately in the definition of SIDEN since the gradient of $f$ vanishes at the origin and therefore $df(tT)/dt |_{t=0} =0$ for all $T$. The SIDEN $\siden \subset \mathbb{R}^2$ is an open neighborhood of the origin such that the only stationary point of $f$ inside $\siden$ is the origin. We formulate this in the following proposition.

\begin{proposition}
Let $tT  \in \siden $. Then $\nabla f(t T)=0$ if and only if $t T=0$.
\label{prop:singleExtProp}
\end{proposition}
\begin{proof}
Let $\nabla f(tT)=0$ for some $tT  \in \siden $. Then, $\nabla_{T} f (tT)=0$, which is the directional derivative of $f$ along the direction $T$ at $tT$. This gives 
\begin{equation*}
\nabla_{T} f(tT)=\frac{d}{d u}  f( t T+ u T) \bigg|_{u=0}= \frac{d}{d u} f\left( (t+u) T \right) \bigg|_{u=0} = \frac{d}{d u} f(u T) \bigg|_{u=t}=\frac{d f(tT)}{dt} = 0
\end{equation*}
which implies that $t=0$, as $t T \in \siden$. The second part $\nabla f (0) =0$  of the statement also holds clearly, since the global minimum of $f$ is at $0$.
\end{proof}

Proposition \ref{prop:singleExtProp} can be interpreted as follows. The only local minimum of the distance function $f$ is at the origin in $\siden$. Therefore, when a translated version $p(X-tT)$ of the reference pattern is aligned with $p(X)$ with a local optimization method like a gradient descent algorithm, the local minimum achieved in $\siden$ is necessarily also the global minimum.

The goal of our analysis is now the following. Given a reference pattern $p$, we would like to find an analytical estimation of $ \siden$. However, the exact derivation of $ \siden$ requires the calculation of the exact zero-crossings of $df(tT)/dt$, which is not easy to do analytically. Instead, one can characterize the SIDEN by computing a neighborhood $\estsiden$ of $0$ that lies completely in $ \siden$; i.e., $\estsiden \subset \siden$. $\estsiden$ can be derived by using a polynomial approximation of $f$ and calculating, for all unit directions $T$, a lower bound $\delta_T$ for the supremum of $\tsiden$ such that $\tsiden T$ is in $\siden$. This does not only provide an analytic estimation of the SIDEN, but also defines a set that is known to be completely inside the SIDEN. The regions $\siden$ and $\estsiden$ are illustrated in Figure \ref{fig:illus_siden}.

\begin{wrapfigure}{l}{0.45\textwidth}
 \begin{center}
  \includegraphics[scale=0.7, trim=0cm 1cm 4cm 0cm, clip=true]{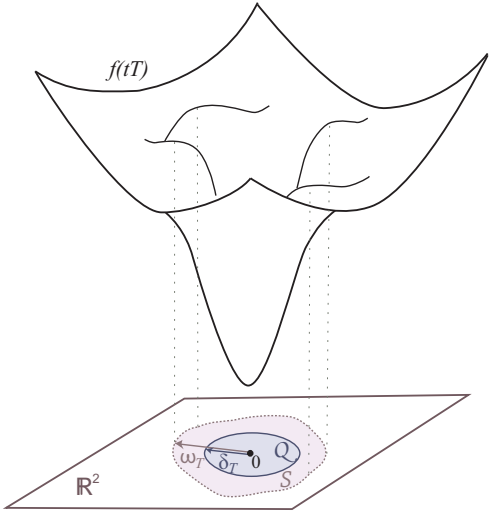}
  \end{center}
  \caption{SIDEN $\siden$ is the largest open neighborhood around the origin within which the distance $f$ is increasing along all rays starting out from the origin. Along each unit direction $T$, $\siden$ covers points $\tsiden T$ such that $f(tT)$ is increasing between $0$ and $\tsiden T$. The estimate $\estsiden$ of $\siden$ is obtained by computing a lower bound $\delta_T$ for the first zero-crossing of $df(tT)/dt$.}
  \label{fig:illus_siden}
\end{wrapfigure}

In Section \ref{sec:siden_estim} we derive $\estsiden$. In particular, $\estsiden$ is obtained in the form of a compact analytic set and $f$ is a differentiable function. This guarantees that, if the translation that aligns the image pair perfectly is in the set $\estsiden$, the distance function $f$ can be minimized with gradient descent algorithms; the solution converges to a local minimum of $f$ in $\estsiden$, which is necessarily the global minimum of $f$, resulting in a perfect alignment. Moreover, we will see in Section \ref{sec:exp_res} that the knowledge of a set $\estsiden \subset \siden$ permits us to design a registration algorithm that can recover large translations   perfectly.

Finally, as $\estsiden$ is obtained analytically and parametrically, it is simple to examine its variation with the low-pass filtering applied to $p$. This is helpful for gaining an understanding of the relation between the alignment regularity and smoothing. We study this relation in Section \ref{ssec:SIDENregular}.

\subsection{Estimation of SIDEN}
\label{sec:siden_estim}

We now derive an estimation $\estsiden$ for the Single Distance Extremum Neighborhood $\siden$. In the following, we consider $T$ to be a fixed unit direction in $S^1$. We derive $ \estsiden \subset \siden$ by computing a $\delta_T$ which guarantees that $df(tT)/dt>0$ for all $0<t\leq \delta_T$. In the derivation of $\estsiden$, we need a closed-form expression for $df(tT)/dt$. Since $f$ is the distance between the patterns $p(X)$ and $p(X-tT)$ that are represented in terms of Gaussian atoms (see Eq.~\ref{eq:pXgaussAtoms}), it involves the integration of the multiplication of pairs of Gaussian atoms. We will use the following proposition about the integration of the product of Gaussian atoms \cite{WandJones1995}.

\begin{proposition}
Let $\phi_{\gamma_j}(X)=\phi(\sigma_j^{-1} \, \Psi_j^{-1} \, (X-\tau_j))$ and $\phi_{\gamma_k}(X)=\phi(\sigma_k^{-1} \, \Psi_k^{-1} \, (X-\tau_k))$. Then 
\begin{equation*}
\int_{\mathbb{R}^2} \phi_{\gamma_j}(X) \phi_{\gamma_k}(X) dX 
= \frac{ \pi \, | \sigma_j \sigma_k |}  {2  \sqrt{ |\Sigma_{jk}| } }
\exp \left( - \frac{1}{2} (\tau_k - \tau_j)^{T} \, \Sigma_{jk}^{-1} \, (\tau_k - \tau_j) \right)
\end{equation*} 
where
$
\Sigma_{jk}:=\frac{1}{2} \left( \Psi_j \, \sigma_j^2 \, \Psi_j^{-1} + \Psi_k \, \sigma_k^2 \, \Psi_k^{-1}  \right)
$.
\label{prop:IntGaussProd}
\end{proposition}

The symbol $\Sigma_{jk}$ defined in Proposition \ref{prop:IntGaussProd} is a function of the parameters of the $j$-th and $k$-th atoms. We also denote
\begin{equation}
\begin{split}
a_{jk} &:= \frac{1}{2} \,\,  T^{T} \,\, \Sigma_{jk}^{-1} \,\, T,
\qquad
b_{jk} := \frac{1}{2} \,\,  T^{T} \,\, \Sigma_{jk}^{-1} \,\, ( \tau_k - \tau_j) \\
c_{jk} &:= \frac{1}{2} \,\,  ( \tau_k - \tau_j)^{T} \,\, \Sigma_{jk}^{-1} \,\, ( \tau_k - \tau_j), 
\qquad
Q_{jk} := \frac{ \pi \, | \sigma_j \sigma_k | e^{- c_{jk} }  }  {  \sqrt{ |\Sigma_{jk}| } }.
\end{split}
\label{eq:a_b_c_jk}
\end{equation}
Notice that $a_{jk}>0$ and $c_{jk}\geq 0$ since $\| T \| = 1$ and $\Sigma_{jk}$, $\Sigma_{jk}^{-1}$ are positive definite matrices. By definition, $Q_{jk}>0$ as well. Note also that $a_{jk}$ and $b_{jk}$ are functions of the unit direction $T$; however, for the sake of simplicity we avoid expressing their dependence on $T$ explicitly in our notation.\\

We can now give our result about the estimation of the SIDEN.

\begin{theorem}
\label{thm:siden_bnd}
The region $\estsiden \subset \Rsq$ is a subset of the SIDEN $\siden$ of the pattern $p$ if
\[
\estsiden=\{ t T: T\in S^1, \, \,  0\leq t \leq \delta_T\}
\]
where $\delta_T$ is the only positive root of the polynomial $  | \alpha_4 | t^3  - \alpha_3 t^2  - \alpha_1  $ and
\begin{eqnarray*}
\alpha_1 &=&  \sumj \sumk \coef_j \coef_k \, Q_{jk} \,  (2 \, a_{jk}  - 4 \, b_{jk}^2)\\
\alpha_3 &=&   \sumj \sumk \coef_j \coef_k \, Q_{jk} \, \left( -\frac{8}{3} \,  b_{jk}^4 + 8 \, b_{jk}^2 \, a_{jk} - 2 \, a_{jk}^2 \right) \\
\alpha_4 &=&   -  1.37   \sumj \sumk | \coef_j \coef_k | \, Q_{jk}  \, \exp \left( \frac{ b_{jk}^2} {a_{jk} }   \right) a_{jk}^{5/2}
\end{eqnarray*}
are constants depending on $T$ and on the parameters $\gamma_k$ of the atoms in $p$. 
\end{theorem}\\

The proof of Theorem \ref{thm:siden_bnd} is given in \cite[Appendix A.1]{Vural12TR}. The proof applies a Taylor expansion of $df(tT)/ dt$ and derives $\delta_T$ such that $df(tT)/dt$ is positive  for all $t \leq \delta_T$. Therefore, along each direction $T$, $\delta_T$ constitutes a lower bound for the first zero-crossing of $df(tT)/dt$ (see Figure \ref{fig:illus_siden} for an illustration of $\delta_T$). By varying $T$ over the unit circle, one obtains a closed neighborhood $\estsiden$ of $0$ that is a subset of $\siden$. This region can be analytically computed using only the parametric representation of $p$ and provides an estimate for the range of translations $tT$ over which $p(X)$ can be exactly aligned with $p(X-tT)$.

\subsection{Variation of SIDEN with Smoothing}
\label{ssec:SIDENregular}

We now examine how smoothing the reference pattern $p$ with a low-pass filter influences its SIDEN. We assume a Gaussian kernel for the filter. The Gaussian function is a commonly used kernel for low-pass filtering and its distinctive properties has made it popular in scale-space theory research \cite{Lindeberg94} (see Section \ref{sec:discussion} for a more detailed discussion). We assume that $p$ is filtered with a Gaussian kernel of the form
$
\frac{1}{\pi \rho^2} \phi_{\rho}(X)
\label{eq:Gausskerdefn}
$
with unit $L^1$-norm. The function $\phi_{\rho}(X)=\phi(\filtsc^{-1} (X))$ is an isotropic Gaussian atom with scale matrix

\begin{equation}
\filtsc= \left[
\begin{array}{c c}
 \rho & 0  \\
 0 & \rho
\end{array} \right].
\end{equation}
The scale parameter $\rho$ controls the size of the Gaussian kernel. We denote the smoothed version of the reference pattern $p(X)$ by $\hat p(X)$, which is given as
\begin{eqnarray}
\begin{split}
\hat p(X) &= \frac{1}{\pi \rho^2} \,  (\phi_{\rho}* p) (X) = \sum_{k=1}^K   \coef_k \, \frac{1}{\pi \rho^2} \,  (\phi_{\rho} * \phi_{\gamma_k} )(X)\\
\end{split}
\label{eq:hatpX}
\end{eqnarray}
by linearity of the convolution operator. In order to calculate $\hat p$, we use the following proposition, which gives the expression for the Gaussian atom obtained from the convolution of two Gaussian atoms \cite{WandJones1995}.

\begin{proposition}
Let $\phi_{\gamma_1}(X)=\phi(\sigma_1^{-1} \, \Psi_1^{-1} \, (X-\tau_1))$ and $\phi_{\gamma_2}(X)=\phi(\sigma_2^{-1} \, \Psi_2^{-1} \, (X-\tau_2))$. Then 
\begin{equation}
(\phi_{\gamma_1} * \phi_{\gamma_2} ) (X) = \frac{\pi | \sigma_1 \sigma_2 |}{| \sigma_3 |}  \phi_{\gamma_3}(X)
\end{equation}
where
$ \phi_{\gamma_3}(X) = \phi(\sigma_3^{-1} \, \Psi_3^{-1} \, (X-\tau_3)) $
and the parameters of $\phi_{\gamma_3}(X)$ are given by
\[
 \tau_3= \tau_1 + \tau_2, \, \, \, \, \, \, \, \, \, \, 
 \Psi_3 \, \sigma_3^2 \, \Psi_3^{-1} = \Psi_1 \, \sigma_1^2 \, \Psi_1^{-1} + \Psi_2 \, \sigma_2^2 \, \Psi_2^{-1}. 
 \]
 \label{prop:convGauss}
\end{proposition}
%
From Proposition \ref{prop:convGauss}, we obtain

\begin{equation}
 \frac{1}{\pi \rho^2} \,  ( \phi_{\rho} * \phi_{\gamma_k}) (X) = \frac{| \sigma_k |}{ | \hat{\sigma}_k | } \phi_{\hat{\gamma}_k} (X) 
\end{equation}
where $ \phi_{\hat{\gamma}_k} (X) =  \phi(\hat {\sigma}_k^{-1} \, \hat{\Psi}_k^{-1} \, (X-\hat{\tau}_k))  $ and
\begin{equation}
\hat{\tau}_k = \tau_k, \, \, \, \, \, \,  \hat{\Psi}_k = \Psi_k,  \, \, \, \, \, \, \hat{\sigma}_k  =\sqrt{ \filtsc^2  + \sigma_k^2 }. 
\label{eq:defn_filtered_atompar}
\end{equation}
Hence, when $p$ is smoothed with a Gaussian filter, the atom $\atomk$ with coefficient $\coef_k$ in the original pattern $p$ is replaced by the smoothed atom $ \phi_{\hat{\gamma}_k}(X)$ with coefficient
\begin{equation}
\hat{\coef}_k = \frac{| \sigma_k |}{ | \hat{ \sigma}_k | } \coef_k
 = \frac{| \sigma_k |}{\sqrt{ | \filtsc^2 + \sigma_k^2 |} }   \coef_k
 = \frac{\sigma_{x,k} \, \sigma_{y,k} }{\sqrt{ (\rho^2+\sigma_{x,k}^2)(\rho^2+\sigma_{y,k}^2 ) }} \coef_k
 \label{eq:defn_hatlambdak}
\end{equation}
where $\sigma_k=diag(\sigma_{x,k} ,\, \sigma_{y,k})$. This shows that the change in the pattern parameters due to the filtering can be captured by substituting the scale parameters $\sigma_k$ with $\hat{\sigma}_k$ and replacing the coefficients $\coef_k$ with $\hat{\coef}_k$. Thus, the smoothed pattern $\hat{p}$ is sparsely representable in the dictionary $\mathcal{D}$ as
\begin{equation}
\label{eq:form_phat}
\hat{p}(X)= \sumk \hat{\coef}_k \phi_{\hat{\gamma}_k}(X) .\\
\end{equation}

Considering the same setting as in Section \ref{ssec:siden_notations}, where the target pattern $p(X-tT)$ is exactly a translated version of the reference pattern $p(X)$, we now assume that both the reference and target patterns are low-pass filtered as it is typically done in hierarchical image registration algorithms. The above equations show that, when a pattern is low-pass filtered, the scale parameters of its atoms increase and the atom coefficients decrease proportionally to the filter kernel size, leading to a spatial diffusion of the image intensity function. The goal of this section is to show that this diffusion increases the volume of the SIDEN. We achieve this by analyzing the variation of the smoothed SIDEN estimate $\hatestsiden$ corresponding to the smoothed distance  
\begin{equation}
\hat{f}(tT)= \int_{\mathbb{R}^2} \big( \hat{p}(X)-\hat{p}(X-tT) \big)^2 dX\\
\end{equation}
with respect to the filter size $\rho$. Since the smoothed pattern has the same parametric form (\ref{eq:form_phat}) as the original pattern, the variation of $\hatestsiden$ with $\rho$ can be analyzed easily by examining the dependence of the parameters involved in the derivation of $\hatestsiden$ on $\rho$. In the following, we express the terms in Section \ref{sec:siden_estim} that have a dependence on $\rho$ with the notation $\hat{(.)}$, such as $\hat{a}_{jk}$, $\hat{b}_{jk}$, $\hat{\coef}_k$, $\hat{\sigma}_k$. We write the terms that do not depend on $\rho$ in the same way as before; e.g., $t$, $T$, $\tau_k$, $\Psi_k$.

Now, we can apply the result of Theorem \ref{thm:siden_bnd} for the smoothed pattern $\hat{p}(X)$. For a given kernel size $\rho$, the smoothed versions $\hat{a}_{jk}$, $\hat{b}_{jk}$, $\hat{c}_{jk}$, $\hat{Q}_{jk}$ of the parameters in (\ref{eq:a_b_c_jk}) can be obtained by replacing the scale parameters $\sigma_k$ with $\hatsigma_k$ defined in (\ref{eq:defn_filtered_atompar}). Then, the smoothed SIDEN corresponding to $\rho$ is given as
$
\hatestsiden=\{ t T: T \in S^1, \, \,  0\leq t \leq \hat{\delta}_T\}
$
where $\hat{\delta}_T$ is the positive root of the polynomial $| \hat{\alpha}_4 | t^3  -  \hat{\alpha}_3 t^2  -  \hat{\alpha}_1  $ such that
\begin{eqnarray*}
&&\hat{\alpha}_1 =  \sumj \sumk \hat{\coef}_j \hat{\coef}_k \, \hat{Q}_{jk} \,  (2 \, \hat{a}_{jk}  - 4 \, \hat{b}_{jk}^2),
\qquad
\hat{\alpha}_3 =  \sumj \sumk \hat{\coef}_j \hat{\coef}_k \, \hat{Q}_{jk} \, \left( -\frac{8}{3} \,  \hat{b}_{jk}^4 + 8 \, \hat{b}_{jk}^2 \, \hat{a}_{jk} - 2 \, \hat{a}_{jk}^2 \right) \\
&&\hat{\alpha}_4 = - 1.37  \sumj \sumk | \hat{\coef}_j \hat{\coef}_k | \, \hat{Q}_{jk} \,   \exp \left( \frac{ \hat{b}_{jk}^2} {\hat{a}_{jk} }   \right) \hat{a}_{jk}^{5/2}. 
\end{eqnarray*}
Similarly to the derivation in Section \ref{sec:siden_estim}, the terms $\hat{a}_{jk}$, $\hat{b}_{jk}$, $\hat{c}_{jk}$, $\hat{Q}_{jk}$ are associated with the integration of the product of smoothed Gaussian atom pairs, and they appear in the closed-form expression of $d \hat{f}(tT) / dt $.

We are now ready to give the following result, which summarizes the dependence of the smoothed SIDEN estimate on the filter size $\rho$.\\
\begin{theorem}
\label{prop:volumeSIDEN}
Let $V(\hatestsiden)$ denote the volume (area) of the SIDEN estimate $\hatestsiden$ for the smoothed pattern $\hat{p}$. Then, the order of dependence of the volume of $\hatestsiden$ on the filter size $\rho$ is given by $V(\hatestsiden) = O(1+\rho^2)$. \revis{Moreover, the distance $\hat{\delta}_T$ of the boundary of $\hatestsiden$ to the origin increases at a rate of $O(\sqrt{1+\rho^2})$ with $\rho$ along any direction $T$.} \\
\end{theorem}

Theorem \ref{prop:volumeSIDEN} is proved in \cite[Appendix A.2]{Vural12TR}. The proof is based on the examination of the order of variation of $\hat{a}_{jk}$, $\hat{b}_{jk}$, $\hat{c}_{jk}$, $\hat{Q}_{jk}$ with $\rho$, which is then used to derive the dependence of $\hat{\delta}_T$ on $\rho$.

Theorem \ref{prop:volumeSIDEN} is the main result of this section. It states that the volume of the SIDEN estimate increases with the size of the filter applied on the patterns to be aligned. The theorem shows that the volume of the region of translations for which the reference pattern $\hat{p}(X)$ can be perfectly aligned with $\hat{p} (X-tT) $ using a descent method, expands at the rate $O(1+ \rho^2)$ with respect to the increase in the filter size $\rho$. Here, the order of variation $O(1+ \rho^2)$ is obtained for the estimate $\hatestsiden$ of the SIDEN. Hence, one may wonder if the volume $V(\hatsiden)$ of the SIDEN  $\hatsiden$ has the same dependence on $\rho$. Remembering that $\hatestsiden \subset \hatsiden$ for all $\rho$, one immediate observation is that the rate of expansion of $ \hatsiden$ must be at least $O(1+ \rho^2)$; otherwise, there would exist a sufficiently large value of $\rho$ such that $\hatestsiden$ is not included in $\hatsiden$. One can therefore conclude that $V(\hatsiden) \geq V(\hatestsiden) = O(1+ \rho^2)$. \revis{Theorem \ref{prop:volumeSIDEN} also states that the SIDEN boundary along any direction $T$ expands at a rate of at least $O(\sqrt{1+\rho^2})$.} However, this only gives a lower bound for the rate of expansion of $\hatsiden$ and the exact rate of expansion of $\hatsiden$ may be larger. In the following, we give a few comments about the variation of $\hatsiden$ with $\rho$. \\

\textbf{Remark.}
As shown in the proof of Theorem \ref{thm:siden_bnd}, the derivative of the distance function $f(tT)$ is of the form
\begin{equation}
\frac{d f(tT)}{dt} = \sumj \sumk  \coef_j \coef_k \,\,  Q_{jk} \,\, s_{jk}(t)
\label{eq:form_dft_dt_v1}
\end{equation} 
where
\begin{equation}
s_{jk}(t) =    e^{-\left(a_{jk} \, t^2 \, +\, 2 \, b_{jk} \, t  \right)} 
                    \left( a_{jk} \, t \, + \, b_{jk} \right)
            \, +\,   e^{-\left(a_{jk} \, t^2 \, - \, 2 \, b_{jk} \, t  \right)} 
 		     \left( a_{jk} \, t \, - \, b_{jk} \right).	 
\label{eq:defn_sjkt}
\end{equation} 

In order to derive $\siden$, one needs to exactly locate the smallest zero-crossing of $\frac{d f(tT)}{dt}$. This is not easy to do analytically due to the complicated form of the functions $s_{jk}(t)$, which we have handled with polynomial approximations in the derivation of $\estsiden$. However, in order to gain an intuition about how the zero-crossings change with filtering, one can look at the dependence of the extrema of the two additive terms in $s_{jk}(t) $ on $\rho$. The function $e^{-\left(a_{jk} \, t^2 \, +\, 2 \, b_{jk} \, t  \right)} \left( a_{jk} \, t \, + \, b_{jk} \right)$ has two extrema at
\begin{equation}
\mu_0=\frac{1}{a_{jk}} \left( -\sqrt{\frac{a_{jk}}{2}} - b_{jk} \right),
\qquad \qquad
\mu_1=\frac{1}{a_{jk}} \left( \sqrt{\frac{a_{jk}}{2}} - b_{jk} \right)
\end{equation} 
and  $ e^{-\left(a_{jk} \, t^2 \, - \, 2 \, b_{jk} \, t  \right)}  \left( a_{jk} \, t \, - \, b_{jk} \right)$ has two extrema at
\begin{equation}
\mu_2=\frac{1}{a_{jk}} \left( -\sqrt{\frac{a_{jk}}{2}} + b_{jk} \right),
\qquad \qquad
\mu_3=\frac{1}{a_{jk}} \left( \sqrt{\frac{a_{jk}}{2}} + b_{jk} \right).
\end{equation} 
Now replacing the original parameters $a_{jk}$, $b_{jk}$ with their smoothed versions $\hat{a}_{jk}$, $\hat{b}_{jk}$ and using the result from the proof of Theorem \ref{prop:volumeSIDEN} that $\hat{a}_{jk}$ and $\hat{b}_{jk}$ decrease at a rate of $O\left( (1+\rho^2)^{-1} \right)$, it is easy to show that the locations of the  extrema $\hat{\mu}_0, \hat{\mu}_1, \hat{\mu}_2, \hat{\mu}_3$ change with a rate of $O \big( (1+ \rho^2)^{1/2} \big)$. One may thus conjecture that the zero-crossings of $d f(tT) / dt $ along a fixed direction $T$ might also move at the same rate, which gives the volume of $\hatsiden$ as $V(\hatsiden) =  O(1+ \rho^2)$.

On the other hand, $V(\hatsiden) $ may also exhibit a different type of variation with $\rho$ depending on the atom parameters of $p$. In particular, $V(\hatsiden)$ may expand at a rate greater than $O(1+ \rho^2)$ for some patterns. For example, as shown in \cite[Proposition 4]{Vural12TR}, there exists a threshold value $\rho_0$ of the filter size such that for all $\rho > \rho_0$, $\hatsiden = \Rsq$ and thus $V(\hatsiden) = \infty$ for patterns that consist of atoms with coefficients of the same sign. In addition, patterns whose atoms with positive (or negative) coefficients are dominant over the atoms with the opposite sign are likely to have this property due to their resemblance to patterns consisting of atoms with coefficients of the same sign.
\\

Theorem \ref{prop:volumeSIDEN} describes the effect of smoothing images before alignment. One may then wonder what the optimal filter size to be applied to the patterns before alignment is, given a reference and a target pattern. Theorem \ref{prop:volumeSIDEN} suggests that, if the target pattern is on the translation manifold of the reference pattern, applying a large filter is always preferable as it provides a large range of translations recoverable by descent algorithms. The accuracy of alignment does not change with the filter size in this noiseless setting, since a perfect alignment is always guaranteed with descent methods as long as the amount of translation is inside the SIDEN. However, the assumption that the target pattern is exactly of the form $p(X-tT)$  is not realistic in practice; i.e., in real image processing applications, the target image is likely to deviate from $\mathcal{M}(p)$ due to the noise caused by image capture conditions, imaging model characteristics, etc. Hence, we examine in Section \ref{sec:noise_anly} if filtering affects the accuracy of alignment when the target image deviates from $\mathcal{M}(p)$.

\section{Analysis of Alignment Accuracy in Noisy Settings}
\label{sec:noise_anly}

We now analyze the effect of noise and smoothing on the accuracy of the estimation of translation parameters. In general, noise causes a perturbation in the location of the global minimum of the distance function. The perturbed version of the single global minimum of the noiseless distance function $f$ will remain in the form of a single global minimum for the noisy distance function with high probability if the noise level is sufficiently small. The noise similarly introduces a perturbation on the SIDEN as well. The exact derivation of the SIDEN in the noisy setting requires the examination of the first zero-crossings of the derivative of the noisy distance function along arbitrary directions $T$ around its global minimum. At small noise levels, these zero-crossings are expected to be perturbed versions of the first zero-crossings of $df(tT)/dT$ around the origin, which define the boundary of the noiseless SIDEN $\siden$. The perturbation on the zero-crossings depends on the noise level. If the noise level is sufficiently small, the perturbation on the zero-crossings will be smaller than the distance between $\siden$ and its estimate $\estsiden$. This is due to the fact that $\estsiden$ is a worst-case estimate for $\siden$ and its boundary is sufficiently distant from the boundary of $\siden$ in practice, which is also confirmed by the experiments in Section \ref{sec:exp_res}. In this case, the estimate $\estsiden$ obtained from the noiseless distance function $f$ is also a subset of the noisy SIDEN. Therefore, under the small noise assumption, $\estsiden$ can be considered as an estimate of the noisy SIDEN as well and it can be used in the alignment of noisy images in practice.\footnote{The validity of this approximation is confirmed by the numerical simulation results in Section \ref{sec:exp_res}.} Our alignment analysis in this section relies on this assumption. Since we consider that the reference and target patterns are aligned with a descent-type optimization method, the solution will converge to the global minimum of the noisy distance function in the noisy setting. The alignment error is then given by the change in the global minimum of the distance function, which we analyze now.
 
The selection of the noise model for the representation of the deviation of the target pattern from the translation manifold of the reference pattern depends on the imaging application. It is common practice to represent non-predictable deviations of the image intensity function from the image model with additive Gaussian noise. This noise model fits well the image intensity variations due to imperfections of the image capture system, sensor noise, etc. Meanwhile, in some settings, one may have a prior knowledge of the type of the deviation of the target image from the translation manifold of the reference image. For instance, the deviation from the translation manifold may be due to some geometric image deformations, non-planar scene structures, etc. In such settings, one may be able to bound the magnitude of the deviation of the image intensity function from the translation model. Considering these, we examine two different noise models in our analysis. We first focus on a setting where the target pattern is corrupted with respect to an analytic noise model in the continuous space $L^2(\Rsq)$. The analytic noise model is inspired by the i.i.d.~Gaussian noise in the discrete space $\mathbb{R}^n$. In Section \ref{ssec:trans_acc_nofilt}, we derive a probabilistic upper bound on the alignment error for this setting in terms of the parameters of the reference pattern and the noise model. Then, in Section \ref{ssec:trans_acc_gennoise}, we generalize the results of Section \ref{ssec:trans_acc_nofilt} to arbitrary noise patterns in $L^2(\Rsq)$ and derive an error bound in terms of the norm of the noise pattern. \revis{In Section \ref{ssec:adapt_settings} we discuss how these results can be adapted to a setting where both the reference image and the target image are corrupted with noise}. Lastly, the influence of smoothing the reference and target patterns on the alignment error is studied in Section \ref{ssec:trans_acc_filt}.

Throughout Section \ref{sec:noise_anly}, we use the notations $\overline{(\cdot)}$ and $\underline{(\cdot)}$ to refer respectively to upper and lower bounds on the variable $(\cdot)$. The parameters corresponding to smoothed patterns are written as $\hat{(.)}$ as in Section \ref{ssec:SIDENregular}. The notations $R_{(.)}$ and $C_{(.)}$ are used to denote important upper bounds appearing in the main results, which are  associated with the parameter in the subscript.

\subsection{Derivation of an Upper Bound for Alignment Error for Gaussian Noise}
\label{ssec:trans_acc_nofilt}

We consider the noiseless reference pattern $p$ in (\ref{eq:pXgaussAtoms}) and a target pattern that is a noisy observation of a translated version of $p$. We assume an analytical noise model given by
\begin{equation}
w(X)=\sum_{l=1}^L \zeta_l \, \phi_{\xi_l}(X),
\label{eq:analy_noise_model}
\end{equation}
where the noise units $\phi_{\xi_l}(X)$ are Gaussian atoms of scale $\epsilon$. The coefficients $\zeta_l$ and the noise atom parameters $\xi_l$ are assumed to be independent. The noise atoms are of the form
$
\phi_{\xi_l}(X) = \phi \big(E^{-1} (X-\delta_l) \big)
$
where
\[   
E=\left[
\begin{array}{c c}
\epsilon & 0  \\
 0 & \epsilon
\end{array} \right], \, \, \,
\delta_l= \left[
\begin{array}{c}
\delta_{x,l} \\
\delta_{y,l}
\end{array} \right].
\]
The vector $\delta_l$ is the random translation parameter of the noise atom $\phi_{\xi_l}$ such that the random variables $\{ \delta_{x,l} \}_{l=1}^L, \{ \delta_{y,l} \}_{l=1}^L \sim U[-\bparam, \bparam]$ have an i.i.d.~uniform distribution. Here, $\bparam$  is a fixed parameter used to define a region $[-\bparam, \bparam ] \times [-\bparam, \bparam ] \subset \Rsq $ in the image plane, which is considered as a support region capturing a substantial part of the energy of reference and target images. The centers of the noise atoms are assumed to be uniformly distributed in this region. In order to have a realistic noise model, the number of noise units $L\gg K$ is considered to be a very large number and the scale $\epsilon>0$ of noise atoms is very small. The parameters $L$ and $\epsilon$ will be treated as noise model constants throughout the analysis. The coefficients $\zeta_l \sim N(0, \eta^2)$ of the noise atoms are assumed to be i.i.d.~with a normal distribution of variance $\eta^2$.

The continuous-space noise model $w(X)$ is chosen in analogy with the digital i.i.d.~Gaussian noise in the discrete space $\mathbb{R}^n$. The single isotropic scale parameter $\epsilon$ of noise units bears resemblance to the  1-pixel support of digital noise units. The uniform distribution of the position $\delta_l$ of noise units is similar to the way digital noise is defined on a uniform pixel grid. The noise coefficients $\zeta_l$ have an i.i.d.~normal distribution as in the digital case. If our noise model $w(X)$ has to approximate the digital Gaussian noise in a continuous setting, the noise atom scale $\epsilon$ is chosen comparable to the pixel width and  $L$ corresponds to the resolution of the discrete image.

Let now $p_n$ be a noisy observation of $p$ such that
$
p_n(X)=p(X)+ w(X) 
$, 
where $w$ and $p$ are independent according to the noise model (\ref{eq:analy_noise_model}). We assume that the target pattern is a translated version of $p_n(X)$ so that it takes the form $p_n(X-tT)$. Then, the noisy distance function between $p(X)$ and $p_n(X-tT)$ is given by
\begin{equation}
g(tT)= \int_{\Rsq} \big( p(X)- p_n(X-tT) \big)^2 \, dX   =  \int_{\Rsq} \big( p(X)- p(X-tT) - w(X-tT)\big)^2 \, dX. 
\label{eq:defngtT}
\end{equation}
This can be written as
$
g(tT)=  f(tT) + h(tT)
$,
where
\begin{equation}
\begin{split}
h(tT)  &:=  - 2  \int_{\Rsq} \big( p(X)- p(X-tT) \big) w(X-tT)  \, dX
        + \int_{\Rsq}  w^2(X-tT) \, dX.
\end{split}
\label{eq:htTdefn}
\end{equation}
The function $h$ represents the deviation of $g$ from $f$. We call $h$ the distance deviation function. The expected value of $h$ is independent of the translation $tT$ and given by 
$
\mu_h := E[h(tT)] = \frac{\pi}{2}  L \eta^2 \epsilon^2
$
where $E[.]$ denotes the expectation \cite[Appendix B.1]{Vural12TR}. Therefore,
$
E[g(tT)] =  f(tT) + \mu_h
$
and the global minimum of $E[g(tT)]$ is at $tT=0$. However, due to the probabilistic perturbation caused by the noise $w$, the global minimum of $g$ is not at $tT=0$ in general. We consider $g$ to have a single global minimum and denote its location by $\tming \Tming$. However, the single global minimum assumption is not a strict hypothesis of our analysis technique; i.e., the upper bound that we derive for the distance between $\tming \Tming$ and the origin is still valid if $g$ has more than one global minimum. In this case, the obtained upper bound is valid for all global minima.

We now continue with the derivation of a probabilistic upper bound on the distance $\tming$ between the location $\tming \Tming$ of the global minimum of $g$ and the location $0$ of the global minimum of $f$. We show in \cite[Appendix B.2]{Vural12TR} that $\tming$ satisfies the equation
\begin{equation}
\begin{split}
\frac{\tming^2}{2} \left(
	  \frac{d^2 f(t \Tming)}{d t^2} \bigg|_{t=\tintg} 
       + \frac{d^2 f(t \Tming)}{d t^2} \bigg|_{t=\tintf} 
       + \frac{d^2 h(t \Tming)}{d t^2} \bigg|_{t=\tintg} 
\right)
	= | h(0)  - h(\tming \Tming) |
\end{split}
\label{eq:maineq_form1}
\end{equation}\\
for some $\tintg \in [0, \tming]$ and $\tintf \in [0, \tming]$. Our derivation of an upper bound for $\tming$ will be based on (\ref{eq:maineq_form1}). The above equation shows that $\tming$ can be upper bounded by finding a lower bound on the term
\begin{equation}
	  \frac{d^2 f(t \Tming)}{d t^2} \bigg|_{t=\tintg} 
       + \frac{d^2 f(t \Tming)}{d t^2} \bigg|_{t=\tintf} 
       + \frac{d^2 h(t \Tming)}{d t^2} \bigg|_{t=\tintg} 
\label{eq:LHS_term_t0}       
\end{equation}      
and an upper bound for the term $| h(0)  - h(\tming \Tming) |$. However, $h$ is a probabilistic function; i.e., $h(tT)$ and its derivatives are random variables. Therefore, the upper bound that we will obtain for $\tming$ is a probabilistic bound given in terms of the variances of $ h(0)  - h(\tming \Tming) $ and $d^2 h(t \Tming)/d t^2 $. 

In the rest of this section, we proceed as follows. First, in order to be able to bound $| h(0)  - h(tT) |$ probabilistically, in Lemma \ref{cor:var_deltah_unif}  we present an upper bound on the variance of $h(0)  - h(tT)$. Next, in order to bound the term in (\ref{eq:LHS_term_t0}), we state a lower bound for $d^2 f(t T) / d t^2$ in Lemma \ref{lemma:d2f_dt2_lb} and an upper bound for the variance of $d^2 h(t T) /d t^2$ in Lemma \ref{cor:bndunif_d2hdt2}. These results are finally put together in the main result of this section, namely Theorem \ref{theo:t0bound}, where an upper bound on $\tming$ is obtained based on (\ref{eq:maineq_form1}). Theorem \ref{theo:t0bound} applies Chebyshev's inequality to employ the bounds derived in Lemmas \ref{cor:var_deltah_unif} and \ref{cor:bndunif_d2hdt2} to define probabilistic upper bounds on the terms $| h(0)  - h(\tming \Tming) |$ and $| d^2 h(t \Tming) /d t^2 |$. Then, this is combined with the bound on $d^2f(tT)/dt^2$ in Lemma \ref{lemma:d2f_dt2_lb} to obtain a probabilistic upper bound on $\tming$ from the relation (\ref{eq:maineq_form1}).

In the derivation of this upper bound, the direction $\Tming$ of the global minimum of $g$ is treated as an arbitrary and unknown unit-norm vector. Moreover, the variances of $ h(0)  - h(t T) $ and $d^2 h(t T)/d t^2 $ have a complicated dependence on $t$, which makes it difficult to use them directly in (\ref{eq:maineq_form1}) to obtain a bound on $\tming$. In order to cope with the dependences of these terms on $t$ and $T$, the upper bounds presented in  Lemmas \ref{cor:var_deltah_unif} and \ref{cor:bndunif_d2hdt2} are derived as uniform upper bounds over the closed ball of radius $\tbound>0$, $B_{\tbound}(0) = \{tT: T\in S^1, 0 \leq t \leq \tbound \}$. The upper bounds are thus independent of $tT$ and valid for all $tT$ vectors in $B_{\tbound}(0)$. In these lemmas, the parameter $\tbound$ is considered to be a known threshold for $\tming$, such that $\tming \leq \tbound$. This parameter will be assigned a specific value in Theorem \ref{theo:t0bound}.

We begin with bounding the variance of term $h(0)  - h(tT) $ in order to find an upper bound for the right hand side of (\ref{eq:maineq_form1}). Let us denote 
$
\Delta h(t T) := h(0)  - h(t T). 
$
From (\ref{eq:htTdefn}), 
$
\Delta h(t T) = h(0)  - h(t T)   = 2   \intRsq \left(  p(X) - p(X-tT)  \right) w(X-tT) dX 
$,
where we have used the fact that $\intRsq w^2(X-tT) dX = \intRsq w^2(X) dX$.  Let $\sigma^2_{\Delta h(t T)}$ denote the variance of $\Delta h(t T)$. In the following lemma, we state an upper bound on  $\sigma^2_{\Delta h(t T)}$. Let us define beforehand the following constants for the $k$-th atom of $p$ 
\begin{eqnarray*}
\Phi_k :=\Psi_k (\sigma_k^2+ E^2)^{-1}  \Psi_k^{-1},
\qquad
\ck :=  \frac{\pi \, |\sigma_{k}| \,  |E| } { \sqrt{  |  \sigma_k^2 + E^2  | } } . 
\end{eqnarray*}
Also, let $J^{- }= \{ (j,k): \coef_j  \coef_k <0 \}$ and $J^{+}= \{ (j,k): \coef_j  \coef_k >0 \}$ denote the set of $(j,k)$ indices with negative and positive coefficient products.

\begin{lemma}
\label{cor:var_deltah_unif}
Let $\tbound>0$, and let $tT \in B_{\tbound}(0)$. Then, the variance  $\sigma^2_{\Delta h(tT)}$ of $\Delta h(tT)$ can be upper bounded as
\begin{equation}
\begin{split}
 \sigma^2_{\Delta h(tT)} <   \bndvardeltahUnif   
     :=  \cthree   \eta^2
\end{split}	 
\label{eq:cor_var_g_unif}
\end{equation}
where
\begin{equation*}
\cthree :=  4 \, L  \,\, \bigg(  
	  \sum_{(j,k) \in J^{+}}    \coef_j \cj \, \coef_k  \ck  \ \cdeltahposjkUB  
     +    \sum_{(j,k) \in J^{-}}     \coef_j \cj \, \coef_k  \ck  \  \cdeltahnegjkLB
        \bigg).
\end{equation*}
Here the terms $\cdeltahposjkUB$ and $\cdeltahnegjkLB$  are constants depending on  $\tbound$ and the atom parameters of $p$. In particular,  $\cdeltahposjkUB$ and $\cdeltahnegjkLB$ are bounded functions of $\tbound$, given in terms of exponentials of second-degree polynomials of $\tbound$ with negative leading coefficients.

\end{lemma}

The proof of Lemma \ref{cor:var_deltah_unif} is presented in \cite[Appendix B.4]{Vural12TR}. In the proof, a uniform upper bound $\cdeltahposjkUB$ and a uniform lower bound $\cdeltahnegjkLB$ are derived for the additive terms\footnote{$\cdeltahposjkUB$ and $\cdeltahnegjkLB$ are upper and lower bounds for the terms $\cdeltahposjk$, $\cdeltahnegjk$ used in \cite[Lemma 1]{Vural12TR}.} constituting the variance of $\Delta h(tT)$. The exact expressions of $\cdeltahposjkUB$ and $\cdeltahnegjkLB$ are given in Appendix \ref{app:der_cdeltahs}.

We have thus stated a uniform upper bound $\bndvardeltahUnif$ for the variance of $\Delta h(tT)$ which will be used to derive an upper bound for the right hand side of (\ref{eq:maineq_form1}) in Theorem \ref{theo:t0bound}.  We now continue with the examination of the left hand side of  (\ref{eq:maineq_form1}). We begin with the term $d^2 f(tT) / dt^2$. The following lemma gives a lower bound on the second derivative of the noiseless distance function $f(tT)$ in terms of the pattern parameters.

\begin{lemma} 
\label{lemma:d2f_dt2_lb}
The second derivative of $f(tT)$ along the direction $T$ can be uniformly lower bounded for all $t \in [0, \tzero]$ and for all directions $T \in S^1$ as follows
\begin{equation}
\frac{d^2 f(tT)}{dt^2} \geq \betazerolb + \betatwolb \tzero^2 + \betathreelb  \tzero^3.
\label{eq:bnd_secderf}
\end{equation}
Here $\betazerolb>0$, $\betatwolb \leq 0$, and $\betathreelb<0$ are constants depending on the atom parameters of $p$. In particular, $\betazerolb$, $\betatwolb $, $\betathreelb$ are obtained from the eigenvalues of some matrices derived from the parameters $\coef_j$, $\tau_j$, $Q_{jk}$, $\Sigma_{jk}$.
\end{lemma}

The proof of Lemma \ref{lemma:d2f_dt2_lb} is given in \cite[Appendix B.5]{Vural12TR} and the exact expressions of $\betazerolb$, $\betatwolb$, $\betathreelb$ are given in Appendix \ref{app:defn_betalbs}. The above lower bound on the second derivative of $f(tT)$ is independent of the direction $T$ and the amount $t$ of translation, provided that $t$ is in the interval $[0, \tzero]$. In fact, the statement of Lemma \ref{lemma:d2f_dt2_lb} is general in the sense that $\tzero$ can be any positive scalar. However, in the proof of Theorem \ref{theo:t0bound}, we use Lemma \ref{lemma:d2f_dt2_lb} for the $\tming$ value that represents the deviation between the global minima of $f$ and $g$.

The result of Lemma  \ref{lemma:d2f_dt2_lb} will be used in Theorem \ref{theo:t0bound} in order to lower bound the second derivative of $f$ in (\ref{eq:maineq_form1}). We now continue with the term $d^2 h(tT)/dt^2$ in (\ref{eq:maineq_form1}). Let $\secderh:=d^2 h(tT)/dt^2$ denote the second derivative of the deviation function $h$ along the direction $T$. Since $\secderh$ can take both positive and negative values, in the calculation of a lower bound for the term (\ref{eq:LHS_term_t0}), we need a bound on the magnitude $| \secderh |$ of this term. It can be bounded probabilistically in terms of the variance of $\secderh$. We thus state the following uniform upper bound on the variance of $ \secderh $.

\begin{lemma}
\label{cor:bndunif_d2hdt2}
Let $\tbound>0$, and let $tT \in B_{\tbound}(0)$. Then, the variance $\sigma^2_{\secderh} $ of $\secderh$ can be upper bounded as
\begin{equation}
\begin{split}
\sigma^2_{\secderh} < \bndvargUnifsecderh := \cfour \eta^2
\end{split}
\label{eq:cor_bndunif_d2hdt2}
\end{equation}
where
\begin{equation*}
\begin{split}
\cfour := 4 L \bigg(
 \sum_{(j,k)\in J^{+}}  \coef_j  \coef_k  \cj \ck \ \csecderhposjkUB
	+  \sum_{(j,k)\in J^{-}}  \coef_j  \coef_k  \cj \ck  \ \csecderhnegjkLB
	\bigg).	
\end{split}
\end{equation*}
Here $\csecderhposjkUB$ is a constant depending on the atom parameters of $p$; and the term $\csecderhnegjkLB$ depends on the atom parameters of $p$ and $\tbound$. In particular, $\csecderhposjkUB$ is given in terms of rational functions of the eigenvalues of $\Phi_k$ matrices; and $\csecderhnegjkLB$ is a bounded function of $\tbound$ given in terms of  exponentials of second-degree polynomials of $\tbound$ with negative leading coefficients. 
\end{lemma}

The proof of Lemma \ref{cor:bndunif_d2hdt2} is given in \cite[Appendix B.7]{Vural12TR}. The proof derives uniform upper and lower bounds $\csecderhposjkUB$, $\csecderhnegjkLB$ for the additive terms\footnote{$\csecderhposjkUB$ and $\csecderhnegjkLB$ are upper and lower bounds for the terms $\csecderhposjk$, $\csecderhnegjk$ used in \cite[Lemma 3]{Vural12TR}.} in the representation of $\sigma^2_{\secderh} $. The exact expressions for $\csecderhposjkUB$ and $\csecderhnegjkLB$ are given in Appendix \ref{app:defn_csecderhs}.

Now we are ready to give our main result about the bound on the alignment error. The following theorem states an upper bound on the distance between the locations of the global minima of $f$ and $g$ in terms of the noise standard deviation $\eta$ and the atom parameters of $p$, provided that $\eta$ is smaller a  threshold $\etabound$. The threshold $\etabound$ is obtained from the bounds derived in Lemmas \ref{cor:var_deltah_unif}, \ref{lemma:d2f_dt2_lb} and \ref{cor:bndunif_d2hdt2} such that the condition $\eta < \etabound$ guarantees that the assumption $\tming < \tbound$ holds. In the theorem, the parameter $\tbound$, which is treated as a predefined threshold on $\tming$ in the previous lemmas, is also assigned a specific value in terms the constants $\betazerolb$, $\betatwolb$, $\betathreelb$ of Lemma \ref{lemma:d2f_dt2_lb}.\\

\begin{theorem}
\label{theo:t0bound}
Let
\begin{equation}
 \tboundA := \sqrt{\frac{\betazerolb}{ 2  | \betatwolb | + 2^{2/3}  \betazerolb^{1/3} |\betathreelb |^{2/3} }}.
\label{eq:defntboundA} 
\end{equation}
Let
$\bndstdevdeltahUnif := \sqrt{\bndvardeltahUnif}$ and  $\bndstdevUnifsecderh := \sqrt{\bndvargUnifsecderh}$, where $\bndvardeltahUnif$ and $\bndvargUnifsecderh$ are as defined in (\ref{eq:cor_var_g_unif}) and (\ref{eq:cor_bndunif_d2hdt2}), and evaluated at the value of $\tboundA$  given above. Also, let
$\sqrtcthree:= \sqrt{\cthree}$ and $\sqrtcfour := \sqrt{\cfour}$.

Assume that for some $s>\sqrt{2}$, the noise standard deviation $\eta$ is smaller than $\etabound$ such that 
\begin{equation}
\eta \leq \etabound := \frac{\tboundA^2 \, \betazerolb }{ 2 \, s \, \sqrtcthree  + \, \tboundA^2 \, s \, \sqrtcfour }.
\label{eq:defn_etabound}
\end{equation}
Then, with probability at least $1-\frac{2}{s^2}$, the distance $\tming$ between the global minima of $f$ and $g$ is bounded as 

\begin{equation}
\tming < \tboundB :=  
 \sqrt{ 
	 \frac{ 2  s \, \bndstdevdeltahUnif  }
	{ \betazerolb - s \, \bndstdevUnifsecderh }	
   		  }.
\label{eq:tboundB_defn}
\end{equation}

\end{theorem}

The proof of Theorem \ref{theo:t0bound} is given in  \cite[Appendix B.8]{Vural12TR}. In the proof, we make use of the upper bounds $\bndvardeltahUnif$, $\bndvargUnifsecderh$ on $\sigma^2_{\Delta h(tT)}$, $\sigma^2_{\secderh}$, and the lower bound on $d^2f(tT)/dt^2$ given in (\ref{eq:bnd_secderf}). The upper bound $\tboundB$ in (\ref{eq:tboundB_defn}) shows that the alignment error increases with the increase in the noise level, since $\bndstdevdeltahUnif$ and $\bndstdevUnifsecderh$ are linearly proportional to the noise standard deviation $\eta$. The increase of the error with the noise is expected. It can also be seen from (\ref{eq:tboundB_defn}) that the increase in the term $\betazerolb$, which is proportional to the second derivative of the noiseless distance $f$, reduces the alignment error; whereas an increase in the term $\bndstdevUnifsecderh$, which is related to the second derivative of $h$, increases the error. This can be explained as follows. If $f$ has a sharp increase around its global minimum at $0$, i.e., $f$ has a large second derivative, the location of its minimum is less affected by $h$. Likewise, if the distance deviation function $h$ has a large second derivative, it introduces a larger alteration around the global minimum of $f$, which causes a bigger perturbation on the position of the minimum.

Theorem \ref{theo:t0bound} states a bound on $\tming$ under the condition that the noise standard deviation $\eta$ is smaller than the threshold value $\etabound$, which depends on the pattern parameters (through the terms $\betazerolb$, $\betatwolb$, $\betathreelb$, $\tbound$) as well as the noise parameters $L$ and $\epsilon$ (through the terms $\sqrtcthree$ and $\sqrtcfour$). The threshold $\etabound$ thus defines an admissible noise level such that the change in the location of the global minimum of $f$ can be properly upper bounded. This admissible noise level is derived from the condition $\tboundB \leq \tboundA$, which is partially due to our proof technique. However, we remark that the existence of such a threshold is intuitive in the sense that it states a limit on the noise power in comparison with the signal power. Note also that the denominator $\betazerolb - s \, \bndstdevUnifsecderh$ of $\tboundB$ should be positive, which also yields a condition on the noise level 
\begin{equation*}
\eta < \etabound' = \frac{ \betazerolb}{ s \, \sqrtcfour}.
\end{equation*}
However, this condition is already satisfied due to the hypothesis $\eta \leq \etabound$ of the theorem, since $\etabound < \etabound'$ from (\ref{eq:defn_etabound}). \revis{Lastly, the fact that the theoretical upper bound tends to infinity at large values of the noise level should be interpreted in the way that the alignment error bound is not informative at high noise levels. This is mainly because when the noise level gets too high, the noise component becomes the determining factor in the estimation of the translation parameters rather than the noiseless component of the image. As this estimate is unreliable, no theoretical guarantee can be obtained for the performance of registration algorithms.}


\subsection{Generalization of the Alignment Error Bound to Arbitrary Noise Models}
\label{ssec:trans_acc_gennoise}

Here, we generalize the results of the previous section in order to derive an alignment error bound for arbitrary noise patterns. In general, the characteristics of the noise pattern vary depending on the imaging application. In particular, while the noise pattern may have high correlation with the reference pattern in some applications (e.g., noise resulting from geometric deformations of the pattern), its correlation with the reference pattern may be small in some other settings where the noise stems from a source that does not depend on the image. We thus focus on two different scenarios. In the first and general setting, we do not make any assumption on the noise characteristics and bound the alignment error in terms of the norm of the noise pattern. Then, in the second setting, we consider that the noise pattern has small correlation with the points on the translation manifold of the reference pattern and show that the alignment error bound can be made sharper in this case.

We assume that the reference pattern $p(X)$ is noiseless and we write the target pattern as  $p_g(X-tT)$, where $p_g(X)=p(X) + z(X)$ is a generalized noisy observation of $p$ such that $z \in L^2(\Rsq)$ is an arbitrary noise pattern. Then, the generalized noisy distance function is
\begin{equation*}
g_g(tT)= \int_{\Rsq} \big( p(X)- p_g(X-tT) \big)^2 \, dX
\end{equation*}
and the generalized deviation function is $h_g= g_g(tT) - f(tT)$. Let us call $\tminggen \Tminggen$ the point where $g_g$ has its global minimum. Then the distance between the global minima of $g_g$ and $f$ is given by $\tminggen$.

We begin with the first setting and state a generic bound for the alignment error $\tminggen$ in terms of the norm of the noise $\nu := \| z \|$. In our main result, we denote by $\normp := \| p \|$ the norm of the pattern $p$, and make use of an upper bound $\bndnormsecderp$ for the norm $\| d^2 p(\cdot+tT) / dt^2  \|$ of the second derivative of $p(X+tT)$. It has been derived in terms of the atom parameters of $p$ in \cite[Lemma 4]{Vural12TR}. We state below our generalized alignment error result for arbitrary noise patterns.

\begin{theorem}
\label{thm:alg_error_bndgen}
Let $\tboundA$ be defined as in (\ref{eq:defntboundA}). Assume that the norm $\nu$ of $z$ is smaller than $\nu_0$ such that
\begin{equation}
\nu \leq \nu_0 := \frac{ \tboundA^2  \betazerolb }{8 \normp + 2 \bndnormsecderp \tboundA^2}
\label{eq:nubound_gen}
\end{equation}
where $\betazerolb$ is the constant in Lemma \ref{lemma:d2f_dt2_lb}. Then, the distance $\tminggen$ between the global minima of $f$ and $g_g$ is bounded as 
\begin{equation}
\label{eq:defn_tboundBgen}
\tminggen \leq \tboundBgen :=\sqrt{ \frac{8  \normp \nu}{ \betazerolb - 2 \bndnormsecderp \nu }  }.
\end{equation}
\end{theorem}

Theorem \ref{thm:alg_error_bndgen} is proved in \cite[Appendix C.2]{Vural12TR}. The theorem states an upper bound on the alignment error for the general case where the only information used about the noise pattern is its norm. The alignment error bound $\tboundBgen$ is a generalized and deterministic version of the probabilistic bound $\tboundB$ derived for the Gaussian noise model. In the proof of the theorem, the change $h_g(0) - h_g(\tminggen \Tminggen)$ in the deviation function is bounded by $4 \normp \nu$. The second derivative of the noiseless distance function $f$ is captured by $\betazerolb$ as in the Gaussian noise case. Finally, the term $2 \bndnormsecderp \nu$ bounds the second derivative of the deviation $h_g$. Based on these, the above result is obtained by following similar steps as in Section \ref{ssec:trans_acc_nofilt}.

We now continue with the second setting where the noise pattern $z$ has small correlation with the points on the translation manifold $\mathcal{M}(p)$ of $p$. We characterize the correlation of two patterns with their inner product. Assume that a uniform correlation upper bound $\bndcorr$ is available such that
\begin{equation}
\left|  \intRsq  p(X+tT) z(X) dX \right| \leq \bndcorr
\label{eq:bndcorr_defn}
\end{equation}
for all $t$ and $T$. The following corollary builds on Theorem \ref{thm:alg_error_bndgen} and states that the bound on the alignment error can be made sharper if the correlation bound is sufficiently small.

\begin{corollary} 
\label{cor:align_error_uncor}
Let $\tboundA$ be defined as in (\ref{eq:defntboundA}) and let a uniform upper bound $\bndcorr$ for the correlation be given such that
$
\bndcorr <  \tboundA^2 \betazerolb / 8
$.

Assume that the norm $\nu$ of $z$ is smaller than $\nu_0$ such that
\begin{equation}
\nu \leq \nu_0 := \frac{ \tboundA^2  \betazerolb   - 8 \bndcorr }{ 2 \bndnormsecderp \tboundA^2}.
\label{eq:nubound_uncor}
\end{equation}
Then, the distance $\tminggen$ between the global minima of $f$ and $g_g$ is bounded as 
\begin{equation}
\label{eq:defn_tboundBgenuncor}
\tminggen \leq \tboundBgenuncor :=\sqrt{ \frac{8 \,  \bndcorr }{ \betazerolb - 2 \bndnormsecderp \nu }  }.
\end{equation}
\end{corollary}

The proof of Corollary \ref{cor:align_error_uncor} is given in \cite[Appendix C.3]{Vural12TR}. One can observe that the alignment error bound $\tboundBgenuncor$  approaches zero as the uniform correlation bound approaches zero. Therefore, if $\bndcorr$ is sufficiently small, $\tboundBgenuncor$ will be smaller than the general bound $\tboundBgen$. This shows that, regardless of the noise level, the alignment error is close to zero if the noise pattern $z$ is almost orthogonal to the translation manifold $\mathcal{M}(p)$ of the reference pattern.\\ 

\revis{\textbf{Remark.} In the alignment error bounds derived in this paper, we assume that entire images are used in the registration rather than windowed regions in the reference and target images. If the reference and target images are windowed in an image registration application, this also causes a perturbation in the distance function, in addition to the perturbation caused by the image noise. In \cite[Appendix C.4]{Vural12TR}, we examine how windowing affects the alignment error. The analysis in \cite[Appendix C.4]{Vural12TR} shows that the increase in the alignment error due to windowing can be mitigated by choosing a window function with a sufficiently slow variation.}\\

\subsection{Adaptation of the error bounds for two-sided noise models}
\label{ssec:adapt_settings}

\revis{Throughout our analysis, we have assumed that the reference image is noiseless and the image noise ($w$ or $z$) acts only on the target image. However, in practice the reference image can also be contaminated with noise. We now discuss how our results can be adapted to this case. 

First, it is easy to show that one gets exactly the same error bounds if   the noisy distance function $g(tT)$ in (\ref{eq:defngtT}) is slightly modified as
\begin{equation}
\intRsq \big( p(X)-p(X-tT) - w(X) \big)^2 dX.
\label{eq:altdefn_noisy_dist}
\end{equation}
This corresponds to the situation where the target image $p(X-tT)+w(X)$ is first translated and then corrupted by the additive noise $w$. Now, if the reference and the target images are corrupted respectively with the noise instances $w_1$ and $w_2$, the distance between the reference image $p(X) + w_1(X)$ and the target image $p(X-tT) + w_2(X) $ is given by
\[
\intRsq \big( p(X) - p(X-tT) - (w_2(X) - w_1(X) )  \big)^2 dX.
\]
One can observe that this distance function is the same as the distance function in (\ref{eq:altdefn_noisy_dist}) with $w=w_2 - w_1$. Therefore, if $w_1$ and $w_2$ have identical distributions conforming to the Gaussian noise model in (\ref{eq:analy_noise_model}), the compound noise $w=w_2 - w_1$ can also be represented with the same model such that the number of atoms $L$ is twice the number of atoms in $w_1$ and $w_2$, and the other model parameters are the same. Hence, one can easily adapt the main result in Theorem \ref{theo:t0bound} to obtain an alignment error bound for the two-sided noise model. We observe from Lemmas \ref{cor:var_deltah_unif} and  \ref{cor:bndunif_d2hdt2} that doubling the number of atoms $L$ would increase the terms $\sqrtcthree$, $\sqrtcfour$, $\bndstdevdeltahUnif$, $\bndstdevUnifsecderh$ by a factor of $\sqrt{2}$. The alignment error bound for the two-sided Gaussian noise model is thus given as
\[
\tming < \tboundB =  
 \sqrt{ 
	 \frac{ 2 \sqrt{2} s \, \bndstdevdeltahUnif  }
	{ \betazerolb - s  \sqrt{2} \, \bndstdevUnifsecderh }	
   		  }
\]
which holds for the admissible noise threshold
\[
\etabound = \frac{\tboundA^2 \, \betazerolb }{ 2 \sqrt{2} \,  s \, \sqrtcthree  + \sqrt{2} \ \tboundA^2 \, s \, \sqrtcfour }
\]
with the same probability as in Theorem \ref{theo:t0bound}.

Next, the analysis in Section \ref{ssec:trans_acc_gennoise} for arbitrary noise patterns can also be adapted to the two-sided noise assumption in a similar way. The distance function between the noisy reference pattern $p(X) + z_1(X)$ and the noisy target pattern $p(X-tT) + z_2(X)$ is equal to 
\[
\intRsq \left( p(X) - p(X-tT) - z(X)  \right)^2 dX
\]
where $z(X) = z_2(X) - z_1(X)$. If the norms of $z_1$ and $z_2$ are bounded by $\nu$, the norm of $z$ is bounded by $2\nu$ by the triangle inequality. Increasing the norm of the noise level $\nu$ by a factor of $2$ in Theorem \ref{thm:alg_error_bndgen}, we get the following alignment error bound for the two-sided noise model
\[
\tminggen \leq \tboundBgen =\sqrt{ \frac{16  \normp \nu}{ \betazerolb - 4 \bndnormsecderp \nu }  }
\]
which holds if
\[
\nu \leq \nu_0 = \frac{ \tboundA^2  \betazerolb }{16 \normp + 4 \bndnormsecderp \tboundA^2}.
\]

}

\subsection{Influence of Filtering on Alignment Error}
\label{ssec:trans_acc_filt}

In this section, we examine how the alignment error resulting from the image noise is affected when the reference and target patterns are low-pass filtered. We consider the Gaussian kernel $\frac{1}{\pi \rho^2} \phi_{\rho}(X)$ defined in Section \ref{ssec:SIDENregular}  and analyze the dependence of the alignment error bounds obtained for the Gaussian noise and generalized noise models in Sections  \ref{ssec:trans_acc_nofilt} and \ref{ssec:trans_acc_gennoise} on the filter size $\rho$ and the noise level parameters $\eta$ and $\nu$.
 
We begin with the Gaussian noise model $w(X)$. The filtered reference pattern $\hat{p}(X)$ and the filtered noisy observation $\hat{p}_n(X)$ of the reference pattern are given by
\begin{eqnarray*}
\hat{p}(X) = \sumk \hat{\coef}_k \phi_{\hat{\gamma}_k}(X),
\qquad
\hat{p}_n(X) = \hat{p}(X) + \hat{w}(X) 
		= \sumk \hat{\coef}_k \phi_{\hat{\gamma}_k}(X) 
		+  \sum_{l=1}^L \hat{\zeta}_l \, \phi_{\hat{\xi}_l}(X).
\end{eqnarray*}

Remember from Section \ref{ssec:SIDENregular} that the rotation and translation parameters of the atoms of $\hat{p}(X)$ do not depend on $\rho$; and the scale matrices vary with $\rho$ such that $\hat{\sigma}_k^2 =  \sigma_k^2 + \filtsc^2$. The parameters of the smoothed noise atoms can be obtained similarly to the parameters of the atoms in $\hat{p}$; i.e., 
$
\phi_{\hat{\xi}_l}(X) = \phi \big(\hat{E}^{-1} (X-\delta_l) \big)
$, 
where $\hat{E}^2 =  E^2 + \filtsc^2 $. This gives the scale parameter of smoothed noise atoms, which is written as
\begin{equation}
\hatepsilon = \sqrt{\epsilon^2 + \rho^2}.
\label{eq:epsilonhatdefn}
\end{equation}
The smoothed noise coefficients are given by
$
\hat{\zeta}_l =  \zeta_l \, | E | / | \hat{E} |   =  \zeta_l \, \epsilon^2 / (\epsilon^2 + \rho^2)
$.
Since all the coefficients $\zeta_l $ are multiplied by a factor of $\epsilon^2 / (\epsilon^2 + \rho^2)$, the variance of the smoothed noise atom coefficients is
\begin{equation}
\hat{\eta}^2 = \left( \frac{\epsilon^2}{\epsilon^2 + \rho^2} \right)^2 \eta^2.
\label{eq:etahatdefn}
\end{equation}

As the noise atom units are considered to have very small scale, one can assume first that $\rho \gg \epsilon$ for typical values of the filter size $\rho$. Then the relations in (\ref{eq:epsilonhatdefn}) and (\ref{eq:etahatdefn}) give the joint variations of $\hatepsilon$ and $\hateta$ with $\eta$ and $\rho$ as
\begin{equation}
\hatepsilon  = O(\rho), 
\qquad \qquad
\hateta=O(\eta \rho^{-2}).
\label{eq:ord_hatepsilon_hateta}
\end{equation}
%

%
%

We now state the dependence of the bound $\hattboundB$ on $\rho$ and  $\eta$ in the following main result.\\

\begin{theorem}
\label{thm:order_bound_t0}
The joint variation of the alignment error bound $\hattboundB$ for the smoothed image pair with respect to $\eta$ and $\rho$ is given by
\begin{equation*}
\hattboundB = O\left(  \sqrt{ 
	 \frac{   \eta \, \rho^{-1}  }
	{ (1+\rho^2)^{-2} -  \eta \, \rho^{-3}}	
   		  }
		  \right) 
		=
		  O\left(  \sqrt{ 
	 \frac{ \eta \,  \rho^{3}   }
	{ 1 -  \eta \, \rho}	
   		  }
		  \right) .
\end{equation*}
Therefore, for a fixed noise level, $\hattboundB$ increases at a rate of $ O \left( \rho^{3/2} \, (1- \rho)^{-1/2} \right)$ with the increase in the filter size $\rho$. Similarly, for a fixed filter size, the rate of increase of $\hattboundB$ with the noise standard deviation $\eta$ is $ O \left( \eta^{1/2} \, (1- \eta)^{-1/2} \right)$.
\end{theorem}\\

The proof of Theorem \ref{thm:order_bound_t0} is presented in \cite[Appendix D.2]{Vural12TR}. The stated result is obtained by using the relations in (\ref{eq:ord_hatepsilon_hateta}) to determine how the terms $\hatbndstdevdeltahUnif$, $\hatbetazerolb$, $\hatbndstdevUnifsecderh$ in the expression of $\hattboundB$ vary with $\rho$ and $\eta$. \revis{We omit the constants in the variation of the error bound for the sake of simplicity. However, due to the relations (\ref{eq:cor_var_g_unif}) and (\ref{eq:cor_bndunif_d2hdt2}), the constants multiplying the noise level parameter $\eta$ in the numerator and the denominator are respectively related with the parameters $\sqrtcthree$ and $\sqrtcfour$. Similarly, the positive constant in the denominator is associated with the parameter $\betazerolb$.}

Theorem $\ref{thm:order_bound_t0}$ constitutes the summary of our analysis about the effect of filtering on the alignment accuracy for the Gaussian noise model. While the aggravation of the alignment error with the increase in the noise level is an intuitive result, the theorem states that filtering the patterns under the presence of noise decreases the accuracy of alignment as well. Remember that this is not the case for noiseless patterns. The result of the theorem can be interpreted as follows. Smoothing the reference and target patterns diffuses the perturbation on the distance function, which is likely to cause a bigger shift in the minimum of the distance function and hence reduce the accuracy of alignment. The estimation $\hattboundB =  O \left( \rho^{3/2} \, (1- \rho)^{-1/2} \right)$ of the alignment error suggests that the dependence of the error on $\rho$ is between linear and quadratic for small values of $\rho$, whereas it starts to increase more dramatically when $\rho$ takes larger values. Similarly, $\hattboundB$ is proportional to the square root of $\eta$ for small $\eta$ and it increases at a sharper rate as $\eta$ grows.


Next, we look at the variation of the bounds $\hattboundBgen$ and $\hattboundBgenuncor$ for arbitrary noise patterns, which are respectively obtained for the general and small-correlation cases. We present the following theorem, which is the counterpart of Theorem \ref{thm:order_bound_t0} for arbitrary noise models.\\

\begin{theorem}
\label{thm:order_bound_t0_gen}
The alignment error bounds $\hattboundBgen$ and $\hattboundBgenuncor$ for arbitrary noise patterns have a variation of
\begin{equation}
 O \left(  
	\sqrt{ \frac{ \nu \, (1+ \rho^2) }
		{1- \nu}  }
		\right)
\end{equation}
with the noise level $\nu$ and the filter size $\rho$. Therefore, for a fixed noise level, the errors $\hattboundBgen$ and $\hattboundBgenuncor$ increase at a rate of $ O \left(  (1+\rho^2)^{1/2} \right) $ with the increase in the filter size $\rho$. Similarly, for a fixed filter size,  $\hattboundBgen$ and $\hattboundBgenuncor$ increase at a rate of $ O\left(  \nu^{1/2} (1-\nu)^{-1/2} \right) $ with respect to the noise norm $\nu$.

\end{theorem}

The proof of Theorem \ref{thm:order_bound_t0_gen} is given in \cite[Appendix E.1]{Vural12TR}. The dependence of the generalized bounds $\hattboundBgen$ and $\hattboundBgenuncor$ on the noise norm $\nu$ is the same as the dependence of $\hattboundB$ on $\eta$. However, the variation of $\hattboundBgen$ and $\hattboundBgenuncor$ with $\rho$ is seen to be slightly different from that of $\hattboundB$. This stems from the difference between the two models. In the generalized noise model $z$, we have treated the norm $\nu$ of $z$ as a known fixed number and we have characterized the alignment error in terms of $\nu$. On the other hand, $w$ is a probabilistic Gaussian noise model; therefore, it is not possible to bound its norm with a fixed parameter. For this reason, the alignment error for $w$ has been derived probabilistically in terms of the standard deviations of the involved parameters. Since the filter size $\rho$ affects the norm of $z$ and the standard deviations of the terms related to $w$ in different ways, it has a different effect on these two type of alignment error bounds. The reason why the two error bounds have the same kind of dependence on the noise level parameters $\eta$ and $\nu$ can be explained similarly. The standard deviations of the terms related to $w$ have a simple linear dependence on $\eta$, which is the same as the dependence of the counterparts of these terms in the generalized model on $\nu$.

\section{Experiments}
\label{sec:exp_res}

\subsection{Evaluation of Alignment Regularity Analysis}
\label{ssec:compute_siden}

\begin{wrapfigure}{l}{0.4\textwidth}
 \begin{center}
  \includegraphics[scale=0.45]{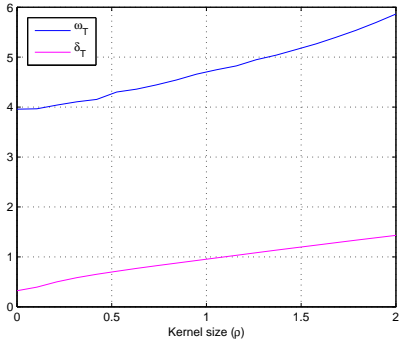}
  \end{center}
  \caption{The variations of the true distance $\hattsiden$ of the boundary of $\hatsiden$ to the origin and its estimation $ \hat{\delta}_T$ with respect to the filter size}
  \label{fig:estim_vs_true_rad}
\end{wrapfigure}

We first evaluate our theoretical results about SIDEN estimation with an experiment that compares the estimated SIDEN to the true SIDEN. We generate a reference pattern consisting of 40 randomly selected Gaussian atoms with random coefficients, and choose a random unit direction $T$ for pattern displacement. Then, we determine the distance $\hattsiden$ of the true SIDEN boundary from the origin along $T$, and compare it to its estimation $ \hat{\delta}_T$ for a range of filter sizes $\rho$ (With an abuse of notation, the parameter denoted as $\hattsiden$ here corresponds in fact to $\sup \hattsiden$ in the definition of SIDEN in (\ref{eq:defn_truesiden})). The distance $\hattsiden$ is computed by searching the first zero-crossing of $d\hat{f}(tT)/dt$ numerically, while its estimate $ \hat{\delta}_T$ is computed according to Theorem \ref{thm:siden_bnd}. We repeat the experiment 300 times with different random reference patterns $p$ and directions $T$ and average the results of the cases where $d\hat{f}(tT)/dt$ has zero-crossings for all values of $\rho$ (i.e., 56\% of the tested cases). The distance $\hattsiden$ and its estimation $\hat{\delta}_T$ are plotted in Figure \ref{fig:estim_vs_true_rad}. The figure shows that $\hat{\delta}_T$ has an approximately linear dependence on $\rho$. This is an expected behavior, since $\hat{\delta}_T = O\left( (1+\rho^2)^{1/2} \right) \approx O(\rho)$ for large $\rho$. The estimate $\hat{\delta}_T$ is smaller than $\hattsiden$ since it is a lower bound for $\hattsiden$. Its variation with $\rho$ is seen to capture well the relative variations of the true SIDEN boundary $\hattsiden$ with $\rho$.


\subsection{Evaluation of Alignment Accuracy Analysis}
\label{ssec:exp_align_accuracy}

We now present experimental results evaluating the alignment error bounds derived in Section \ref{sec:noise_anly}. We conduct the experiments on reference and target patterns made up of Gaussian atoms, where the target pattern is generated by corrupting the reference pattern with noise and applying a random translation $tT$. In all experiments, an estimation $t_e T_e$ of $tT$ is computed by aligning the reference and target images with a gradient descent algorithm\footnote{In the computation of $t_e T_e$, in order to be able to handle large translations, before the optimization with gradient descent we first do a coarse preregistration of the reference and target images with a search on a coarse grid in the translation parameter domain, whose construction is explained in Section \ref{ssec:grid_application}.}, which gives the experimental alignment error as $\| t T - t_e T_e \|$. The experimental error is then compared to the theoretical bounds derived in Section \ref{sec:noise_anly}. 

\subsubsection{Gaussian noise model}

In the first set of experiments, we evaluate the results for the Gaussian noise model. We compare the experimental alignment error to the theoretical bound given in Theorem \ref{theo:t0bound}. \footnote{The bound $\bndvargUnifsecderh$ given in Lemma \ref{cor:bndunif_d2hdt2} is derived from the preliminary bound $\bndvarsecderh$ in \cite[Lemma 3]{Vural12TR}. In the implementation of Theorem \ref{theo:t0bound}, in order to obtain a sharper estimate of $\bndvargUnifsecderh$, we compute it by searching the maximum value of $\sigma^2_{\secderh}$ over $t$ and $T$ from the expressions for $\Ej$ and $\Fj$ used in the derivation of $\bndvarsecderh$.} In all experiments, the parameter $s$ in Theorem \ref{theo:t0bound}, which controls the probability, is chosen such that $\tming < \tboundB$ holds with probability greater than $0.5$. For each reference pattern, the experiment is repeated for a range of values for noise variances $\eta^2$ and filter sizes $\rho$. The maximum value of the noise standard deviation is taken as the admissible noise level $\etabound$ in Theorem \ref{theo:t0bound}.

\begin{figure}[]
\begin{minipage}{0.3\linewidth}
\begin{center}
     \subfigure[]
       {\label{fig:exprand_rho_exp}\includegraphics[height=4.0cm]{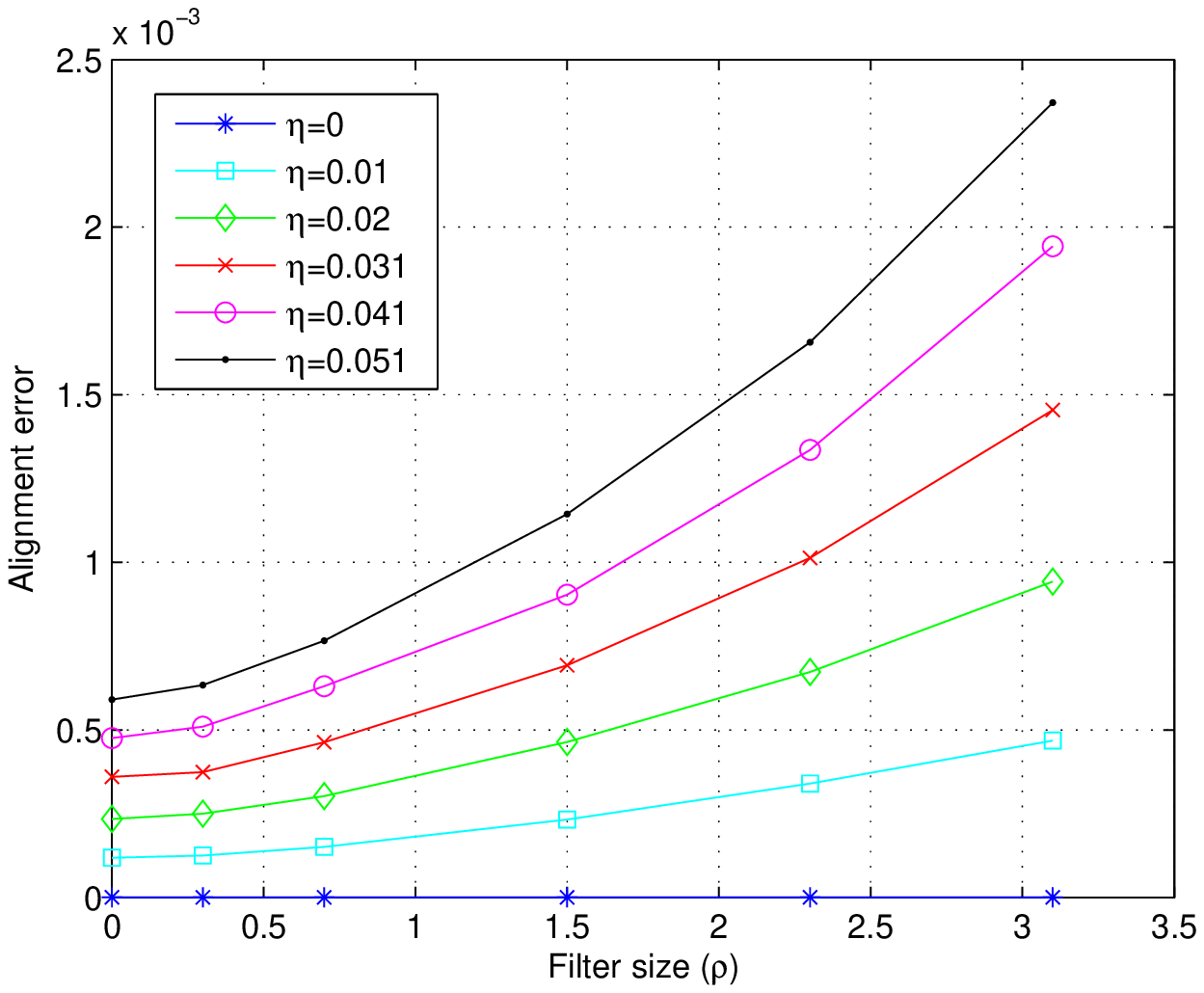}}
     \subfigure[]
       {\label{fig:exprand_rho_theo}\includegraphics[height=4.0cm]{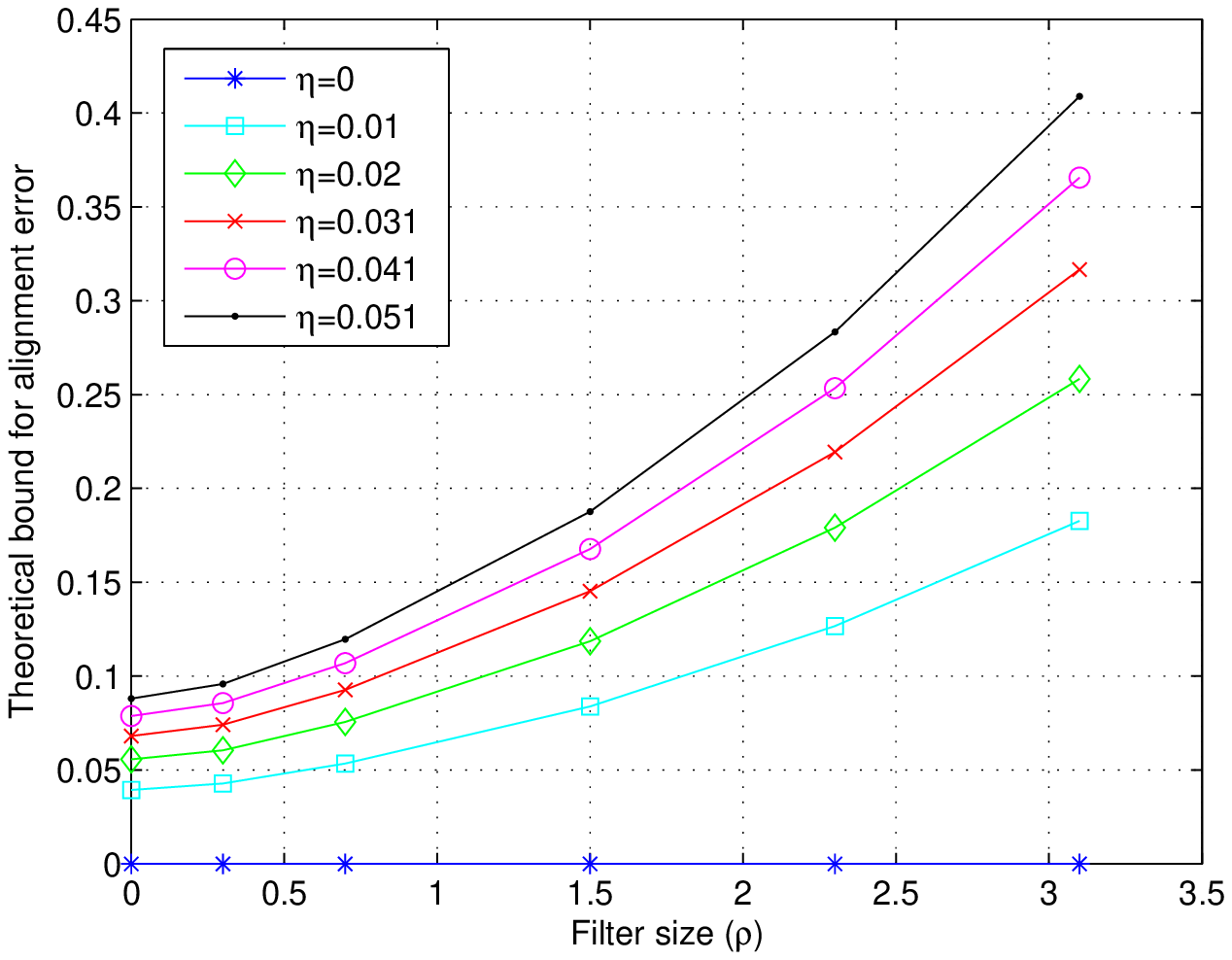}}
 \end{center}
 \caption{Alignment error of random patterns as a function of filter size $\rho$.}
 \label{fig:exprand_rho}
\end{minipage}
\hspace{0.1cm}
\begin{minipage}{0.3\linewidth}
\begin{center}
     \subfigure[]
       {\label{fig:exprand_eta_exp}\includegraphics[height=4.0cm]{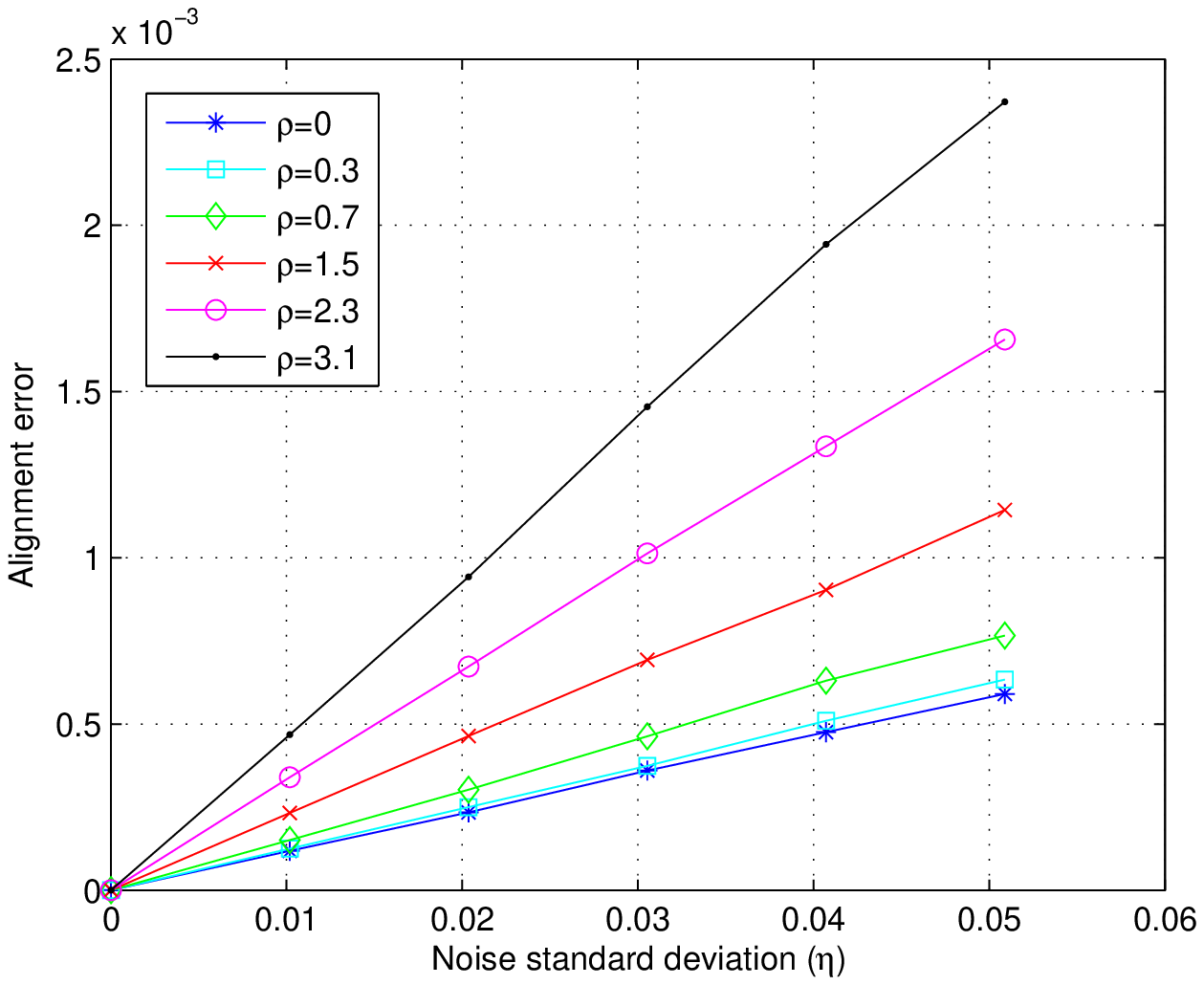}}
     \subfigure[]
       {\label{fig:exprand_eta_theo}\includegraphics[height=4.0cm]{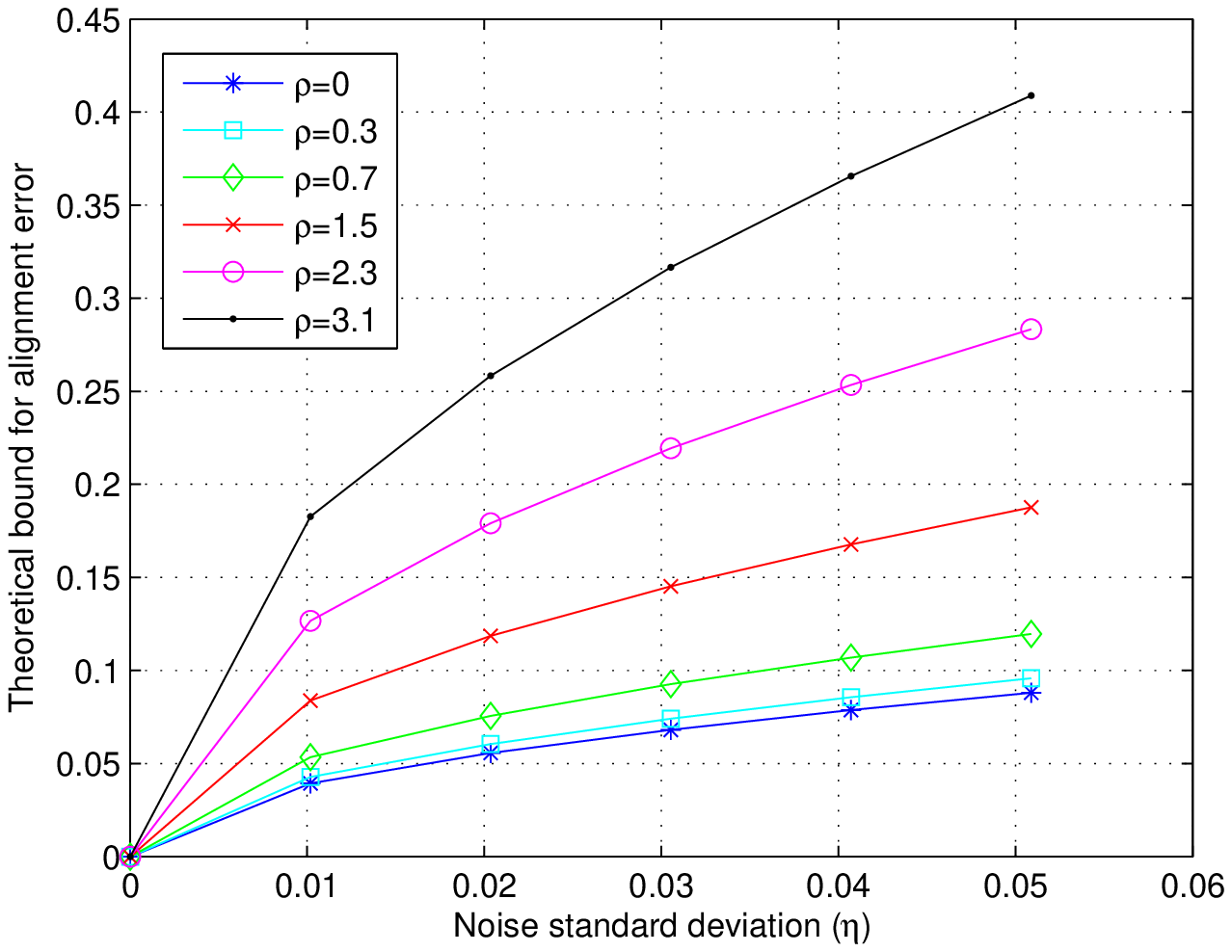}}
 \end{center}
 \caption{Alignment error of random patterns as a function of noise standard deviation $\eta$.}
 \label{fig:exprand_eta}
\end{minipage}
\hspace{0.1cm}
\begin{minipage}{0.3\linewidth}
\begin{center}
      \subfigure[]
       {\label{fig:randpat_eta_exp_bigeta}\includegraphics[height=4.0cm]{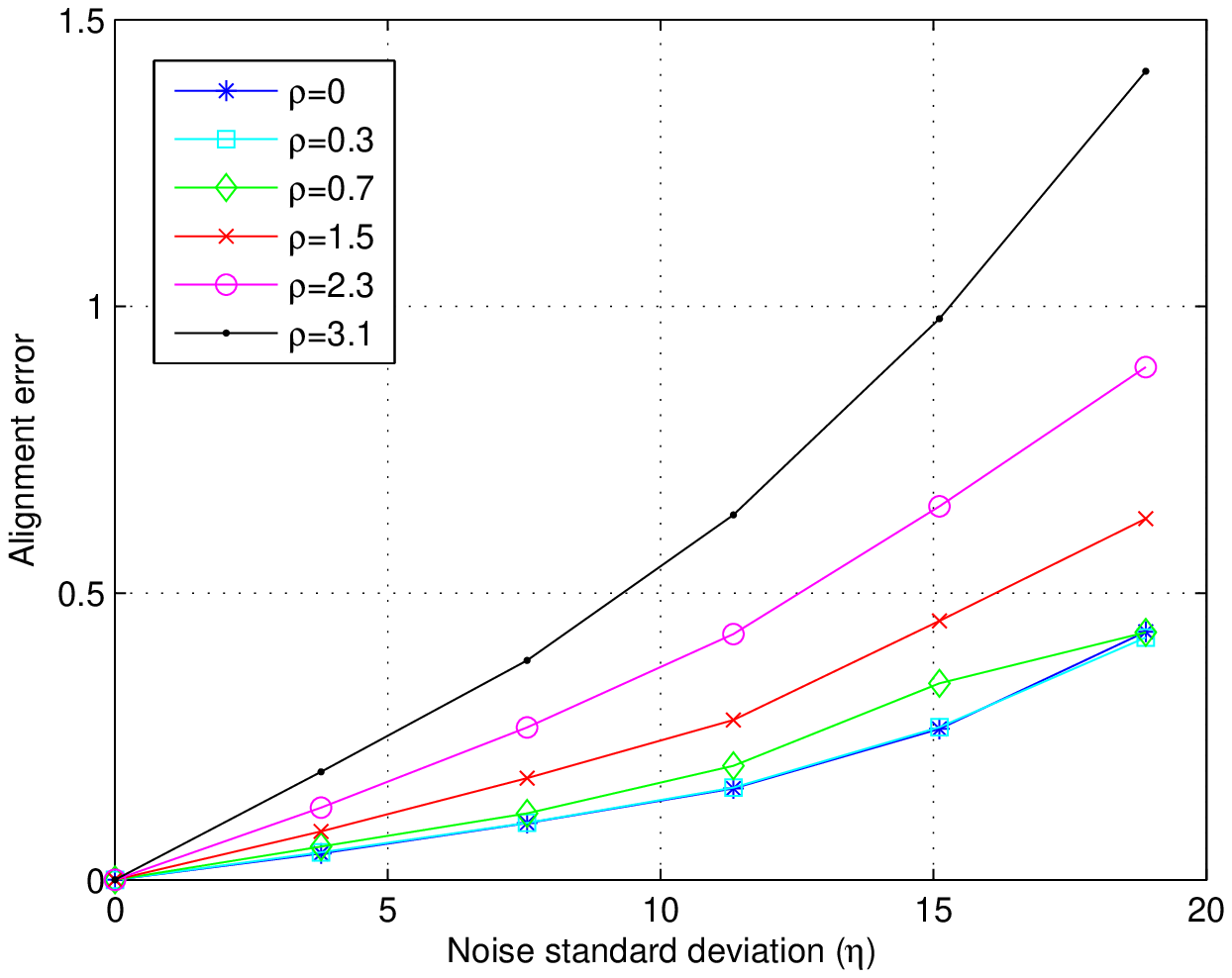}}
      \subfigure[]
       {\label{fig:randpat_rho_exp_bigeta}\includegraphics[height=4.0cm]{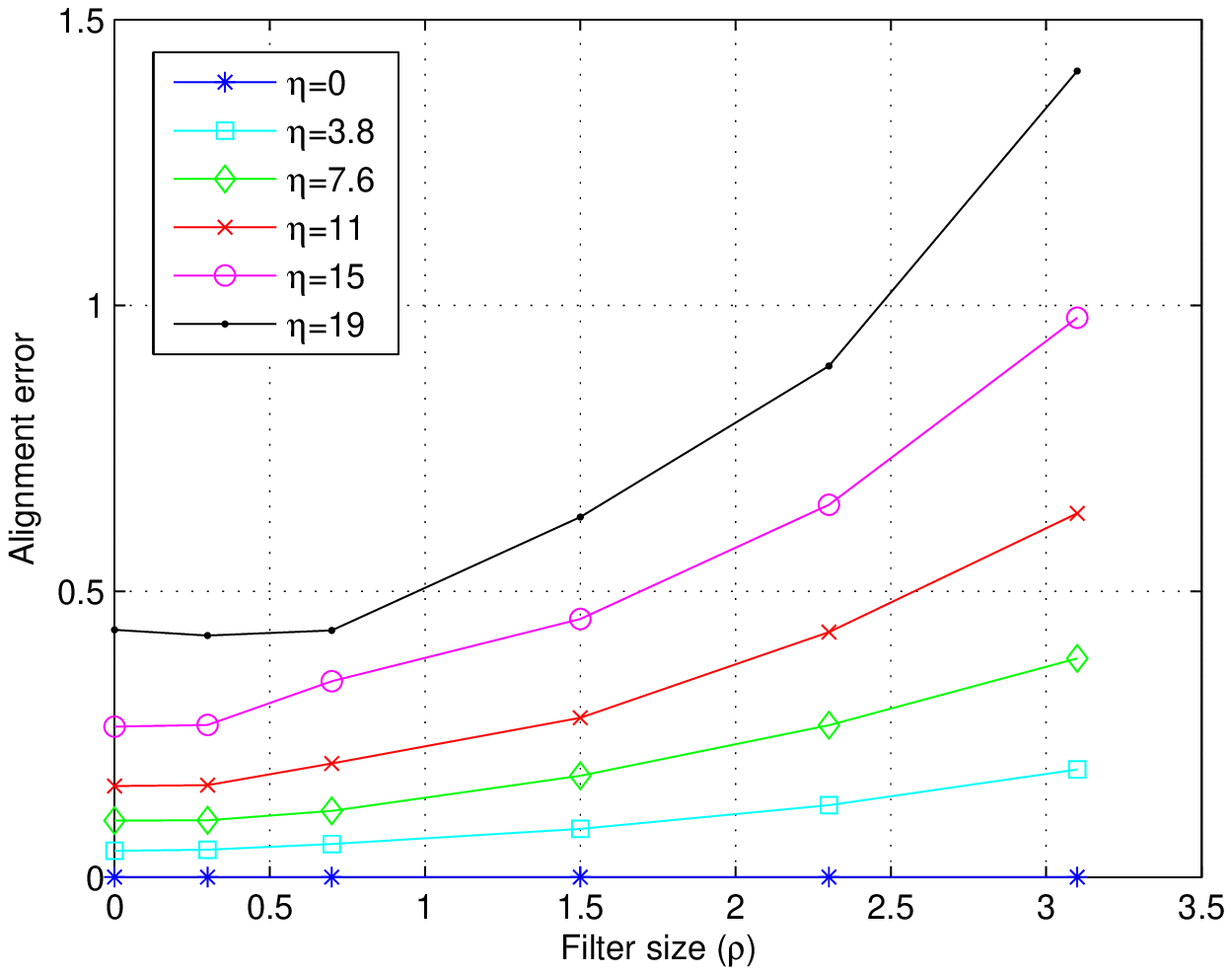}}
 \end{center}
 \caption{Alignment error of random patterns as functions of the noise standard deviation and the filter size, at high noise levels.}
 \label{fig:exprand_highnoise}
\end{minipage}
\end{figure}

We first experiment on reference patterns built with 20 Gaussian atoms with randomly chosen parameters. The atom coefficients $\coef_k$ in the reference patterns are drawn from a uniform distribution in $[-1, 1]$; and the position and scale parameters of the atoms are selected such that $\tau_x, \tau_y \in [-4, 4]$ and $\sigma_x, \sigma_y \in [0.3, 2]$.  The noise model parameters are set as $L=750$, $\epsilon=0.1$. The experiment is repeated on 50 different reference patterns. Then, 50 noisy target patterns are generated for each reference pattern according to the Gaussian noise model $w$ in (\ref{eq:analy_noise_model}) with a random translation $tT$ in the range $t T_x, tT_y \in [-4, 4]$. The results are averaged over all reference and target patterns. In Figure \ref{fig:exprand_rho}, the experimental and theoretical values of the alignment error are plotted with respect to the filter size $\rho$, where different curves correspond to different $\eta$ values. Figures \ref{fig:exprand_rho_exp} and \ref{fig:exprand_rho_theo} show respectively the experimental value $\| tT - t_e T_e  \|$ and the theoretical upper bound $\tboundB$ of the alignment error. Figure \ref{fig:exprand_eta} shows the same results, where the error is given as a function of $\eta$. The experimental values and the theoretical bounds are given respectively in Figures \ref{fig:exprand_eta_exp} and \ref{fig:exprand_eta_theo}. \revis{Note that, due to the range of atom translation parameters $\tau_x$, $\tau_y$, the energy of the reference pattern is concentrated in the region $[-4, 4] \times [-4, 4]$. Therefore, the maximum filter size $\rho=3.1$ tested in this experiment is close to the half of the image width. The maximum noise level $\eta=0.051$ corresponds to an SNR of approximately 26 dB.}

The results in Figure \ref{fig:exprand_rho} show that, although the theoretical upper bound is pessimistic (which is due to the fact that the bound is a worst-case analysis), the variation of the experimental value of the alignment error as a function of the filter size is in agreement with that of the theoretical bound. The experimental behavior of the error conforms to the theoretical prediction $\hattboundB \approx  O \left( \rho^{3/2} \, (1- \rho)^{-1/2} \right)$ of Theorem \ref{thm:order_bound_t0}. Next, the plots of Figure \ref{fig:exprand_eta} suggest that the variation of the theoretical bound $\tboundB$ as a function of $\eta$ is consistent with the result of Theorem \ref{thm:order_bound_t0}, which can be approximated as $\hattboundB \approx O(\sqrt{\eta}) $ for small values of $\eta$. On the other hand, the experimental value of the alignment error seems to exhibit a more linear behavior. However, this type of dependence is not completely unexpected. Theorem \ref{thm:order_bound_t0} predicts that $\hattboundB$ is of $O(\sqrt{\eta})$ for small $\eta$; and $O\left( \eta^{1/2} (1-\eta)^{-1/2} \right)$ for large $\eta$, while the experimental value of the error can be rather described as $\| tT - t_eT_e  \|=O(\eta)$, which is between these two orders of variation. In order to examine the dependence of the error on $\eta$ in more detail, we have repeated the same experiments with much higher values of $\eta$. The experimental alignment error is given in Figure \ref{fig:exprand_highnoise}, where the error is plotted with respect to the noise standard deviation in Figure \ref{fig:randpat_eta_exp_bigeta} and the filter size in Figure \ref{fig:randpat_rho_exp_bigeta}. \revis{The maximum noise level $\eta=19$ in this experiments corresponds to an SNR of approximately -22.5 dB.} The results show that, at high noise levels, the variation of the error with $\eta$ indeed increases above the linear rate $O(\eta)$. The noise levels tested in this high-noise experiment are beyond the admissible noise level derived in Theorem \ref{theo:t0bound}; therefore, we cannot apply Theorem \ref{theo:t0bound} directly in this experiment. However, in view of Theorem \ref{thm:order_bound_t0}, which states that the error is of  $O\left( \eta^{1/2} (1-\eta)^{-1/2} \right)$, these results can be interpreted to provide a numerical justification of our theoretical finding: at relatively high noise levels, the error is expected to increase with $\eta$ at a sharply increasing rational function rate above the linear rate. The variation of the error with $\rho$ at high noise levels plotted in Figure \ref{fig:randpat_rho_exp_bigeta} is seen to be similar to that of the previous experiments.

\begin{figure}[]
\begin{minipage}[b]{0.3\linewidth}
\begin{center}
     \subfigure[Face image]
       {\label{fig:face_image}\includegraphics[height=3.7cm]{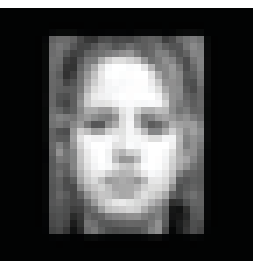}}
     \subfigure[]
       {\label{fig:expface_rho_exp}\includegraphics[height=4.0cm]{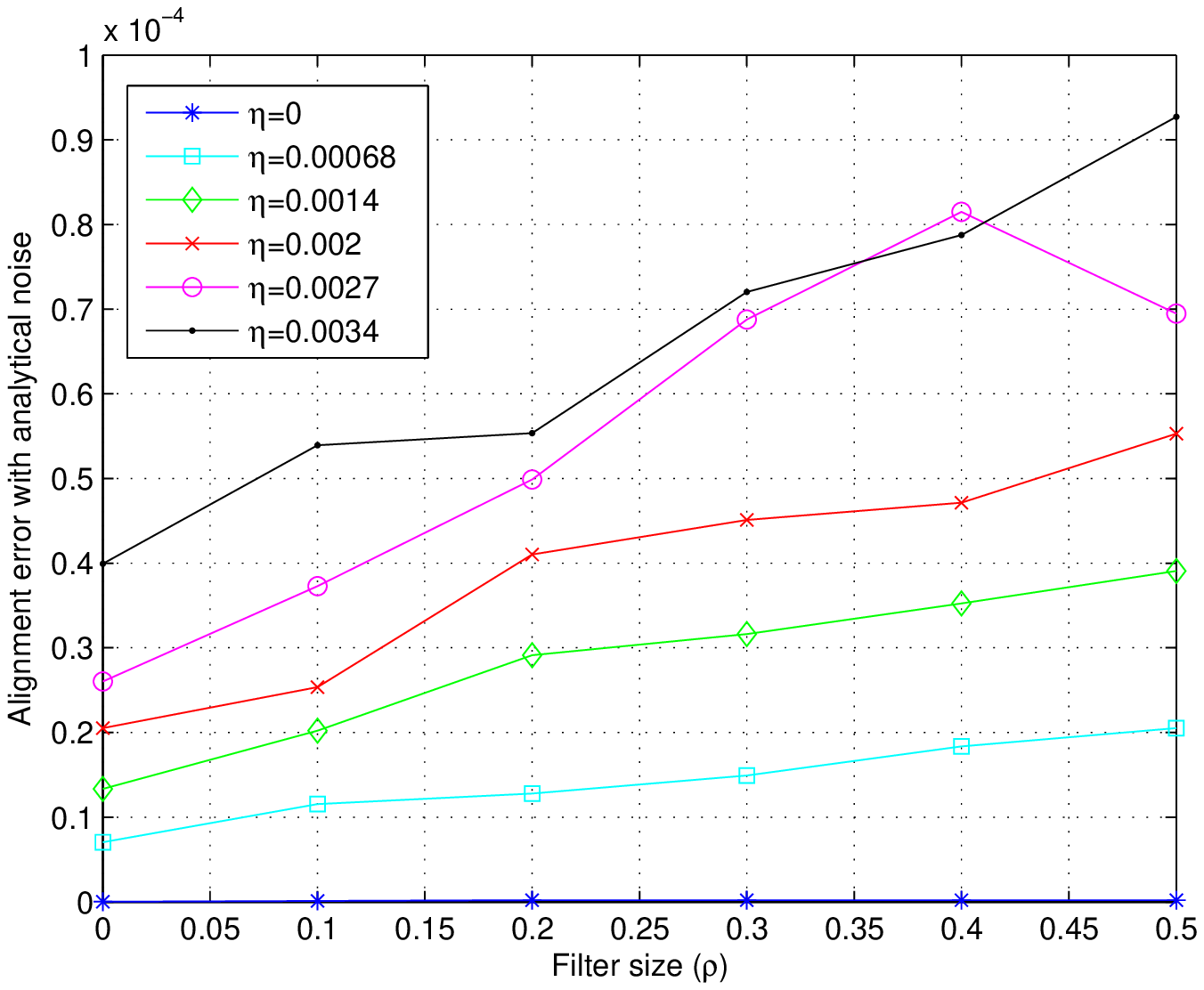}}
     \subfigure[]
       {\label{fig:expface_rho_theo}\includegraphics[height=4.0cm]{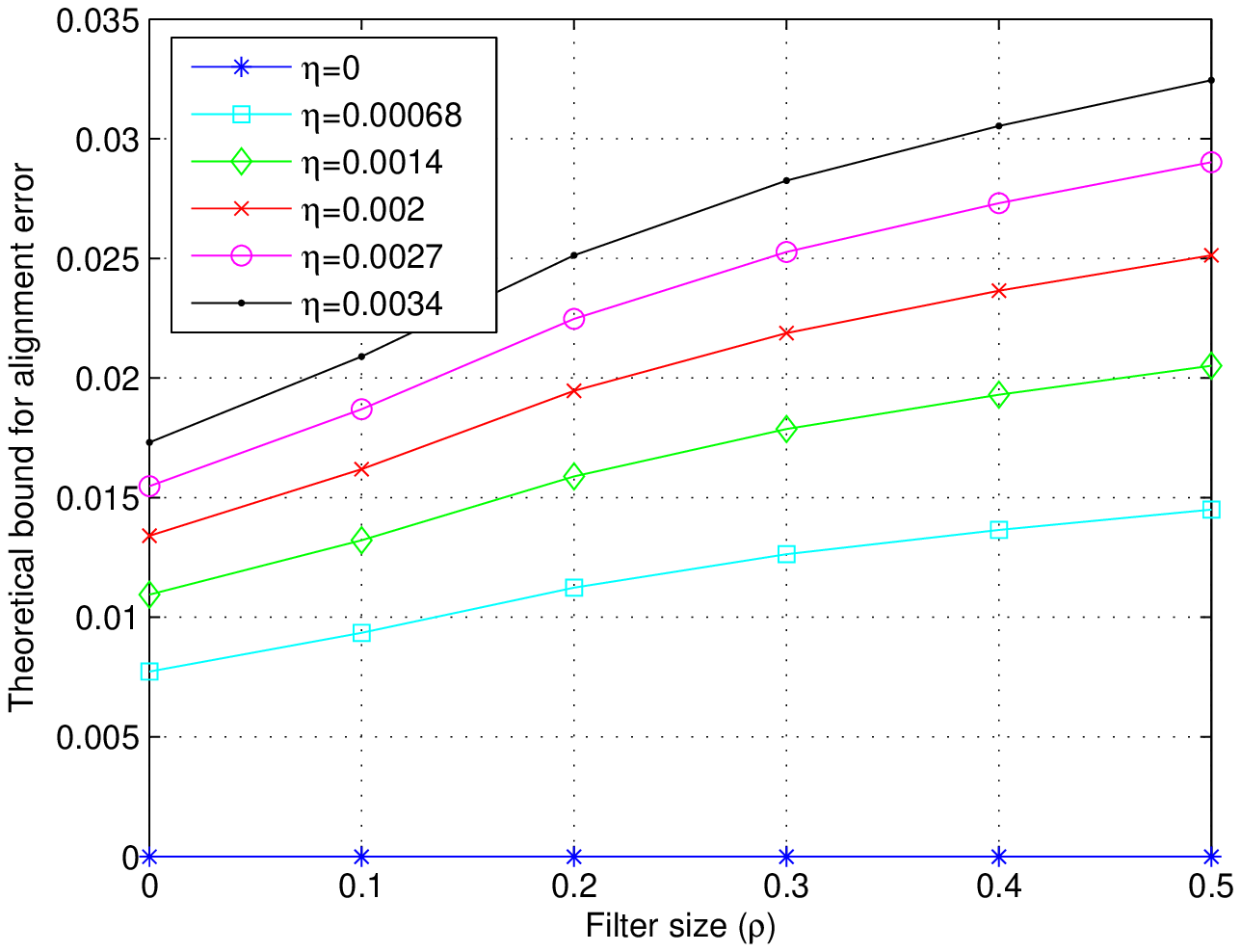}}
      \subfigure[]
       {\label{fig:expface_rho_dig}\includegraphics[height=4.0cm]{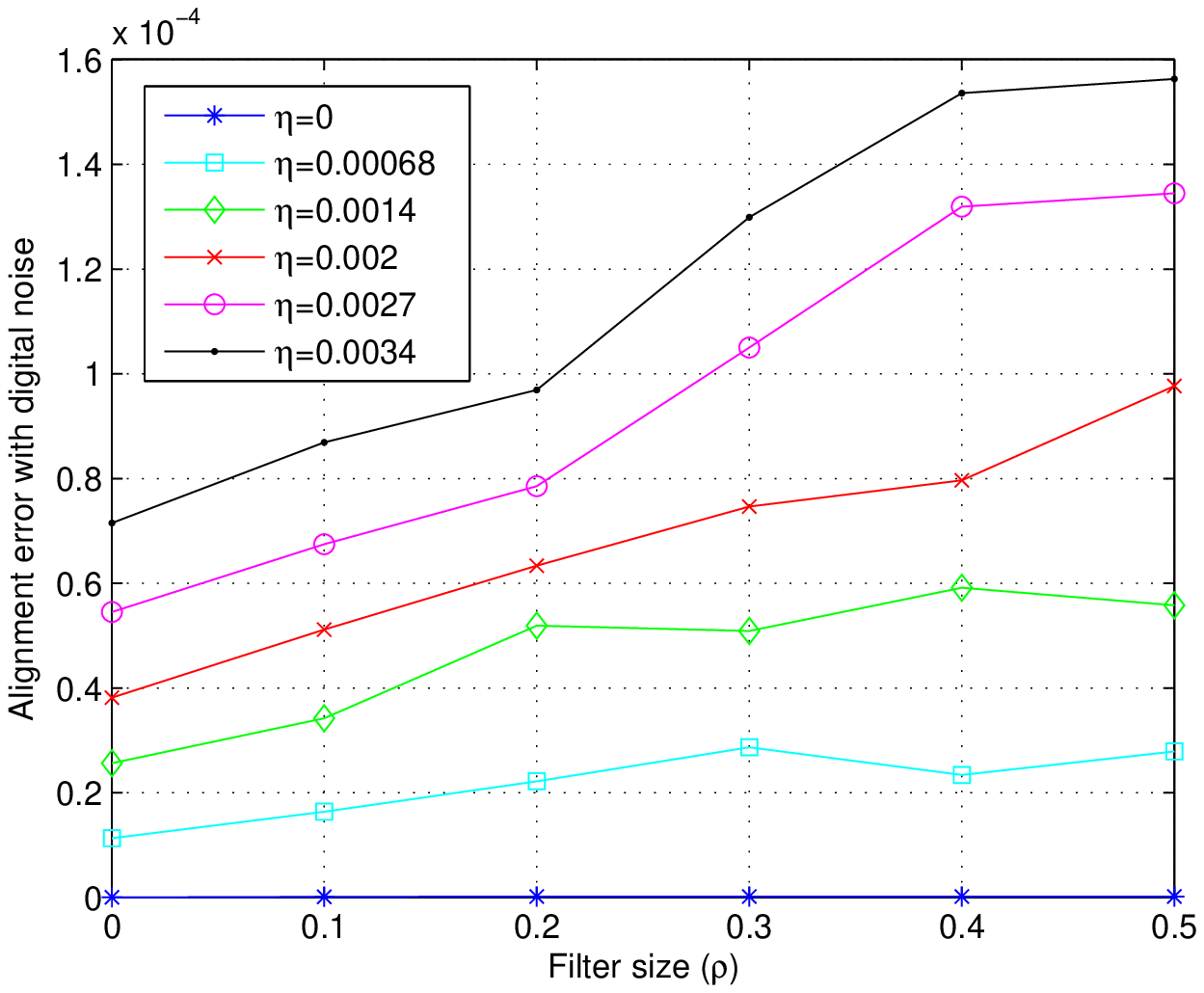}}
 \end{center}
 \caption{Face pattern and alignment error as a function of filter size $\rho$.}
 \label{fig:expface_rho}
\end{minipage}
\hspace{0.1cm}
\begin{minipage}[b]{0.3\linewidth}
\begin{center}
     \subfigure[Digit image]
       {\label{fig:digit_image}\includegraphics[height=3.7cm]{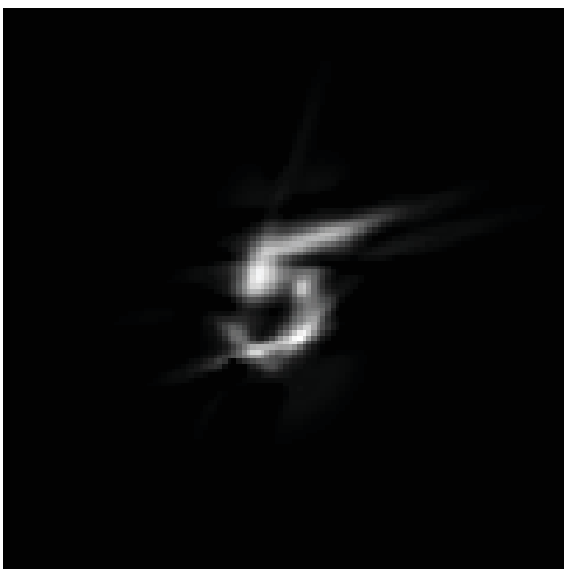}}
     \subfigure[]
       {\label{fig:expdigit_rho_exp}\includegraphics[height=4.0cm]{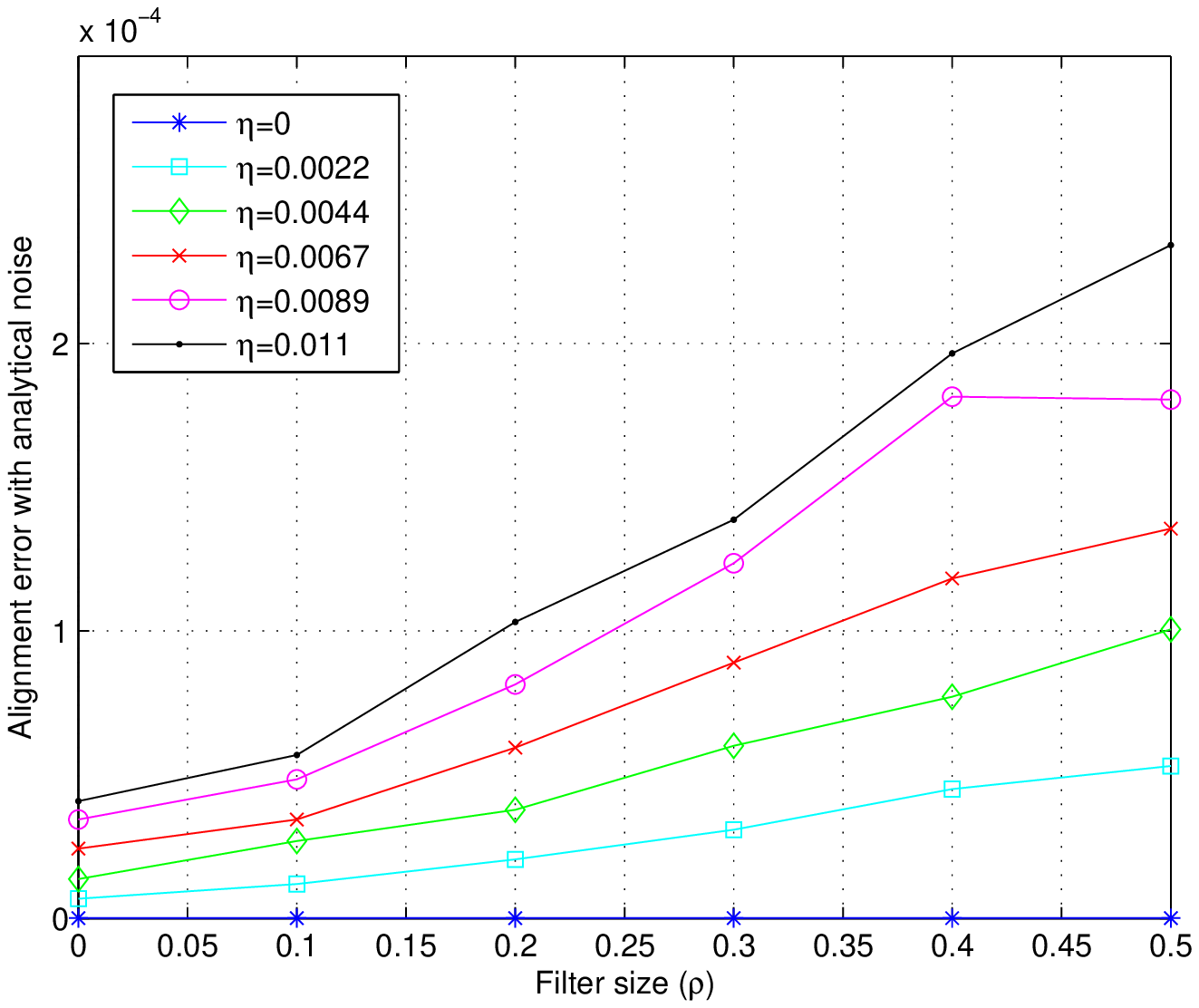}}
     \subfigure[]
       {\label{fig:expdigit_rho_theo}\includegraphics[height=4.0cm]{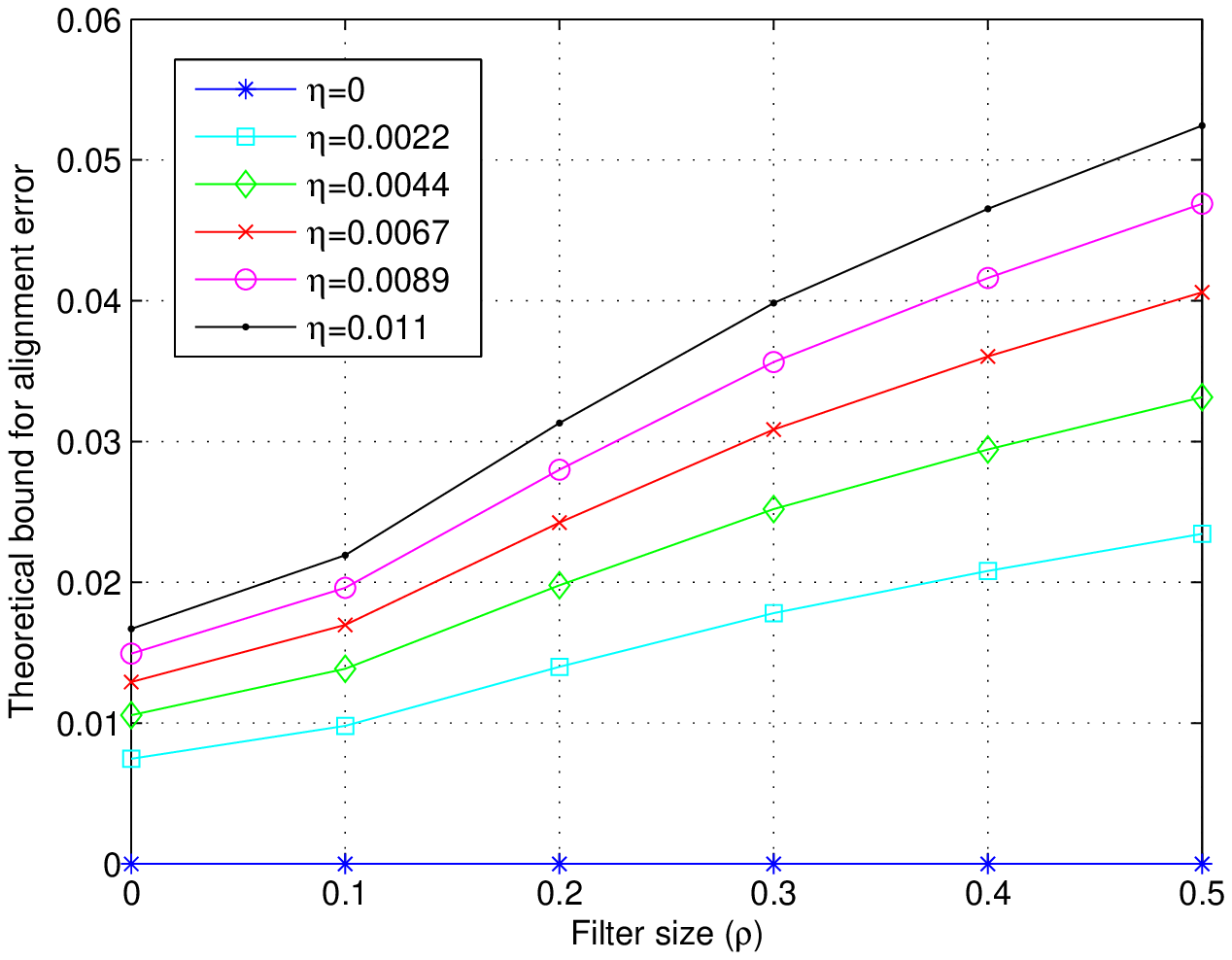}}
      \subfigure[]
       {\label{fig:expdigit_rho_dig}\includegraphics[height=4.0cm]{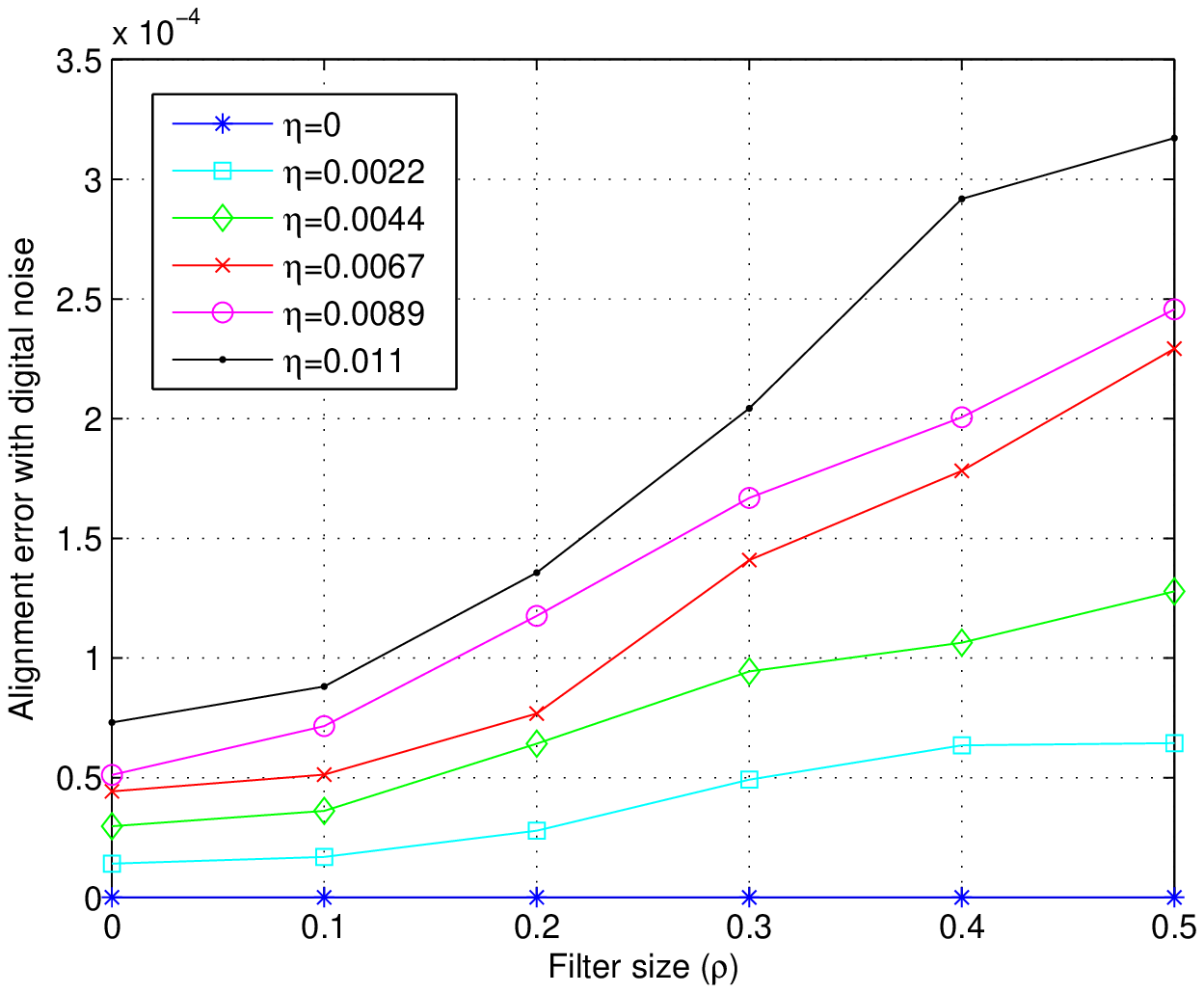}}
 \end{center}
 \caption{Digit pattern and alignment error as a function of filter size $\rho$.}
 \label{fig:expdigit_rho}
\end{minipage}
\hspace{0.1cm}
\begin{minipage}[b]{0.3\linewidth}
\begin{center}
     \subfigure[]
       {\label{fig:randpat_rho_exp_gen}\includegraphics[height=4.0cm]{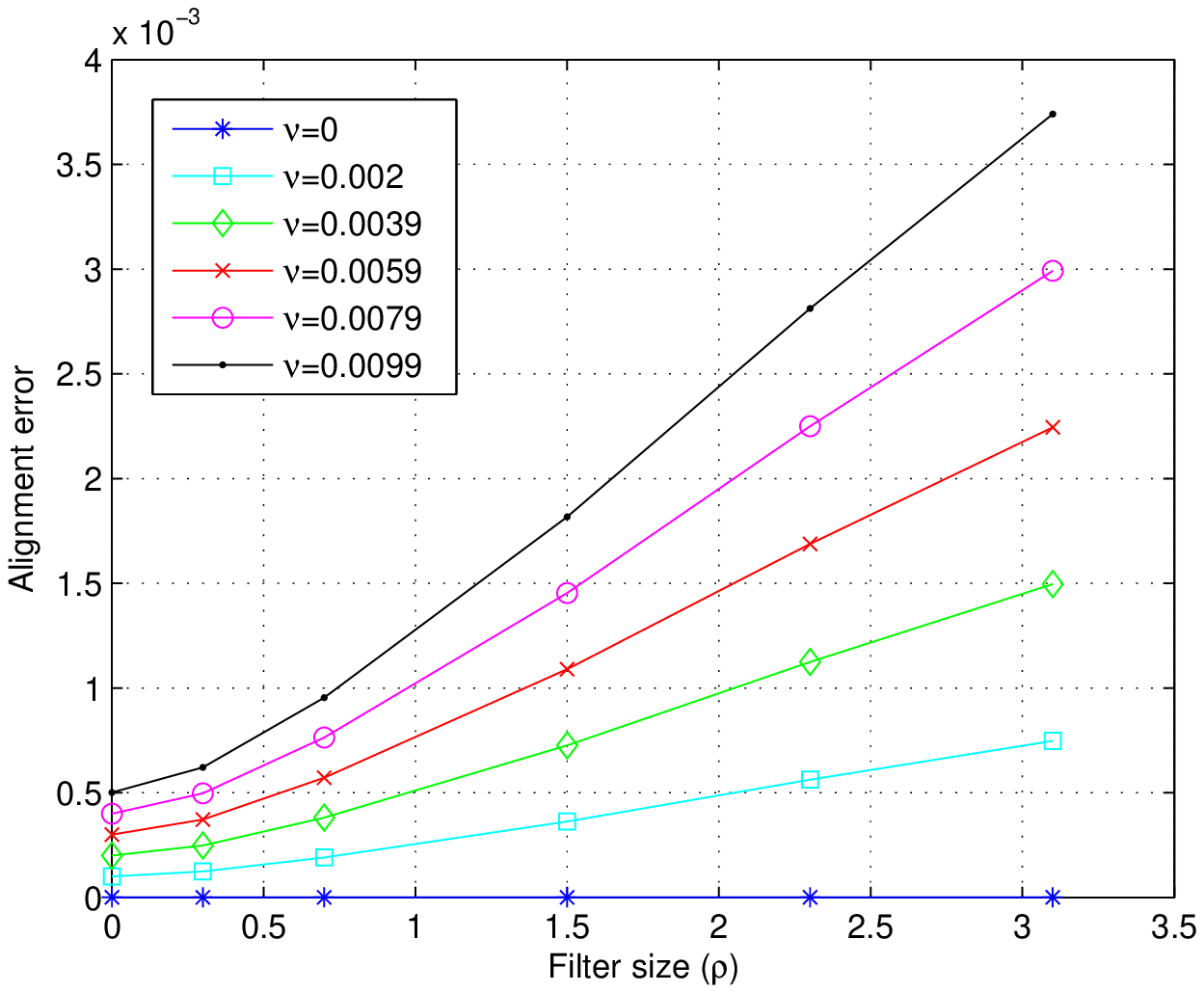}}
     \subfigure[]       
      {\label{fig:randpat_rho_theo_gen}\includegraphics[height=4.0cm]{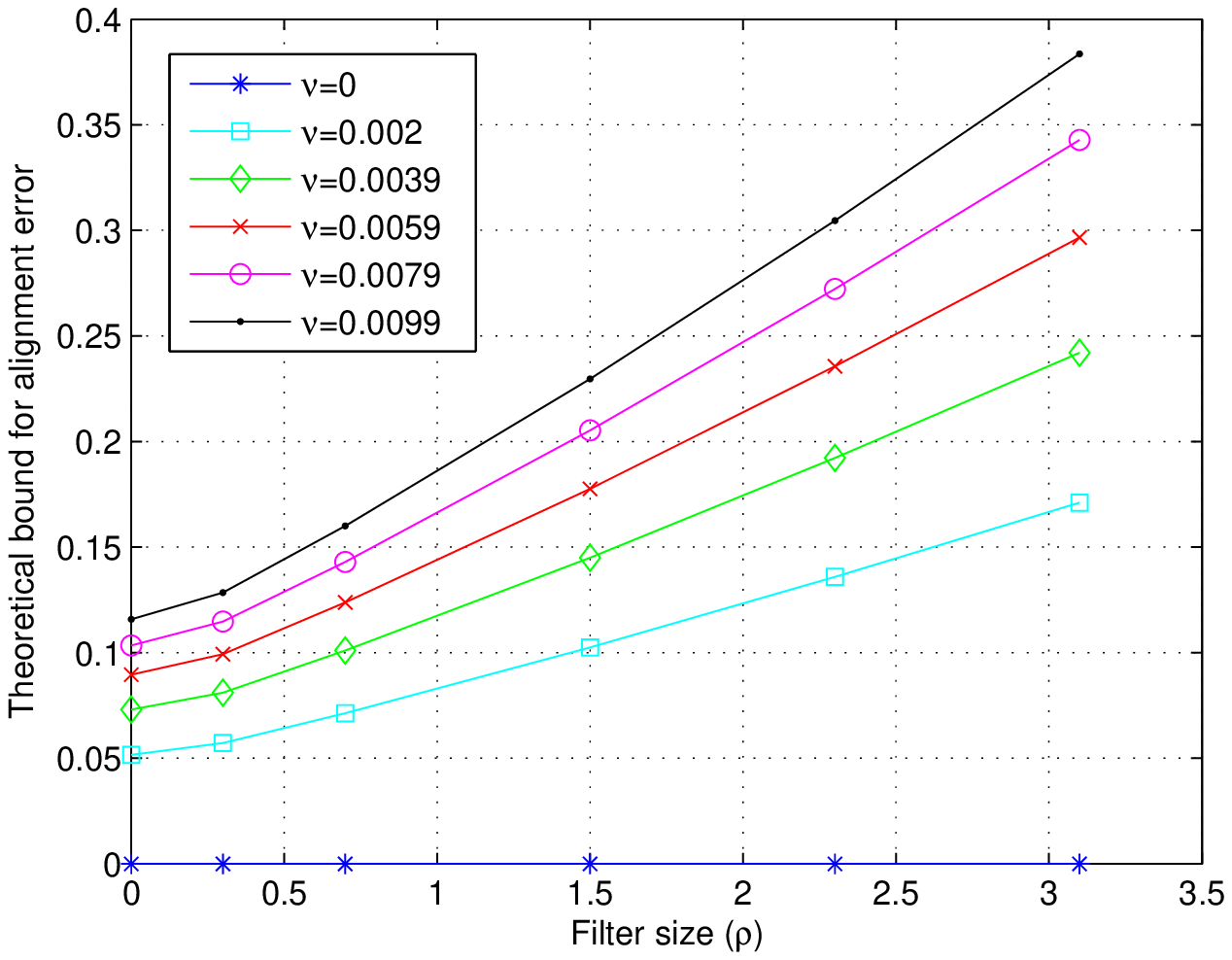}}
     \subfigure[]
       {\label{fig:randpat_rho_exp_uncor}\includegraphics[height=4.0cm]{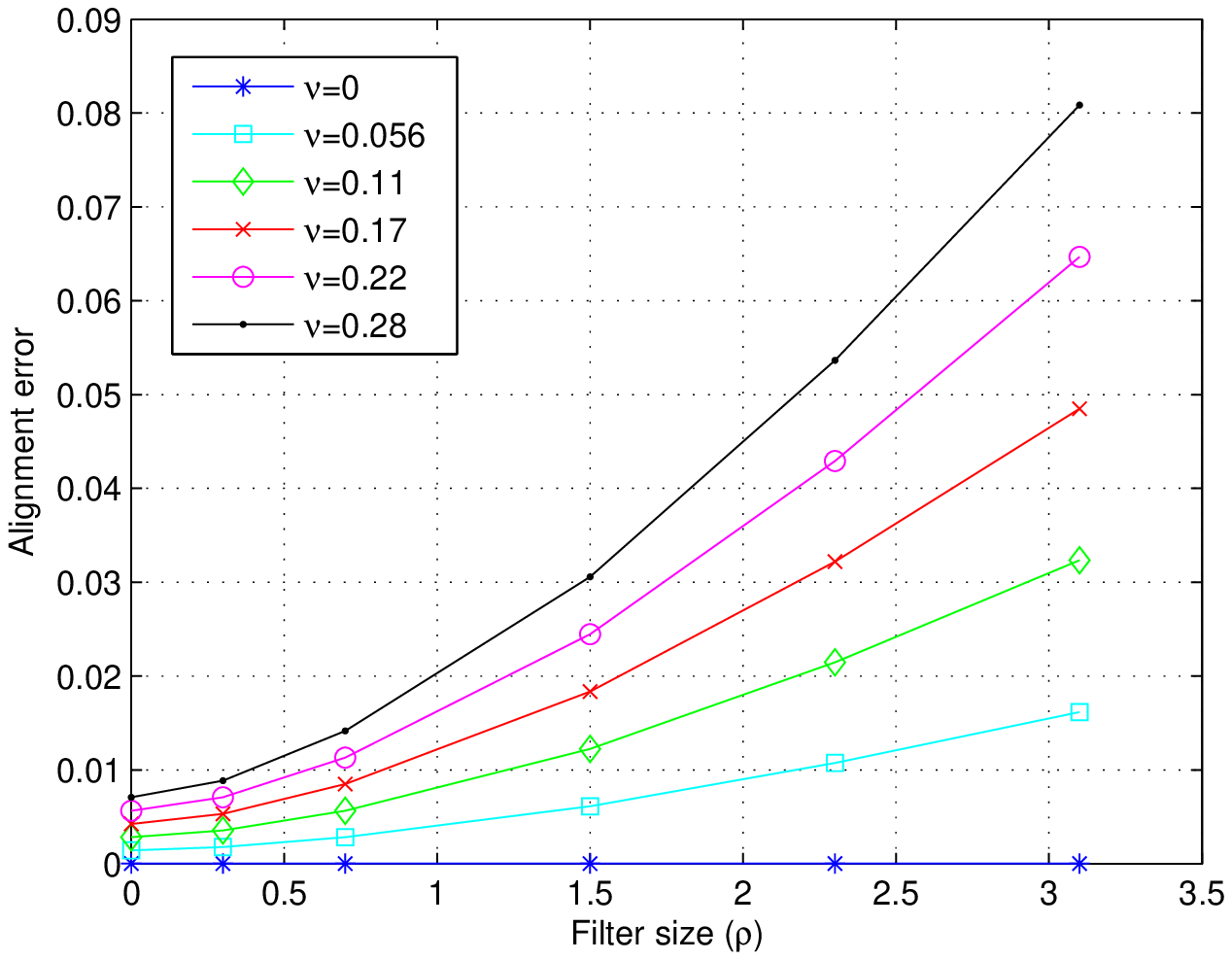}}
     \subfigure[]       
      {\label{fig:randpat_rho_theo_uncor}\includegraphics[height=4.0cm]{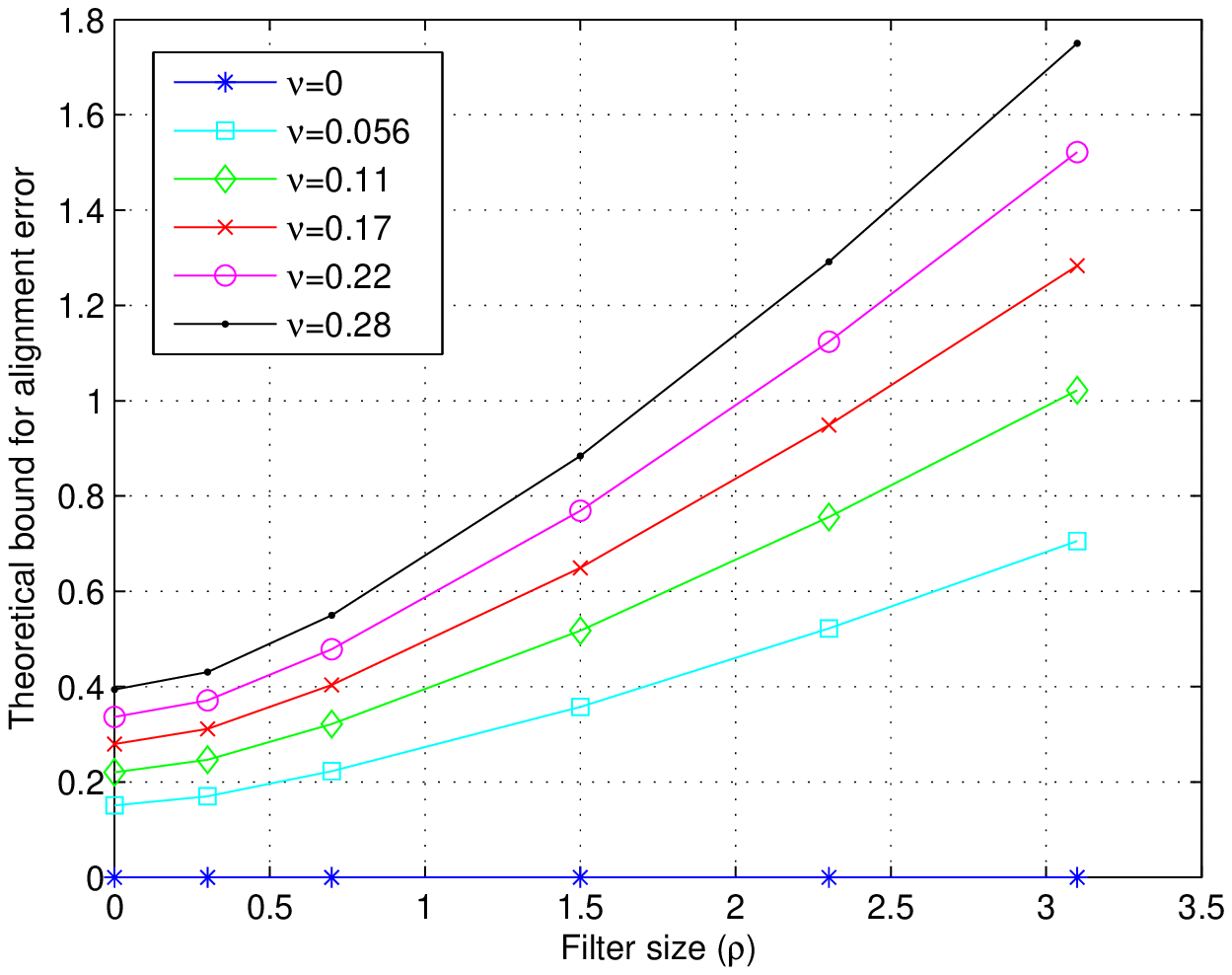}}
 \end{center}
 \caption{Alignment error for random patterns and generic noise, as a function of filter size $\rho$. (a) and (c) show the error for noise patterns respectively with high and low correlation with $p$. Corresponding theoretical bounds $\tboundBgen$ and $\tboundBgenuncor$ are given respectively in (b) and (d).}
 \label{fig:expgenrand_rho}
\end{minipage}
\end{figure}

We now evaluate our alignment accuracy results under Gaussian noise on face and digit images. First, the reference face pattern is obtained by approximating the face image shown in Figure \ref{fig:face_image} with 50 Gaussian atoms. The average atom coefficient magnitude of the face pattern is $|\coef | = 0.14 $, and the position and scale parameters of the pattern are in the range $ [-0.9, 0.9]$ for $\tau_x, \tau_y$, and $ [0.04, 1.1]$ for $\sigma_x, \sigma_y$. For the digit experiments, the reference pattern shown in Figure \ref{fig:digit_image} is the approximation of a handwritten ``5'' digit with 20 Gaussian atoms. The pattern parameters are such that the average atom coefficient magnitude is  $|\coef | = 0.87 $, and the position and scale parameters are given by $\tau_x, \tau_y \in [-0.7, 0.7]$, and $\sigma_x, \sigma_y \in [0.05, 1.23]$. \revis{The atom parameters for the face and digit images are computed with the matching pursuit algorithm \cite{1028585}}. The range of translation values $t T_x$ and  $t T_y$ is selected as $ [-1, 1]$ for both settings. In both settings, two different noise models are tested. First, the target patterns are corrupted with respect to the analytical Gaussian noise model $w$ of (\ref{eq:analy_noise_model}), where the noise parameters are set as $L=750$, $\epsilon= 0.04$. Then, a digital Gaussian noise model is tested, where the pixels in the discrete representation of the images are corrupted with additive i.i.d.~Gaussian noise having the same standard deviation $\eta$ as $w$. The digital Gaussian noise model is supposed to be well-approximated by the analytical noise model. Again, 50 target patterns are generated with random translations. The alignment errors are plotted with respect to $\rho$ in Figures \ref{fig:expface_rho} and \ref{fig:expdigit_rho} respectively for the face and digit patterns. The experimental error with the analytical noise model, the theoretical upper bound obtained for the analytical noise model, and the experimental error with the digital noise model are plotted respectively in (b), (c) and (d) in both figures. The results are averaged over all target patterns. \revis{In the face image experiment, the maximum noise level $\eta =0.0034$ corresponds to SNRs of $46.8$ dB and $34.7$ dB for the analytical and digital noise models.  In the digit image experiment, the maximum noise level $\eta=0.011$ yields the respective SNRs of $39.3$ dB and $28.1$ dB.}

The plots in Figures \ref{fig:expface_rho} and \ref{fig:expdigit_rho} show that the experimental and theoretical errors have a similar variation with respect to $\rho$. The dependence of the error on $\rho$ in these experiments seems to be different from that of the previous experiment of Figure \ref{fig:exprand_rho}. Although the theory predicts the variation  $\hattboundB = O \left( \rho^{3/2} \, (1- \rho)^{-1/2} \right)$, this result is average and approximate. The exact variation of the error with $\rho$ may change between different individual patterns, as the constants of the variation function are determined by the actual pattern parameters. The similarity between the plots for the analytical and digital noise models suggests that the noise model $w$ used in this study provides a good approximation for the digital Gaussian noise, which is often encountered in digital imaging applications. In \cite{Vural12TR}, the alignment error for face and digit patterns is also plotted with respect to $\eta$, and the results are similar to those of the previous experiment with random patterns.

\subsubsection{Generic noise model}

In the second set of experiments we evaluate the results of Theorem \ref{thm:alg_error_bndgen} and Corollary \ref{cor:align_error_uncor} for the generic noise model $z$. In each experiment, the target patterns are generated by corrupting the reference pattern $p$ with a noise pattern $z$ and by applying random translations in the range $t T_x, tT_y \in [-4, 4]$. In order to study the effect of the correlation between $z$ and the points on $\mathcal{M}(p)$ on the actual alignment error and on its theoretical bound, we consider two different settings. In the first setting, the noise pattern $z$ is chosen as a pattern that has high correlation with  $p$. In particular, $z$ is constructed with a subset of the atoms used in $p$ with the same coefficients. The general bound $\tboundBgen$ is used in this setting. In the second setting, the noise $z$ is constructed with randomly selected Gaussian atoms so that it has small correlation with $p$. The bound $\tboundBgenuncor$ for the small correlation case is used in the second setting, where the correlation parameter $\bndcorr$ in (\ref{eq:bndcorr_defn}) is computed numerically for obtaining the theoretical error bound. In both cases, the atom coefficients of $z$ are normalized such that the norm $\nu$ of $z$ is below the admissible noise level $\nu_0$. The theoretical bounds\footnote{We compute the bound for the second derivative of $p$ numerically by minimizing $\| d^2 p(X+tT) / d t^2 \|$ over $t$ and $T$. While the bound $\bndnormsecderp$ is useful for the theoretical analysis as it has an open-form expression, the numerically computed bound is sharper.} are then compared to the experimental errors for different values of the filter size $\rho$ and the noise level $\nu$.

We conduct the experiment on the random patterns used above, in Figures \ref{fig:exprand_rho}-\ref{fig:exprand_highnoise}. The noise pattern $z$ is constructed with 10 atoms. The average alignment errors are plotted with respect to the filter size $\rho$ in Figure \ref{fig:expgenrand_rho}. \revis{The noise levels $\nu=0.0099$ and $\nu=0.28$ correspond respectively to the approximate SNRs of $65.7$ dB and $36.7$ dB.} The plots in Figure \ref{fig:expgenrand_rho} show that the variation of the theoretical upper bounds with $\rho$ fits well the behavior of the actual error in both settings. The results are in accordance with Theorem \ref{thm:order_bound_t0_gen}, which states that $\hattboundBgen$ and $\hattboundBgenuncor$ are of $O((1+\rho^2)^{1/2})$. The results of the experiment show that $\hattboundBgenuncor$ is less pessimistic than $\hattboundBgen$ as an upper bound since it makes use of the information of the maximum correlation between $z$ and the points on $\mathcal{M}(p)$. Moreover, comparing Figures \ref{fig:randpat_rho_theo_gen} and  \ref{fig:randpat_rho_theo_uncor} we see that, when a bound $\bndcorr$ on the correlation is known, the admissible noise level increases significantly (from around $\nu_0=0.01$ to $\nu_0=0.28$). The plots of the errors with respect to $\nu$ are also available in \cite{Vural12TR}; in short, they show that the variation of the alignment error with $\nu$ bears resemblance to its variation with $\eta$  observed in the previous experiment of Figure \ref{fig:exprand_eta}. This is an expected result, as it has been seen in Theorem \ref{thm:order_bound_t0_gen} that the dependences of $\hattboundBgen$ and $\hattboundBgenuncor$ on $\nu $ are the same as the dependence of $\hattboundB$ on $\eta$. These plots show also that, at the same noise level, the actual alignment error is slightly smaller when $z$ has small correlation with the points on $\mathcal{M}(p)$. In \cite{Vural12TR}, this experiment is also repeated with the face and digit patterns, with results that are similar to those obtained with random test patterns.

The overall conclusion of the experiments is that increasing the filter kernel size results in a bigger alignment error when the target image deviates from the translation manifold of the reference image due to noise. The results show also that the theoretical bounds for the alignment error capture well the order of dependence of the actual error on the noise level and the filter size, for both the Gaussian noise model and the generalized noise model. Also, the knowledge of the correlation between the noise pattern and the translated versions of the reference pattern is useful for improving the theoretical bound for the alignment error in the general setting. \revis{Although the actual values of the theoretical upper bounds are quite pessimistic, they are roughly the same as the actual alignment errors up to a multiplicative factor. In a practical image registration application, the value of this factor can be numerically estimated. The normalization of the theoretical bounds with this factor then yields an accurate model for the multiscale alignment error.}

\subsection{Application: Design of an optimal registration algorithm}
\label{ssec:grid_application}

We now demonstrate the usage of our SIDEN estimate in the construction of a grid in the translation parameter domain that is used for image registration. In Section \ref{sec:siden_estim}, we have derived a set $\estsiden$ of translation vectors that can be correctly computed by minimizing the distance function with descent methods, where $\estsiden$ is a subset of the SIDEN $\siden$ corresponding to the noiseless distance function. As discussed in the beginning of Section \ref{sec:noise_anly}, in noisy settings, one can assume that $\estsiden$ is also a subset of the perturbed SIDEN corresponding to the noisy distance function, provided that the noise level is sufficiently small. Therefore the estimate $\estsiden$ can be used in the registration of both noiseless and noisy images; small translations that are inside $\estsiden$ can be recovered with gradient descent minimization. However, the perfect alignment guarantee is lost for relatively large translations that are outside $\estsiden$ and the descent method may terminate in a local minimum other than the global minimum. Hence, in order to overcome this problem, we propose to construct a grid in the translation parameter domain and estimate large translations with the help of the grid. In particular, we describe a grid design procedure such that any translation vector $tT$ lies inside the SIDEN of at least one grid point. Such a grid guarantees the recovery of the translation parameters if the distance function is minimized with a gradient descent method that is initialized with the grid points. In order to have a perfect recovery guarantee, each one of the grid points must be tested. However, as this is computationally costly, we use the following two-stage optimization instead, which offers a good compromise with respect to the accuracy-complexity tradeoff. First, we search for the grid vector that gives the smallest distance between the image pair, which results in a coarse alignment. Then, we refine the alignment with a gradient descent method initialized with this grid vector. In practice, this method is quite likely to give the optimal solution, which has been the case in all of our simulations. 

We now explain the construction of the grid. First, notice from (\ref{eq:a_b_c_jk}) that $  a_{jk}(T)= a_{jk}(-T)$ and $b_{jk}(T)= -b_{jk}(-T)$. Therefore, the function $s_{jk}(t)$ given in (\ref{eq:defn_sjkt}) is the same  for $T$ and $-T$ by symmetry. As $Q_{jk}$ does not depend on $T$, from the form of $df(tT)/dt$ in (\ref{eq:form_dft_dt_v1}) we have $ df(tT) / dt = df(-tT) / dt$. Hence, the SIDEN is symmetric with respect to the origin. It is also easy to check that the estimation $\delta_T$ of the SIDEN boundary along the direction $T$ satisfies $\delta_T = \delta_{-T}$. One can easily determine a grid unit in the form of a parallelogram that lies completely inside the estimate $\estsiden$ of the SIDEN and tile the $(t T_x,t T_y)$-plane with these grid units. This defines a regular grid in the $(t T_x,t T_y)$-plane such that each point of the plane lies inside the SIDEN of at least one grid point. Note that the complexity of image registration based on a grid search is given by the number of grid points. In our case, the number of grid points is determined by the area of $\estsiden$; and therefore, the alignment complexity depends on the well-behavedness of the distance function $f$. In particular, as $V(\estsiden)$ increases with the filter size, the area of the grid units expand at the rate $O(1+\rho^2)$ and the number of grid points decrease at the rate $O\left( (1+\rho^2)^{-1} \right)$ with $\rho$. Therefore, the alignment complexity with the proposed method is of $O\left( (1+\rho^2)^{-1} \right)$.

\begin{figure}[]
\begin{minipage}[b]{0.3\linewidth}
\begin{center}
     \subfigure[Original pattern and grid images]
       {\label{fig:orig_pats}\includegraphics[height=3.0cm]{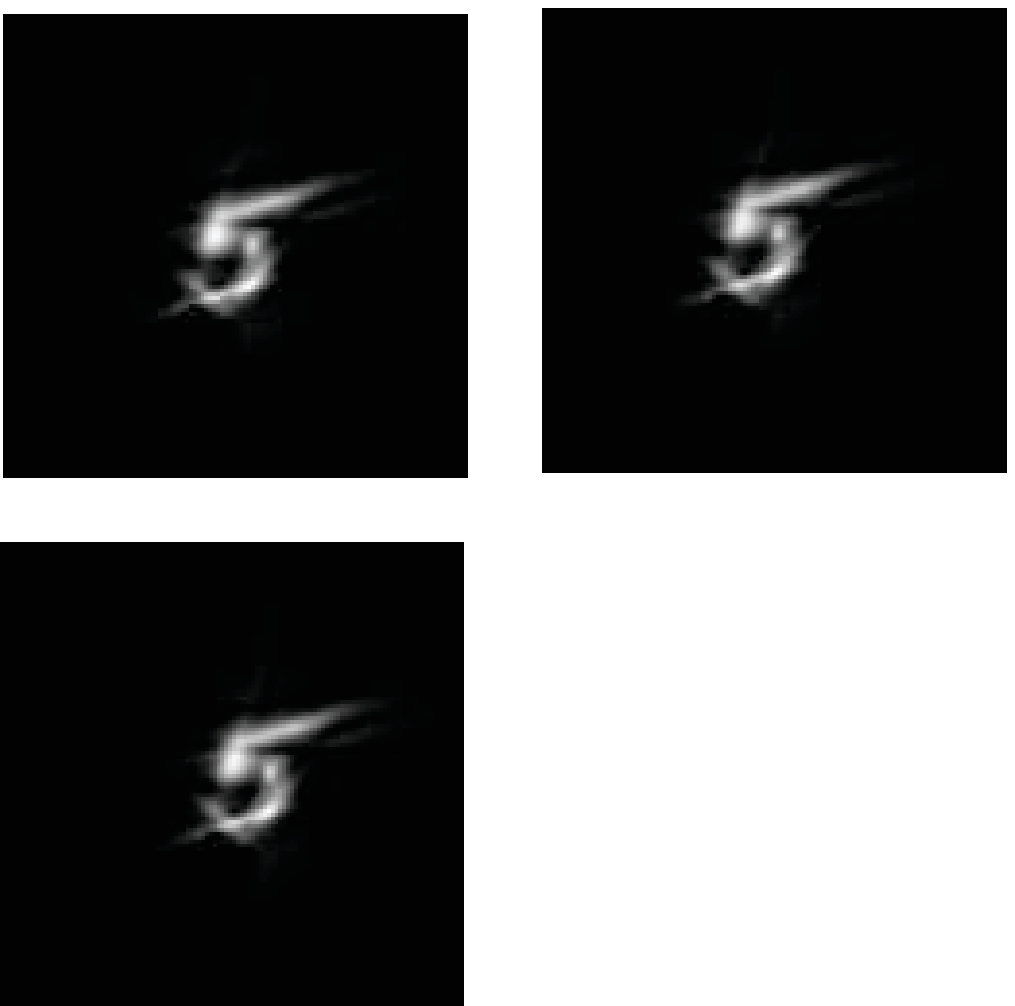}}      
     \subfigure[Smoothed pattern and grid images]
       {\label{fig:smoothed_pats}\includegraphics[height=3.0cm]{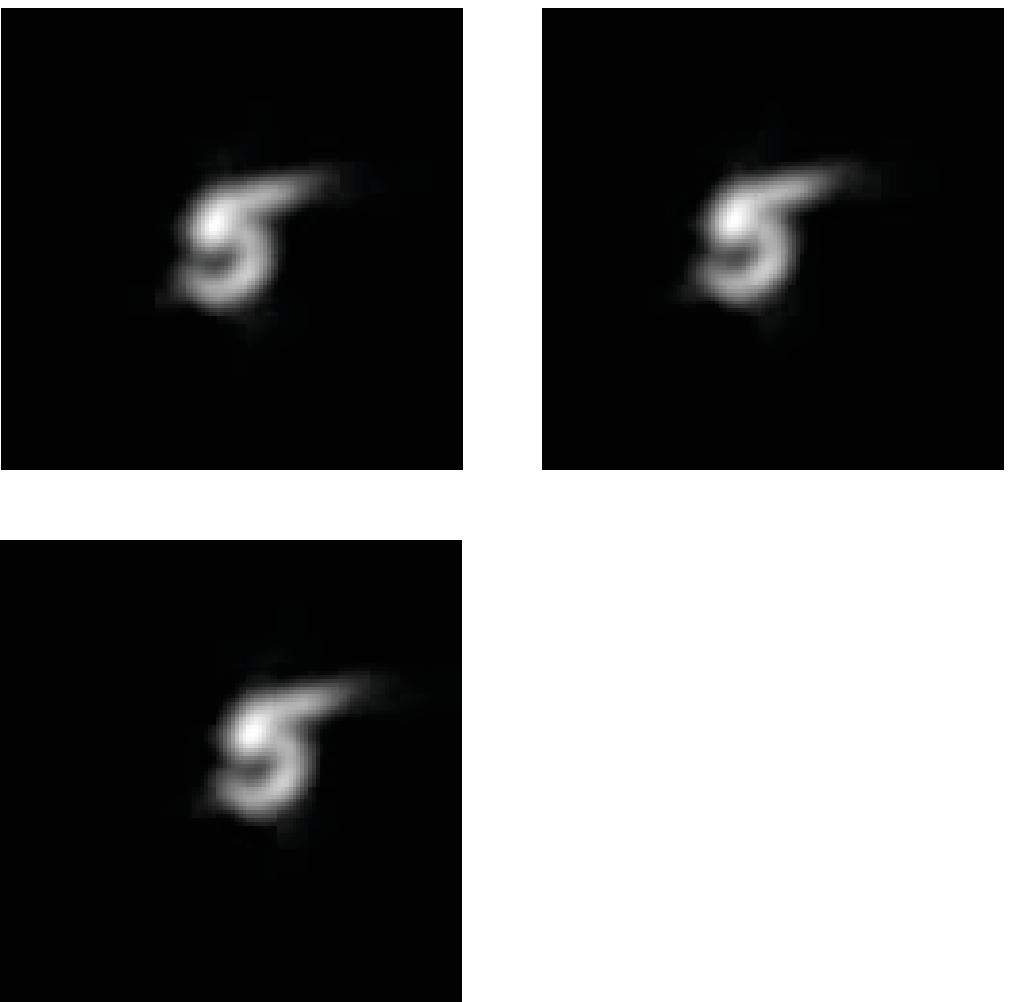}}
 \end{center}
 \caption{Neighboring grid patterns for the original and smoothed images.}
 \label{fig:gridPatterns}
\end{minipage}
\hspace{0.1cm}
\begin{minipage}[b]{0.3\linewidth}
\begin{center}
      \subfigure[Original grid]
       {\label{fig:orig_grid}\includegraphics[width=4.9cm]{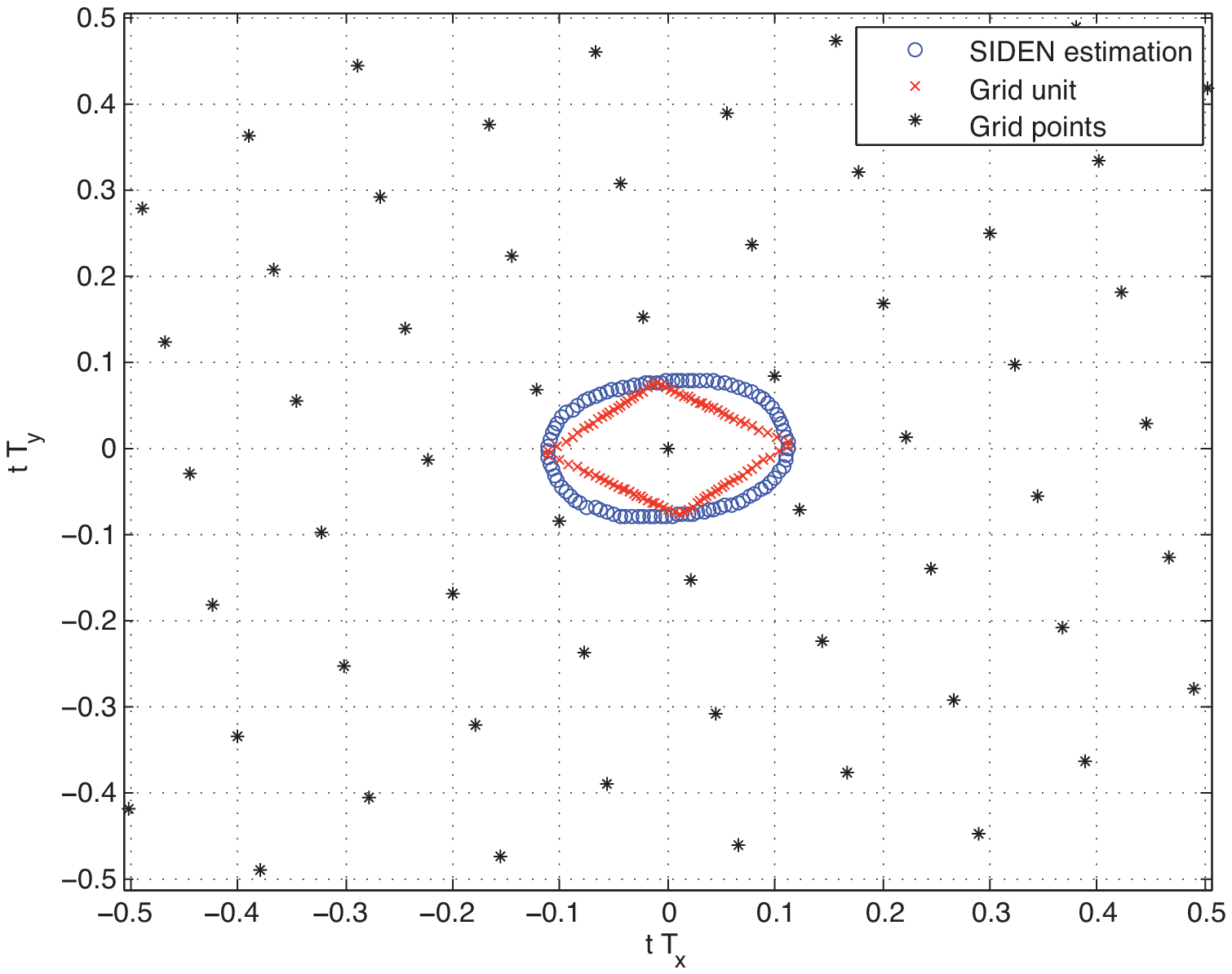}} 
      \subfigure[Smoothed grid]
       {\label{fig:smoothed_grid}\includegraphics[width=4.9cm]{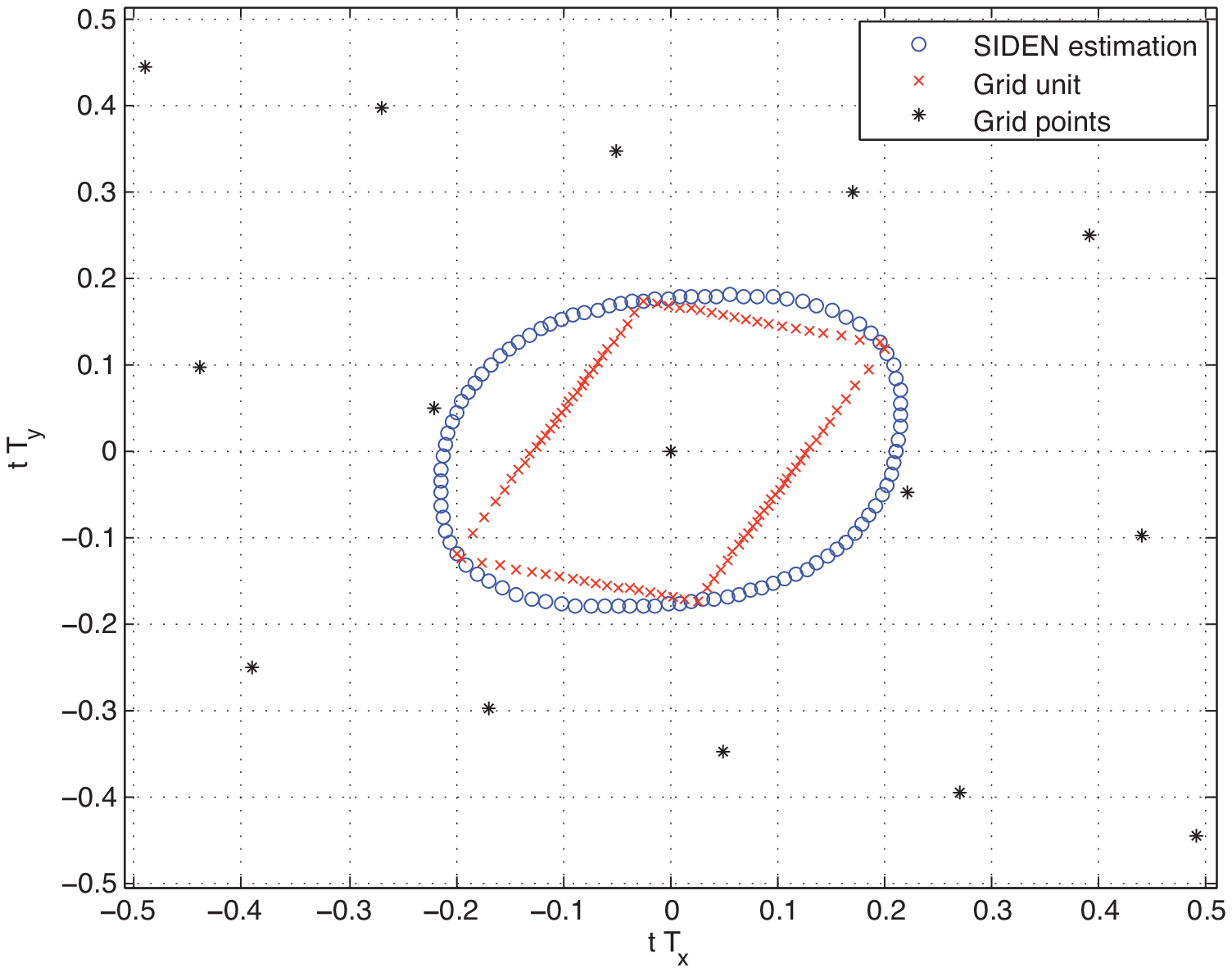}} 
 \end{center}
 \caption{Grid construction}
 \label{fig:gridConstr}
\end{minipage}
\hspace{0.1cm}
\begin{minipage}[b]{0.3\linewidth}
\begin{center}
       \subfigure[Original distance function $f(tT)$]
       {\label{fig:dist_func_orig}\includegraphics[width=4.5cm]{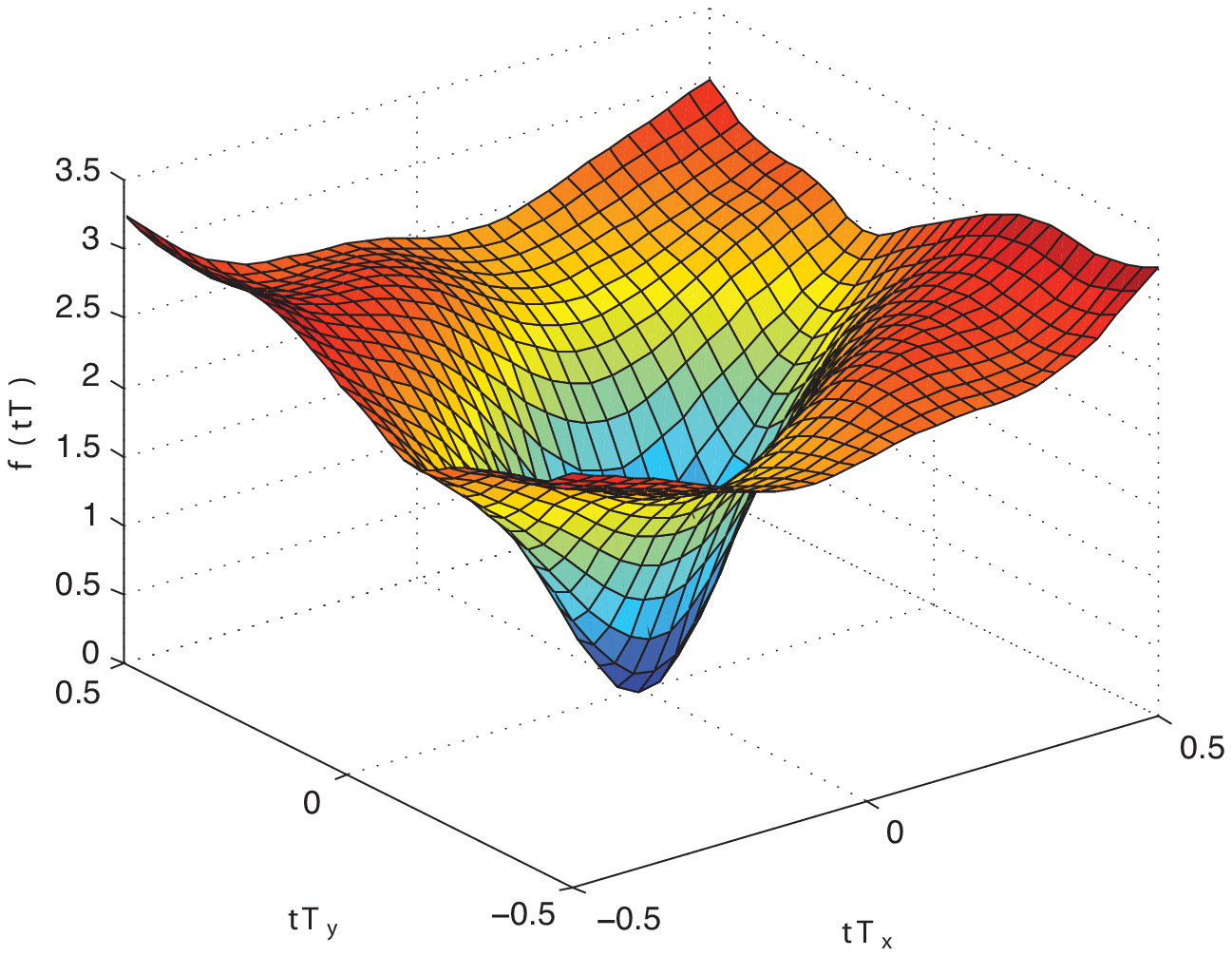}} 
      \subfigure[Smoothed distance function $\hat{f}(tT)$]
       {\label{fig:dist_func_smooth}\includegraphics[width=4.5cm]{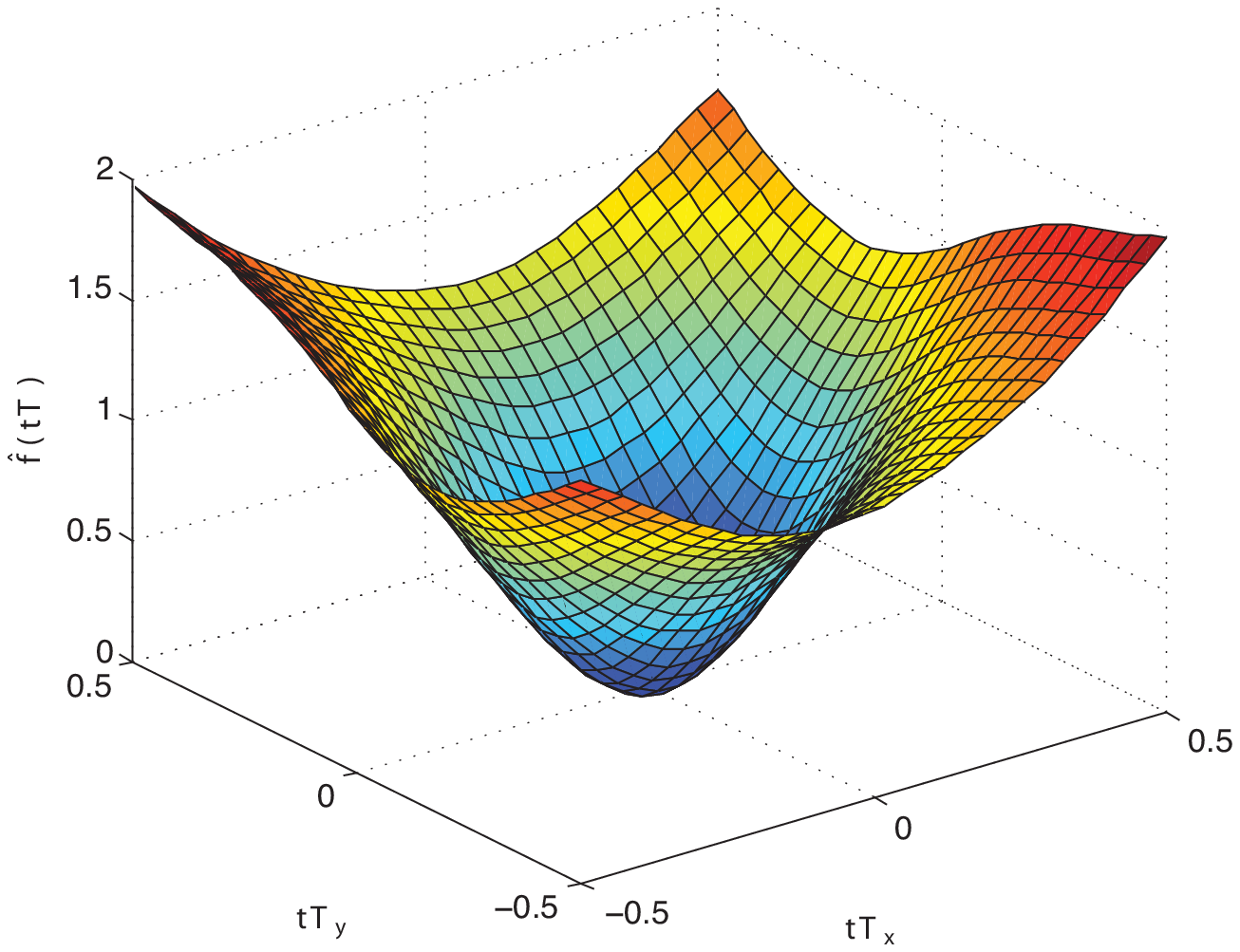}} 
 \end{center}
 \caption{The variation of the distance function with smoothing.}
 \label{fig:plotDistFunc}
\end{minipage}
\end{figure}

%
%
%

\begin{figure}[]
\begin{center}
     \subfigure[Random patterns]
       {\label{fig:exprand_grid}\includegraphics[height=3.8cm]{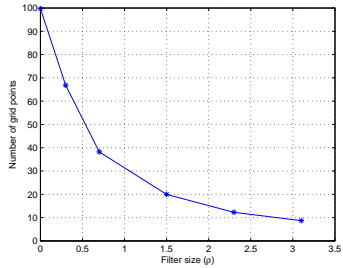}}
     \subfigure[Face pattern ]
       {\label{fig:expface_grid}\includegraphics[height=3.8cm]{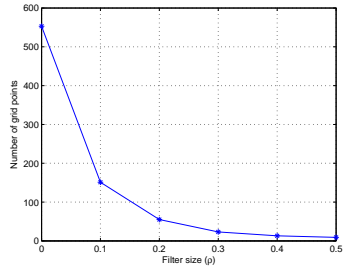}}
      \subfigure[Digit pattern]
       {\label{fig:expdigit_grid}\includegraphics[height=3.8cm]{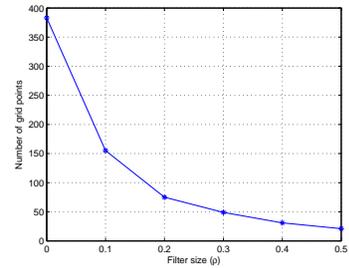}}
 \end{center}
 \caption{Number of grid points. The decay rate is of $O\left( (1+\rho^2)^{-1} \right)$.}
 \label{fig:num_grid_pts}
\end{figure}

The construction of a regular grid in this manner is demonstrated for the image of the ``5'' digit used in the experiments of Section \ref{ssec:exp_align_accuracy}. In Figure \ref{fig:orig_pats}, the reference pattern and its translated versions corresponding to the neighboring grid points in the first and second directions of sampling are shown. In Figure \ref{fig:smoothed_pats}, the reference pattern is shown when smoothed with a filter of size $\rho=0.15$, as well as the neighboring patterns in the smoothed grid. The original grid and the smoothed grid for $\rho=0.15$ are displayed in Figures \ref{fig:orig_grid} and \ref{fig:smoothed_grid}, where the SIDEN estimates $\estsiden$, $\hatestsiden$ and the grid units are also plotted. One can observe that smoothing the pattern is helpful for obtaining a coarser grid that reduces the computational complexity of image registration in hierarchical methods. The corresponding distance functions $f(tT)$ and $\hat{f}(tT)$ are plotted in Figure \ref{fig:plotDistFunc}, which shows that smoothing eliminates undesired local extrema of the distance function and therefore expands the SIDEN.

Then, in order to demonstrate the relation between alignment complexity and filtering, we build multiresolution grids corresponding to different filter sizes and plot the variation of the number of grid points with the filter size. The results obtained with the random patterns and the face and digit patterns used in Section \ref{ssec:exp_align_accuracy} are presented in Figure \ref{fig:num_grid_pts}. The results show that the number of grid points decreases monotonically with the filter size, as predicted by Theorem \ref{prop:volumeSIDEN}, which suggests that the number of grid points must be of $O\left( (1+\rho^2)^{-1} \right)$.

The performance guarantee of this two-stage registration approach is confirmed by the experiments of Section \ref{ssec:exp_align_accuracy}, which use this registration scheme. In the plots of Figures \ref{fig:exprand_rho}-\ref{fig:expgenrand_rho}, where the coarse alignment in each experiment is done with the help of a grid adapted to the filter size using the grid design procedure explained above, we see that the proposed registration technique results in an alignment error of $0$ for the noiseless case ($\eta=0$ or $\nu=0$) for all values of the filter size $\rho$. These alignment error results, together with the grid size plots of Figure \ref{fig:num_grid_pts}, show that increasing the filter size reduces the alignment complexity while retaining the perfect alignment guarantee in the noiseless case. Figures \ref{fig:exprand_rho}-\ref{fig:expgenrand_rho} show also that the proposed grid can be successfully used in noisy settings. The alignment error in these experiments stems solely from the change in the global minimum of the distance function due to noise, and not from the grid;  otherwise, we would observe much higher alignment errors that are comparable to the distance between neighboring grid points. This confirms that the estimate $\estsiden$ remains in the perturbed SIDEN and its usage does not lead to an additional alignment error if the noise level is relatively small. 

\revis{Finally, we note that, due to the computations needed for obtaining the SIDEN estimates, the construction of an optimal grid that is specifically adapted to a reference image is favorable especially in an application where the reference image is fixed and various target images are to be aligned with it. In an application where target images are aligned with different reference images, a more practical solution may  be preferred by constructing a non-adaptive, ``average'' grid. One may, for instance, compute the SIDEN estimates for some typical patterns that have similar properties (e.g.~frequency characteristics, size, support) with the reference images at hand and construct a non-adaptive grid based on these typical SIDEN estimates. The multiscale version of this grid can then be computed by adjusting the grid density to match a rate of $O( (1+\rho^2)^{-1})$ as discussed previously.}

\section{Discussion of Results}
\label{sec:discussion}

The results of our analysis show that smoothing improves the regularity of alignment by increasing the range of translation values that are computable with descent-type methods. However, in the presence of noise, smoothing has a negative influence on the accuracy of alignment as it amplifies the alignment error caused by the image noise; and this increases with the increase in the filter size. Therefore, considering the computation cost - accuracy tradeoff, the optimal filter size in image alignment with descent methods must be chosen by taking into account the deviation between the target pattern and the translation manifold of the reference pattern; i.e., the expected noise level.

Our study constitutes a theoretical justification of the principle behind hierarchical registration techniques that use local optimizers. Coarse scales are favorable at the beginning of the alignment as they permit the computation of large translation amounts with low complexity using simple local optimizers; however, over-filtering decreases the accuracy of alignment as the target image is in general not exactly a translated version of the reference image. This is compensated for at finer scales where less filtering is applied, thus avoiding the amplification of the alignment error resulting from noise. Since images are already roughly registered at  coarse scales, at fine scales the remaining increment to be added to the translation parameters for fine-tuning the alignment is small; it can be achieved at a relatively low complexity in a small search region.

\revis{In the following, we first take these observations one step further to sketch some rules for designing a coarse-to-fine registration algorithm that recovers the correct translation parameters. Next, we discuss the links between our results and previous works.

\subsection{Design of hierarchical registration algorithms}

Consider a descent-type coarse-to-fine registration method that computes the translation parameters by aligning smoothed versions of the image pair with a series of low-pass filters of size $\rho_1, \rho_2, \dots, \rho_n$ such that the initial solution in stage $k$ is taken as the estimate obtained in the previous stage $k-1$. Let $t_* T_*$ denote the optimal translation vector that best aligns the reference image and the target image; i.e., the target image is a noisy observation of the translated version $p(X- t_* T_*)$ of the reference image $p(X)$. We now state some conditions under which such a registration algorithm converges for the Gaussian noise model and the generic noise model.\\

\textbf{Gaussian noise model.} Let the noise level $\eta$ be upper bounded inversely proportionally to the amount of translation $t_*$ between the reference and target images
\begin{equation*}
\eta \leq O\left(  \frac{1}{t_*} \right).
\end{equation*}
Then, if the size $\rho_1$ of the filter in the first stage is chosen proportionally to the amount of translation  $t_*$ such that
\begin{equation*}
\rho_1 = O\left( \sqrt{t_*^2 -1 } \right) \approx O\left( t_* \right)
\end{equation*}
and if the filter sizes $\rho_{k-1}$, $\rho_k$ in adjacent stages are selected such that  
\begin{equation}
 \rho_k = O\left(  \sqrt{ 
	 \frac{ \eta \,  \rho_{k-1}^{3}   }
	{ 1 -  \eta \, \rho_{k-1}}	
   		-1   }
		  \right)
\label{eq:filt_update_rule_gauss}
\end{equation}
then the hierarchical registration algorithm converges to the correct solution $t_* T_*$.\\

\textbf{Generic noise model.} Assume that the noise level $\nu$ is below some sufficiently small constant threshold. If the filter size $\rho_1$ in the first stage is chosen proportionally to $t_*$ such that
\begin{equation*}
\rho_1 = O\left( \sqrt{t_*^2 -1 } \right) \approx O\left( t_* \right)
\end{equation*}
and if the filter sizes in adjacent stages are selected such that 
\begin{equation}
 \rho_k = O\left(  \sqrt{ 
	\left( \frac{ \nu }{ 1 - \nu}	\right)
	\left(1 + \rho_{k-1}^2  \right)
   		-1   }
		  \right)
 \label{eq:filt_update_rule_gen}
\end{equation}
then the hierarchical registration algorithm converges to the correct solution $t_* T_*$.\\

These results are derived in \cite[Appendix F]{Vural12TR}. The proposed strategies to select the filter size $\rho_k$ for both noise models ensure that the SIDEN in stage $k$ is sufficiently large to correct the alignment error of the previous stage $k-1$. The bounds on the noise levels then guarantee that the alignment error decays gradually to zero during the hierarchical alignment.

It is interesting to compare the conditions  (\ref{eq:filt_update_rule_gauss}) and (\ref{eq:filt_update_rule_gen}) to  common filter selection strategies used in practice. The filter size update relations in (\ref{eq:filt_update_rule_gauss}) and (\ref{eq:filt_update_rule_gen}) can be approximated\footnote{The more precise condition on the noise level $\eta \leq O(\alpha^2 / (1+ \alpha^2) \, t_*^{-1} )  $ derived in \cite[Appendix F]{Vural12TR} implies that the filter size in (\ref{eq:filt_update_rule_gauss}) satisfies $\rho_k \leq O( \sqrt{\alpha^2 \rho_{k-1}^2 -1 })$. Therefore, the selection $\rho_k = O( \sqrt{\alpha^2 \rho_{k-1}^2 -1 }) \approx C \, \rho_{k-1}$ assures that the SIDEN at stage $k$ is sufficiently large, while it can also be shown to retain the convergence guarantee.} in the form $\rho_k \approx C \, \rho_{k-1}$. This is consistent with the practical filter size selection strategy $\rho_k = 1/2 \, \rho_{k-1}$, which is used quite commonly in hierarchical registration and optical flow estimation algorithms. Our results imply, however, that the noise level must be taken into account in the selection of the filter size reduction factor $C$, such that $C$ should be set to a larger value if the noise level is higher.

}

\subsection{Comparison of results with previous studies}

We now interpret the findings of our work in comparison with some previous results. We start with the article \cite{Robinson04} by Robinson et al., which studies the Cram\'{e}r-Rao lower bound (CRLB) of the registration. Since the CRLB is related to the inverse of the Fisher information matrix (FIM) $J$, the authors suggest to use the trace of $J^{-1}$ as a general lower bound for the MSE of the translation estimation. Therefore, the square root of $tr(J^{-1})$ can be considered as a lower bound for the alignment error. It has been shown in \cite{Robinson04} that $\sqrt{tr(J^{-1})} = O (\eta)$, where $\eta$ is the standard deviation of the Gaussian noise. In fact, this result tells that the alignment error with any estimator is lower bounded by $O(\eta)$, i.e., its dependence on the noise level is at least linear. Meanwhile, our study, which focuses on estimators that minimize the SSD with local optimizers, concludes that the alignment error is at most $O(\sqrt{\eta / (1-\eta)})$ for these estimators. Notice that for small $\eta$,  $O(\sqrt{\eta / (1-\eta)}) \approx O(\sqrt{\eta}) > O(\eta)$; and for large $\eta$, we still have $O(\sqrt{\eta / (1-\eta)})> O(\eta)$ due to the sharply increasing rational function form of the bound. Therefore, the result in \cite{Robinson04} and our results are consistent and complementary, pointing together to the fact that the error of an estimator performing a local optimization of the SSD must lie between $O(\eta)$ and $O(\sqrt{\eta / (1-\eta)})$. Note also that, as it has been seen in the experiments of Figure \ref{fig:randpat_eta_exp_bigeta}, the error of this type of estimators may indeed increase with $\eta$ at a rate above $O(\eta)$ in practice, as predicted by our upper bound. Next, as for the effect of filtering on the estimation accuracy, the authors of \cite{Robinson04}  experimentally observe that $tr(J^{-1})$ decreases as the image bandwidth increases, which suggests that the lower bound on the MSE of a translation estimator is smaller when the image has more high-frequency components. This is stated more formally in \cite{Pham05}. It is shown that the estimation of the $x$ component of the translation has variance larger than $\eta^2/\| (\partial_x p(\cdot) )^2 \|^2$, and similarly for the $y$ component, where $\partial_x p(X)=\partial p(X) / \partial x$ is the partial derivative of the pattern $p$ with respect to the spatial variable $x$.\footnote{This bound is obtained by assuming Gaussian noise on both reference and target patterns and employing CRLB.} Therefore, as smoothing decreases the norm of the partial derivatives of the pattern, it leads to an increase in the variance of the estimation. These observations are also consistent with our theoretical results.

Next, we discuss some results from the recent article \cite{Sabater11}, which presents a continuous-domain noise analysis of block-matching. The blocks are assumed to be corrupted with additive Gaussian noise and the disparity estimate is given by the global minimum of the noisy distance function as in our work. Although there are differences in their setting and ours, such as the horizontal and non-constant disparity field assumption in \cite{Sabater11}, it is interesting to compare their results with ours. We consider the disparity of the block to be constant in \cite{Sabater11}, such that it fits our global translation assumption. In \cite{Sabater11}, an analysis of the deviation between the estimated disparity and the true disparity is given, which is similar to the distance $\tming$ between the global minima of $f$ and $g$ in our work. Sticking to the notation of our text, let us denote this deviation by $\tming$. In Theorem 3.2 of \cite{Sabater11}, $\tming$ is estimated as the sum of three error terms, where the variances of the first and second terms are respectively of $O(\eta^2)$ and $O(\eta^4)$ with respect to the noise standard deviation $\eta$. These two terms are stated as dominant noise terms. Then, the third term represents the high-order Taylor terms of some approximations made in the derivations, which however depends on the value of $\tming$ itself. As  the overall estimation of $\tming$ is formulated using the term $\tming$ itself, their main result is interesting especially for small values of $\tming$, since the third term is then negligible. It can be concluded from \cite{Sabater11} that $\tming \approx O( \eta + \eta^2 + H.O.)$, where $H.O.$ represents high-order terms. This result is consistent with the CRLB of $\tming$ in \cite{Robinson04} stating that $\tming$ is at least of $O(\eta) $, and our upper bound $O(\sqrt{\eta} / {\sqrt{1-\eta}})$. The analysis in \cite{Sabater11} and ours can be compared in the following way. First, since our derivation is rather rigorous and does not neglect high-order terms, these terms manifest themselves in the rational function form of the resulting bound. Meanwhile, they are represented as $H.O.$ and not explicitly examined in the estimation $O( \eta + \eta^2 + H.O.)$ in \cite{Sabater11}. For small $\eta$, this result can be approximated as $\tming = O(\eta)$, while our result states that $\tming \leq O(\sqrt{\eta})$. As $\sqrt{\eta}>\eta$ for small $\eta$, this also gives a consistent comparison since our estimation is an upper bound and the one in \cite{Sabater11} is not. Indeed, the experimental results in \cite{Sabater11} suggest that their derivation gives a slight underestimation of the error. Lastly, our noise analysis treats the image alignment problem in a multiscale setting and analyzes the joint variation of the error with the noise level $\eta$ and the filter size $\rho$, whereas the study in \cite{Sabater11} only concentrates on the relation between the error and the noise level.

Finally, we mention some facts from scale-space theory \cite{Lindeberg94}, which may be useful for the interpretation of our findings regarding the variation of the SIDEN with the filter size $\rho$. The scale-space representation of a signal is given by convolving it with kernels of variable scale. The most popular convolution kernel is the Gaussian kernel, as it has been shown that under some ``well-behavedness'' constraints, it is the unique kernel for generating a scale-space. An important result in scale-space theory is \cite{Witkin83}, which states that the number of local extrema of a 1-D function is a decreasing function of $\rho$. This provides a mathematical characterization of the well-known smoothing property of the Gaussian kernel. However, it is known that this cannot be generalized to higher-dimensional signals; e.g., there are no nontrivial kernels on $\Rsq$ with the property of never introducing new local extrema when the scale increases \cite{Lindeberg94}. One interesting result that can possibly be related to our analysis is about the density of local extrema of a signal as a function of scale. In order to gain an intuition about the behavior of local extrema, the variation of the local extrema is examined in \cite{Lindeberg94} for 1-D continuous white noise and fractal noise processes. It has been shown that the expected density of the local minima of these signals decreases at rate $\rho^{-1}$. In the estimation of the SIDEN in our work, we have analyzed how the first zero crossing of $d\hat{f}(tT)/dt$ along a direction $T$ around the origin varies with the scale. Therefore, what we have examined is the distance $\hat{f}$ between the scale space $\hat{p}(X)$ of an image and the scale space $\hat{p}(X-tT)$ of its translated version. Since this is different from the scale-space of the distance function $f$ itself, it is not possible to compare the result in \cite{Lindeberg94} directly to ours. However, we can observe the following. Restricting $\hat{f}$ to a specific direction $T_a$ so that we have a 1-D function $\hat{f}(t T_a)$ of $t$ as in \cite{Lindeberg94}, our estimation for the stationary point of $\hat{f}(t T_a)$ closest to the origin expands at a rate of $O( (1+ \rho^2)^{1/2} )$, which is $O(\rho)$ for large $\rho$. One can reasonably expect this distance to be roughly inversely proportional to the density of the local extrema of $\hat{f}$. This leads to the conclusion that the density of the distance extrema is expected to be around $O(\rho^{-1})$, which interestingly matches the density obtained in \cite{Lindeberg94}.

\section{Conclusion}
\label{sec:Conclusion}

We have presented a theoretical analysis of image alignment with descent-type local minimizers, where we have specifically focused on the effect of low-pass filtering and noise on the regularity and accuracy of alignment. First, we have examined the problem of aligning with gradient descent a reference and a target pattern that differ by a two-dimensional translation. We have derived a lower bound for the range of translations for which the reference pattern can be exactly aligned with its translated versions, and investigated how this region varies with the filter size when the images are smoothed. Our finding is that the volume of this region increases quadratically with the filter size, showing that smoothing the patterns improves the regularity of alignment. Then, we have considered a setting with noisy target images and examined Gaussian noise and arbitrary noise patterns, which may find use in different imaging applications. We have derived a bound for the alignment error and searched the dependence of the error on the noise level and the filter size. Our main results state that the alignment error bound is proportional to the square root of the noise level at small noise, whereas this order of dependence increases at larger noise levels. More interestingly, the alignment error is also significantly affected by the filter size. The probabilistic error bound obtained with the Gaussian noise model has been seen to increase with the filter size at a sharply increasing rational function rate, whereas the deterministic bound obtained for arbitrary noise patterns of deterministic norm increases approximately linearly with the filter size. These theoretical findings are also confirmed by experiments. To the best of our knowledge, none of the previous works about image registration has studied the alignment regularity problem. Meanwhile, our alignment accuracy analysis is consistent with previous results, provides a more rigorous treatment, and studies the problem in a multiscale setting unlike the previous works. The results of our study show that, in multiscale image registration, filtering the images with large filter kernels improves the alignment regularity in early phases, while the use of smaller filters improves the accuracy at later phases. From this aspect, our estimations of the regularity and accuracy of alignment in terms of the noise and filter parameters provide insight for the principles behind hierarchical registration techniques and may find use in the design of efficient, low-complexity registration algorithms.

\begin{appendix}
\section{Exact expressions for  $\cdeltahposjkUB$ and $\cdeltahnegjkLB$}
\label{app:der_cdeltahs}

Here we present the expressions for the terms $\cdeltahposjkUB$ and $\cdeltahnegjkLB$ used in Lemma \ref{cor:var_deltah_unif}. Let $\alpha_k = \lambda_{\text{min}}(\Phi_k)$ and $\quad  \beta_k = \lambda_{\text{max}}(\Phi_k)$ denote respectively the smaller and greater eigenvalues of $\Phi_k$. Since $\Phi_k$ is a positive definite matrix, $\beta_k \geq \alpha_k >0$. Let $\tau_k=[\tau_{x,k} \,\,\, \tau_{y,k} ]^T$,
\begin{eqnarray*}
&\cLx& = (\beta_j + \beta_k)  \max \left\{  
	 \left( b + \tbound -  \frac{\beta_j \Tauxj + \beta_k \Tauxk}{\beta_j + \beta_k}   \right)^2
	, 
	\left( -b - \tbound -  \frac{\beta_j \Tauxj + \beta_k \Tauxk}{\beta_j + \beta_k}   \right)^2
 	\right\}\\
&\ctwoLx& = (\beta_j + \beta_k) \max 
		\left \{
		\left( b + \tbound \frac{\beta_j}{\beta_j +\beta_k} 
 			 -  \frac{\beta_j \Tauxj + \beta_k \Tauxk}{\beta_j +\beta_k} \right)^2
			 ,
		\left(  -b - \tbound \frac{\beta_j}{\beta_j +\beta_k} 
 			 -  \frac{\beta_j \Tauxj + \beta_k \Tauxk}{\beta_j +\beta_k}	 \right)^2	 
		\right \}\\
&\Rx& = \frac{\beta_j \beta_k}{\beta_j + \beta_k}
	 \max
	\left\{  (- \tbound + \Tauxk - \Tauxj)^2 ,  ( \tbound + \Tauxk - \Tauxj)^2 \right\} 
\end{eqnarray*}
and $\cLy$, $\ctwoLy$ and $\Ry$ be defined similarly by replacing $x$ with $y$ in the above expressions. Let
\begin{equation*}
\Hbetax = - \frac{  \beta_j \,\Tauxj   + \beta_k \, \Tauxk  }   { \beta_j + \beta_k }, 
\qquad 
\Gbetax = \frac{  \beta_j \, \Tauxj ^2  + \beta_k \, \Tauxk ^2  }   { \beta_j + \beta_k }
\end{equation*}
and $\Hbetay$, $\Gbetay$ be defined similarly by replacing $x$ with $y$ in the above expressions. Also, let $\Halphax$, $\Halphay$, $\Galphax$, $\Galphay$ denote the terms obtained by substituting $\beta_j$, $\beta_k$ with $\alpha_j$, $\alpha_k$ in the expressions of respectively $\Hbetax$, $\Hbetay$, $\Gbetax$, $\Gbetay$. \footnote{The terms written as $\Hbetax, \dots, \Galphay$ here correspond to the terms $\Hbetax(0,0), \dots,  \Galphay(0,0)$ in \cite{Vural12TR}.} Then, let
\begin{equation*}
  \begin{split}
  \BthreejkLBx = &\frac{\sqrt{\pi}}{4b}      \frac{1}{\sqrt{\beta_j + \beta_k}} \,
  \exp \left(  - (\beta_j +\beta_k)  \left(  \Gbetax - (\Hbetax)^2  \right) \right)  \\
  & \cdot \left[
    \erf   \left(  \sqrt{\beta_j + \beta_k} \, (b+ \Hbetax)  \right)
   - \erf   \left(  \sqrt{\beta_j + \beta_k} \, (-b+ \Hbetax)  \right)
   \right] 
   \end{split}
\end{equation*}
and $\BthreejkLBy$ be defined similarly, and let us denote by $\BthreejkUBx$, $\BthreejkUBy$ the terms obtained by replacing $\beta_j$, $\beta_k$ with $\alpha_j$, $\alpha_k$ in the expressions of $\BthreejkLBx$, $\BthreejkLBy$. Then defining
\begin{eqnarray*}
&\BonejkLBunif& = \exp \left(  - \cLx - \cLy - \frac{\beta_j \beta_k \, \| \tau_k - \tau_j  \| ^2 }{(\beta_j + \beta_k)}   \right),
\qquad
\BonejkUBunif =  \frac{\pi}{4 b^2 \, (\alpha_j + \alpha_k)} 
		 \exp \left(  -  \frac{\alpha_j \alpha_k \, \| \tau_k - \tau_j\|^2}{(\alpha_j + \alpha_k)}   \right)\\	
&\BtwojkLBunif& = \exp \left(  - \ctwoLx - \ctwoLy -  \Rx - \Ry  \right),
\quad
\BtwojkUBunif =  \frac{\pi}{ 4b^2 \, (\alpha_j + \alpha_k) },
\quad
\BthreejkLB = \BthreejkLBx \, \BthreejkLBy,
\quad
\BthreejkUB = \BthreejkUBx \, \BthreejkUBy
\end{eqnarray*}
the parameters $\cdeltahposjkUB$ and $\cdeltahnegjkLB$ are given by
$
\cdeltahposjkUB = \BonejkUBunif - 2 \BtwojkLBunif   + \BthreejkUB
$, 
$
\, \cdeltahnegjkLB = \BonejkLBunif - 2 \BtwojkUBunif   + \BthreejkLB
$.

\section{Exact expressions for  $\betazerolb$, $\betatwolb$ and $\betathreelb$}
\label{app:defn_betalbs}

We now give the expressions for the terms $\betazerolb$, $\betatwolb$ and $\betathreelb$ used in Lemma \ref{lemma:d2f_dt2_lb}. Let $\betazerolb$ be the smaller eigenvalue of the following positive definite matrix
\begin{equation*}
\Sigmabetazero = \sumj \sumk \coef_j \coef_k \, Q_{jk}  
			\left(  \Sigma_{jk}^{-1}   -     \Sigma_{jk}^{-1}  (\tau_k - \tau_j)(\tau_k - \tau_j)^{T} \Sigma_{jk}^{-1}   \right).
\end{equation*}
Next, as shown in \cite{Vural12TR}, the following upper and lower bounds can be obtained
\begin{eqnarray*}
&a_{jk}^2 &\geq  \underline{a}_{jk}^2 = \frac{1}{4} \lambda_{\min}^2 (\Sigma_{jk}^{-1}), 
\qquad
a_{jk}^2 \leq  \overline{a}_{jk}^2 = \frac{1}{4} \lambda_{\max}^2 (\Sigma_{jk}^{-1})\\
&b_{jk}^2 a_{jk} &\leq  \overline{b}_{jk}^2 \overline{a}_{jk}
= \frac{1}{8}  \lambda_{\max} (\Sigmabetatwojk)  \lambda_{\max} (\Sigma_{jk}^{-1}),  
\qquad
b_{jk}^4 \leq \overline{b}_{jk}^4 = \frac{1}{16}  \lambda_{\max}^2 (\Sigmabetatwojk) 
\end{eqnarray*}
where $
\Sigmabetatwojk = \Sigma_{jk}^{-1} (\tau_k - \tau_j)(\tau_k - \tau_j)^{T} \Sigma_{jk}^{-1}
$.
Then, we have 
\begin{equation*}
\betatwolb'  = \sum_{(j,k)\in J^{+}}  \coef_j \coef_k Q_{jk}  ( - 8 \, \overline{b}_{jk}^4 +     - 6 \, \overline{a}_{jk}^2)   
\quad + \sum_{(j,k)\in J^{-}}  \coef_j \coef_k Q_{jk}  (  24 \,  \overline{b}_{jk}^2 \overline{a}_{jk}  - 6 \, \underline{a}_{jk}^2)    
\end{equation*}
and $\betatwolb$ is given by  $\betatwolb =\min(\betatwolb', 0) $. Finally, 
\begin{equation*}
\betathreelb = - \sumj \sumk 
			\frac{5.46}{2^{5/2}} \, | \coef_j \coef_k | \,  \frac{\pi | \sigma_j \sigma_k |}{\sqrt{| \Sigma_{jk}  |}} \,
			(\lambda_{\max} (\Sigma_{jk}^{-1})) ^{5/2} .
\end{equation*}

\section{Exact expressions for $\csecderhposjkUB$ and $\csecderhnegjkLB$}
\label{app:defn_csecderhs}

Now we present the  terms $\csecderhposjkUB$ and $\csecderhnegjkLB$ used in Lemma \ref{cor:bndunif_d2hdt2}. Let 
\begin{eqnarray*}
\ThAonejUB = \left( \frac{ (3^{3/4} + 3^{5/4}) e^{ -\frac{\sqrt{3}}{2} }  }{16}  	+ \frac{3 \sqrt{\pi}}{ 2^{9/2} } \right) \ \frac{1}{\alpha_j^{5/2}}, 
\quad
&\ThBonejUB& = \sqrt{  \frac{\pi}{2 \alpha_j} },
\quad
\ThtwojUB := \left( \frac{ e^{- \half}}{4}  	+ \frac{\sqrt{\pi}}{ 2^{5/2} } \right) \ \frac{1}{\alpha_j^{3/2}}\\
\EjUBunif = \frac{\beta_j^4}{4 b^2} \, \left( 2 \ThAonejUB  \ThBonejUB + 2 \ThtwojUB^2  \right),
&\qquad&
\FjUBunif = \frac{1}{4b^2} \ThBonejUB^2.
\end{eqnarray*}
Then, the terms $\csecderhposjkUB$ and $\csecderhnegjkLB$ are given by
\begin{equation*}
\begin{split}
\csecderhposjkUB = 16 \sqrt{  \EjUBunif \EkUBunif } + 4  \beta_j \beta_k \ \BonejkUBunif,
\qquad
\csecderhnegjkLB  = 	- 8 \ \beta_k \sqrt{ \EjUBunif \FkUBunif}
	- 8 \ \beta_j \sqrt{ \FjUBunif \EkUBunif}
	  + 4  \alpha_j \alpha_k \ \BonejkLBunif.
\end{split}
\end{equation*}

\end{appendix}

\setcounter{section}{7}

\section{Acknowledgment}

The authors would like to thank Gilles Puy for the helpful discussions on the optimization problem in Section \ref{ssec:siden_notations}, and Prof. Michael B. Wakin and the anonymous reviewers for their useful comments for improving the manuscript.

\bibliographystyle{siam}
\bibliography{refs}

\begin{thebibliography}{10}

\bibitem{Alvarez99}
{\sc L.~Alvarez, J.~Weickert, and J.~S\'{a}nchez}, {\em A scale-space approach
  to nonlocal optical flow calculations}, in Proc. 2nd Inter. Conf.on
  Scale-Space Theories in Computer Vision, SCALE-SPACE '99, London, UK, UK,
  1999, Springer-Verlag, pp.~235--246.

\bibitem{Anandan89}
{\sc P.~Anandan}, {\em {A computational framework and an algorithm for the
  measurement of visual motion}}, International Journal of Computer Vision, 2
  (1989), pp.~283--310.

\bibitem{Antoine2004}
{\sc J.~Antoine, R.~Murenzi, P.~Vandergheynst, and S.~Ali}, {\em
  Two-{D}imensional {W}avelets and their {R}elatives}, Signal Processing,
  Cambridge University Press, 2004.

\bibitem{Barron94}
{\sc J.~L. Barron, D.~J. Fleet, and S.~S. Beauchemin}, {\em Performance of
  optical flow techniques}, International Journal of Computer Vision, 12
  (1994), pp.~43--77.

\bibitem{Bergen92}
{\sc J.~R. Bergen, P.~Anandan, K.~J. Hanna, and R.~Hingorani}, {\em
  Hierarchical model-based motion estimation}, Springer-Verlag, 1992,
  pp.~237--252.

\bibitem{Brandt94}
{\sc J.~W. Brandt}, {\em Analysis of bias in gradient-based optical flow
  estimation}, in 1994 Conf. Rec. of the 28th Asilomar Conf. on Signals,
  Systems and Computers, vol.~1, 1994, pp.~721--725.

\bibitem{Brown92}
{\sc L.~G. Brown}, {\em A survey of image registration techniques}, ACM Comput.
  Surv., 24 (1992), pp.~325--376.

\bibitem{Cafforio83}
{\sc C.~Cafforio and F.~Rocca}, {\em The differential method for image motion
  estimation.}, Springer-Verlag, 1983.

\bibitem{Fitzgibbon03}
{\sc A.~W. Fitzgibbon and A.~Zisserman}, {\em Joint manifold distance: a new
  approach to appearance based clustering}, IEEE Conference on Computer Vision
  and Pattern Recognition, 1 (2003), p.~26.

\bibitem{Kearney87}
{\sc J.~K. Kearney, W.~B. Thompson, and D.~L. Boley}, {\em Optical flow
  estimation: {A}n error analysis of gradient-based methods with local
  optimization}, IEEE Trans. Pattern Anal. Machine Intel.,  (1987).

\bibitem{Lefebure01}
{\sc M.~Lef{\'e}bure and L.D. Cohen}, {\em Image registration, optical flow and
  local rigidity}, Journal of Mathematical Imaging and Vision, 14 (2001),
  pp.~131--147.

\bibitem{Lindeberg94}
{\sc T.~Lindeberg}, {\em Scale-Space Theory in Computer Vision}, Kluwer
  Academic Publishers, 1994.

\bibitem{Kanade81}
{\sc B.~D. Lucas and T.~Kanade}, {\em An iterative image registration technique
  with an application to stereo vision}, in Proc. 7th Intl. Joint Conf. on
  Artificial Intelligence, 1981, pp.~674--679.

\bibitem{Maintz98}
{\sc J.~Maintz and M.~Viergever}, {\em A survey of medical image registration},
  Medical Image Analysis, 2 (1998), pp.~1--36.

\bibitem{1028585}
{\sc S.~G. Mallat and Z.~Zhang}, {\em Matching pursuits with time-frequency
  dictionaries}, IEEE Trans. Signal Process., 41 (1993), pp.~3397--3415.

\bibitem{Memin02}
{\sc E.~M\'{e}min and P.~P\'{e}rez}, {\em Hierarchical estimation and
  segmentation of dense motion fields}, International Journal of Computer
  Vision, 46 (2002), pp.~129--155.

\bibitem{Nestares00}
{\sc O.~Nestares and D.~J. Heeger}, {\em Robust multiresolution alignment of
  {MRI} brain volumes}, Magnetic Resonance in Medicine, 43 (2000),
  pp.~705--715.

\bibitem{Netravali79}
{\sc A.~N. {Netravali} and J.~D. {Robbins}}, {\em Motion-compensated television
  coding: Part {I}}, Bell System Tech. Jour., 58 (1979), pp.~631--670.

\bibitem{Pati93}
{\sc Y.C. Pati, R.~Rezaiifar, and P.~S. Krishnaprasad}, {\em Orthogonal
  matching pursuit: recursive function approximation with applications to
  wavelet decomposition}, in 1993 Conference Record of The Twenty-Seventh
  Asilomar Conference on Signals, Systems and Computers, 1993, pp.~40--44.

\bibitem{Pham05}
{\sc T.~Q. Pham, M.~Bezuijen, L.~J. van Vliet, K.~Schutte, and C.~L. Luengo},
  {\em Performance of optimal registration estimators}, in Proc. SPIE, 2005,
  pp.~133--144.

\bibitem{Robinson04}
{\sc D.~Robinson and P.~Milanfar}, {\em Fundamental performance limits in image
  registration}, IEEE Trans. Img. Proc., 13 (2004), pp.~1185--1199.

\bibitem{Sabater11}
{\sc N.~Sabater, J.~M. Morel, and A.~Almansa}, {\em How accurate can block
  matches be in stereo vision?}, SIAM Journal on Imaging Sciences, 4 (2011),
  pp.~472--500.

\bibitem{Simard98}
{\sc P.~Simard, Y.~LeCun, J.~S. Denker, and B.~Victorri}, {\em Transformation
  invariance in pattern recognition-tangent distance and tangent propagation},
  in Neural Networks: Tricks of the Trade, New York: Springer-Verlag, 1998.

\bibitem{Tziritas94}
{\sc G.~Tziritas and C.~Labit}, {\em Motion Analysis for Image Sequence
  Coding}, Elsevier Science Inc., New York, NY, USA, 1994.

\bibitem{VasconcelosL05}
{\sc N.~Vasconcelos and A.~Lippman}, {\em A multiresolution manifold distance
  for invariant image similarity}, IEEE Transactions on Multimedia, 7 (2005),
  pp.~127--142.

\bibitem{Vural12TR}
{\sc E.~Vural and P.~Frossard}, {\em Analysis of descent-based image
  registration}.
\newblock Available at: http://infoscience.epfl.ch/record/183845.

\bibitem{Vural13}
\leavevmode\vrule height 2pt depth -1.6pt width 23pt, {\em Learning {S}mooth
  {P}attern {T}ransformation {M}anifolds}, {IEEE} {T}ransactions on {I}mage
  {P}rocessing, 22 (2013), pp.~1311--1325.

\bibitem{Walker84}
{\sc D.~Walker and K.~Rao}, {\em Improved pel-recursive motion compensation},
  IEEE Trans. Communications, 32 (1984), pp.~1128 -- 1134.

\bibitem{WandJones1995}
{\sc M.~P. Wand and M.~C. Jones}, {\em Kernel Smoothing}, Chapman and Hall/CRC,
  1995.

\bibitem{Witkin83}
{\sc A.~P. Witkin}, {\em Scale-space filtering.}, in 8th Int. Joint Conf.
  Artificial Intelligence, vol.~2, Karlsruhe, Aug. 1983, pp.~1019--1022.

\bibitem{Xu2009}
{\sc M.~Xu, H.~Chen, and P.~K. Varshney}, {\em Ziv-zakai bounds on image
  registration}, IEEE Trans. Signal Proc., 57 (2009), pp.~1745--1755.

\bibitem{Yetik06}
{\sc \.{I}.~\c{S}. Yetik and A.~Nehorai}, {\em Performance bounds on image
  registration}, IEEE Trans. Signal Proc., 54 (2006), pp.~1737 -- 1749.

\bibitem{Zitova03}
{\sc B.~Zitov\'{a} and J.~Flusser}, {\em Image registration methods: a survey},
  Image and Vision Computing, 21 (2003), pp.~977--1000.

\end{thebibliography}

\end{document}